\documentclass[12pt]{article} 
\usepackage{smile}
\usepackage{graphicx}
\usepackage{booktabs} 
\usepackage{fullpage}
\usepackage{natbib}
\usepackage{tikz}

\usepackage{wrapfig}

\usepackage{lineno}
\usepackage{caption}
\usepackage{subcaption}

\usepackage[colorlinks,
linkcolor=red,
anchorcolor=blue,
citecolor=blue
]{hyperref}

\def\[#1\]{\begin{bmatrix}#1\end{bmatrix}}

\usepackage{hyperref}
\usepackage{url}

\def\cme{{\mathtt{CME}}}
\def\Tr{\operatorname{tr}}

\def\att{\mathtt{attn}}
\def\softmax{\mathtt{softmax}}

\def\mha{{\mathtt{mha}}}
\def\head{{\mathtt{head}}}

\def\mask{{\mathtt{msk}}}
\def\tm{{\mathtt{m}}}

\def\noend{\notag \\}

\def\agg{\mathtt{agg}}
\def\nn{\mathtt{nn}}
\def\fk{\mathfrak{K}}
\def\fl{\mathfrak{L}}

\def\sm{{\mathtt{SM}}}

\def\Norm{\mathtt{norm}}
\def\tq{{\mathrm{q}}}
\def\tk{{\mathrm{k}}}
\def\tv{{\mathrm{v}}}
\def\tp{{\mathrm{p}}}

\def\relu{{\mathtt{ReLU}}}
\def\trans{{\mathtt{trans}}}
\def\ffn{{\mathtt{ffn}}}
\def\tx{{\mathrm{x}}}
\def\ty{{\mathrm{y}}}

\def\vec{{\mathrm{vec}}}
\def\ssl{{\mathtt{SSL}}}

\def\rbf{{\mathtt{RBF}}}
\def\gp{\mathtt{GP}}
\def\ltv{{\mathtt{LTV}}}
\def\dt{{\mathtt{DS}}}
\def\pre{{\mathtt{PT}}}
\def\fX{{\mathfrak{X}}}
\def\fY{{\mathfrak{Y}}}

\def\fW{{\mathfrak{W}}}
\def\texp{{\mathtt{EXP}}}

\title{\huge An Analysis of Attention via the Lens of Exchangeability and Latent Variable Models}

\author{\small Yufeng Zhang\thanks{equal contribution} \thanks{Northwestern University; \texttt{yufengzhang2023@u.northwestern.edu}},~~Boyi Liu\footnotemark[1] \thanks{Northwestern University; \texttt{boyiliu2018@u.northwestern.edu}},~~Qi Cai\thanks{Northwestern University; \texttt{qicai2022@u.northwestern.edu}},~~Lingxiao Wang\thanks{Northwestern University; \texttt{lingxiaowang2022@u.northwestern.edu}},~~Zhaoran Wang\thanks{Northwestern University; \texttt{zhaoranwang@gmail.com}}}
\date{}
\begin{document}

\maketitle

\begin{abstract}
With the attention mechanism, transformers achieve significant empirical successes in natural language processing and computer vision. Despite the intuitive understanding that transformers perform relational inference (or ``inductive reasoning'') over long sequences to produce desirable representations, we lack a rigorous theory on how the attention mechanism achieves it. In particular, several intriguing questions remain open: (a) What makes a desirable representation? (b) How does the attention mechanism infer the desirable representation within the forward pass? (c) How does a pretraining procedure learn to infer the desirable representation through the backward pass?

\vskip4pt
We aim to answer the three questions via the lens of exchangeability. Specifically, we observe that, as is the case in BERT and ViT, input tokens are often exchangeable since they already include positional encodings. The notion of exchangeability induces a latent variable model that is invariant to input sizes, which enables our theoretical analysis. 
\begin{itemize}[wide=0pt,topsep=1pt,itemsep=1pt,partopsep=0pt,parsep=0pt]
\item[-] To answer (a) on representation, we establish the existence of a sufficient and minimal representation of input tokens. In particular, such a representation instantiates the posterior distribution of the latent variable (or ``concept'') given input tokens, which plays a central role in predicting output labels and solving downstream tasks.
\item[-] To answer (b) on inference, we prove that attention with the desired parameter infers the latent posterior up to an approximation error, which is decreasing in input sizes. In detail, we quantify how attention approximates the conditional mean of the value given the key, which characterizes how it performs relational inference over long sequences. 
\item[-] To answer (c) on learning, we prove that both supervised and self-supervised objectives allow empirical risk minimization to learn the desired parameter up to a generalization error, which is independent of input sizes. Particularly, in the self-supervised setting, we identify a condition number that is pivotal to solving downstream tasks.
\end{itemize}
Our theoretical analysis gives a complete characterization of the attention mechanism as a ``greybox'' design, which unifies the handcrafted architecture induced by the latent variable model (``whitebox'') and the learnable parameter estimated from data (``blackbox'') with provable approximation, generalization, and optimization guarantees. 

\end{abstract}

\section{Introduction} 
Transformers are the state-of-the-art architecture for a variety of tasks in natural language processing \citep{vaswani2017attention}, computer vision \citep{dosovitskiy2020image}, and multimodal generation \citep{ramesh2021zero}. At the core of their significant empirical successes is the attention mechanism, which is defined by a computation graph for the forward pass. In particular, the computation graph performs a specific form of message passing across input tokens \citep{bronstein2021geometric}. It is commonly believed that the attention mechanism is capable of handling long sequences and performing relational inference (or ``inductive reasoning''), which appears to be the key advantage of transformers. However, the intuitive understanding lacks a quantitative justification, which leaves many intriguing questions open:

\begin{itemize}[wide=0pt,topsep=2pt,itemsep=2pt,partopsep=0pt,parsep=0pt]
\item[(a)] What makes a desirable representation? Ideally, the desirable representation of input tokens is sufficient and minimal in the sense that it preserves all relevant information for predicting output labels or solving downstream tasks (sufficiency) while it neglects all irreverent information (minimality). However, we lack a quantitative definition of sufficiency and minimality, which requires a probabilistic model. 

\item[(b)] How does the attention mechanism infer the desirable representation within the forward pass? Intuitively, the attention mechanism is defined by a computation graph that resembles kernel smoothing or kernel regression \citep{shawe2004kernel} for predicting the value given the key. However, we lack a formal characterization of what function class the attention mechanism parameterizes or approximates. Also, it remains unclear why the specific form of message passing produces the desirable representation of input tokens, that is, one with sufficiency and minimality. 

\item[(c)] How does a pretraining procedure learn the desirable representation through the backward pass? Empirically, the pretraining procedure that minimizes empirical risks for predicting masked tokens \citep{devlin2018bert, dosovitskiy2020image, he2022masked} appears to succeed in the presence of long sequences. However, we lack a theoretical justification of whether the pretraining procedure with the masked objective attains a desirable estimator that generalizes and why the generalization error does not appear to degrade for long sequences. In particular, it remains unclear to what degree the estimated representation facilitates solving downstream tasks. 
\end{itemize}

In this paper, we answer the three questions via the lens of exchangeability. The key observation is that, as is the case in BERT \citep{devlin2018bert} and ViT \citep{dosovitskiy2020image}, input tokens are exchangeable since they include positional encodings. In other words, the joint distribution of input tokens, e.g., vector embeddings of words in a paragraph or patches in an image with positional encodings, remains the same upon permuting their orders. Meanwhile, the attention mechanism and entrywise feedforward neural networks preserve the notion of exchangeability throughout all transformer layers. By the de Finetti Theorem \citep{de1937prevision}, the notion of exchangeability induces a latent variable model that is invariant to input sizes. Unlike classical Bayesian settings, where the latent variable model is defined across many data points, ours is defined over input tokens within one data point (in an ``in-context'' manner), which captures a fine-grained structure of interactions as relational inductive biases \citep{battaglia2018relational}. The latent variable model enables our theoretical analysis, which is summarized in the following:

\begin{itemize}[wide=0pt,topsep=2pt,itemsep=2pt,partopsep=0pt,parsep=0pt]
\item[-] To answer (a) on representation, we establish the existence of a sufficient and minimal representation of input tokens based on the latent variable model, which is induced by exchangeability. In particular, we leverage the latent variable model to define sufficiency and minimality following the factorization theorem and the sufficiency principle \citep{fisher1922mathematical}. Moreover, we prove that the posterior distribution of the latent variable given input tokens is a sufficient and minimal representation, which plays a central role in predicting output labels and solving downstream tasks. Intuitively, the latent variable instantiates the ``concept'' of a paragraph or an image, which is ``summarized'' over words or patches. In detail, the ``summarization'' process is formalized by the mapping from input tokens to the posterior distribution of the latent variable, that is, inferring the ``concept'' in a Bayesian manner within the forward pass. 

Given the answer to (a), which defines the desirable representation as the latent posterior, it remains unclear how to parameterize or approximate the latent posterior, which is addressed by the answer to (b).


\item[-] To answer (b) on inference, we prove that the attention mechanism with the desired parameter infers the latent posterior up to an approximation error, which is decreasing in input sizes. In particular, we prove that a specific parameterization of the latent posterior yields a variant of the attention mechanism based on kernel conditional mean embedding (CME), namely the CME attention, which infers the conditional mean of the value given the key. Here the value and the key (or the query) are obtained from a parameterized transformation of input tokens, where the unknown parameter requires learning. 

Although the CME attention recovers the latent posterior for any input sizes, it differs from the commonly used softmax attention by a normalization matrix. To this end, we prove that the CME attention and the softmax attention are equivalent at the infinite limit of input sizes by drawing a connection to nonparametric conditional density estimation. In other words, the softmax attention recovers the latent posterior up to an approximation error that is decreasing in input sizes, which characterizes how it performs relational inference over long sequences. As byproducts, we justify the necessity of multiple attention heads in transformers and provide a causal interpretation of the inferred representation through instrumental variables. 

Given the answer to (b), which quantifies the approximation error for the latent posterior, it remains unclear how to learn the desired parameter of the attention mechanism, which is addressed by the answer to (c).




\item[-] To answer (c) on learning, we prove that both supervised and self-supervised objectives allow empirical risk minimization to learn the desired parameter up to a generalization error, which is independent of input sizes. In particular, through maximum likelihood estimation, we establish the connection between the latent posterior and the masked objective, which is defined by the empirical risk for predicting masked tokens. 

Moreover, we prove that the global minimizer of the masked objective attains a generalization error that is independent of input sizes, which justifies why transformers allow long sequences. Our proof exploits the invariance and equivariance of the attention mechanism and entrywise feedforward neural networks, which deviates from most existing analyses of the generalization error. Particularly, in the self-supervised setting, e.g., as in MAE \citep{he2022masked}, we identify a condition number that is pivotal to solving downstream tasks. Intuitively, the condition number quantifies the amount of information that is transferred from the pretraining task to a new task. 

Meanwhile, in the overparameterized regime, we prove that any stationary point of the masked objective is almost globally optimal when the attention mechanism and entrywise feedforward neural networks have sufficient expressive power. As a result, stochastic gradient descent finds the global minimizer of the masked objective, which generalizes as discussed above. 

Combining the above analysis of the approximation, generalization, and optimization errors in the answer to (a)-(c), we provide a complete characterization of the attention mechanism. 



\end{itemize}
\vskip4pt
\noindent{\bf Contribution:} In summary, our theoretical contribution is threefold:
\begin{itemize}[wide=0pt,topsep=2pt,itemsep=2pt,partopsep=0pt,parsep=0pt]
\item[(i)] We identify a general principle for parameterizing function classes and constructing learning objectives based on latent posterior inference, which requires a minimal assumption of data. In contrast to classical learning paradigms, the latent variable model is defined over input tokens within one data point, which captures relational inductive biases.


\item[(ii)] We recover the attention mechanism from a specific parameterization of latent posterior inference based on kernel conditional mean embedding and nonparametric conditional density estimation. In particular, we demonstrate how the attention mechanism combines the handcrafted architecture, which is induced by latent posterior inference, and the learnable parameter, which determines the kernel function.


\item[(iii)] We characterize the approximation, generalization, and optimization errors for estimating the learnable parameter of the attention mechanism through minimizing the masked objective. In particular, we prove that input sizes do not degrade the approximation and generalization errors, which justifies why transformers allow long sequences. 
\end{itemize}

\vskip4pt
\noindent{\bf Discussion:}
Our theoretical analysis casts the attention mechanism as a ``greybox'' approach to modeling, that is, it combines the handcrafted architecture, which is coined by a probabilistic model over input tokens within one data point (``whitebox''), and the learnable parameter, which is estimated in an end-to-end manner through empirical risk minimization (``blackbox''). It is worth mentioning that our theoretical analysis studies the class of transformers like BERT \citep{devlin2018bert} and ViT \citep{dosovitskiy2020image} (``encoder-only''), which does not exploit the autoregressive structure as in GPT \citep{brown2020language} (``decoder-only''). On the other hand, the general principle identified in (i) is applicable to other probabilistic models like hidden Markov models or general graphical models over trees and grids, which motivates other principled architectures beyond the attention mechanism. We leave it as a future direction.

\subsection*{Related Works} 

\noindent{\bf Transformers and Attention.} The pioneering work  \citep{vaswani2017attention} proposes transformers for the first time and highlights the key role of the attention mechanism. Subsequently, there are a vast body of works that propose various transformer architectures and different pretraining paradigms. See, e.g., \cite{ devlin2018bert, radford2018improving, radford2019language, dai2019transformer, brown2020language, dosovitskiy2020image, he2022masked} and the references therein. Transformers demonstrate significant  empirical successes in natural language processing \citep{wolf2020transformers}, computer vision \citep{dosovitskiy2020image}, protein structure prediction \citep{jumper2021highly}, and sequential decision making \citep{chen2021decision}. Our work provides a theoretical justification of transformers and the attention mechanism, that is, how a latent variable model induced by exchangeability allows us to derive transformer architectures and pretraining paradigms in a principled manner.



\vskip4pt
\noindent{\bf Analysis of Transformers and Attention.} Our work is related to a recent line of works that analyze transformers and the attention mechanism \citep{tsai2019transformer, vuckovic2020mathematical, hron2020infinite, yang2020tensor, yang2021tensor, edelman2021inductive, wei2021statistically, xie2021explanation, malladi2022kernel, garg2022can, zhang2022unveiling}.
Specifically, 
\cite{tsai2019transformer} demonstrate that the attention mechanism can be viewed as a kernel smoother over input tokens.
\citet{vuckovic2020mathematical} establish the Lipschitz continuity of transformers via the lens of interacting particle systems.  
\cite{hron2020infinite, yang2020tensor, yang2021tensor, malladi2022kernel} characterize the infinite-width limit of transformers under the framework of neural tangent kernels \citep{jacot2018neural}. Among them, \cite{malladi2022kernel} demonstrate that neural tangent kernels can capture the parameter update in the fine-tuning phase.
\cite{edelman2021inductive} prove that transformers can represent a sparse function of input tokens and establish a sample complexity that scales logarithmically in input sizes. \citet{wei2021statistically} characterize the approximation and generalization errors for learning a Turing machine with transformers. \cite{xie2021explanation} prove that transformers can infer a latent variable (or ``concept'') assuming that the data distribution is a mixture of hidden Markov models. \cite{garg2022can} demonstrate that transformers can learn to perform linear predictions within one data point (in an ``in-context'' manner). \cite{zhang2022unveiling} evaluate the empirical performance of transformers for learning equality and group operations.

Our work provides a complete characterization of the representation, inference, and learning aspects of the attention mechanism via the lens of exchangeability and latent variable models, which requires a minimal assumption on the data distribution (exchangeability). Specifically, in comparison with \cite{tsai2019transformer}, we demonstrate that the attention mechanism not only parameterizes nonparametric conditional density estimation but also approximates kernel conditional mean embedding, which infers the conditional mean of the value given the key. Moreover, we invoke the latent variable model induced by exchangeability to justify the attention mechanism as a specific parameterization of latent posterior inference. Meanwhile, we leverage the latent variable model to derive the common choice of both supervised and self-supervised objectives, e.g., the masked objective.
In comparison with \cite{xie2021explanation}, we do not assume that the data distribution takes a specific form (a mixture of hidden Markov models) or the latent posterior is given a priori by a specific parameter (no learning required). Instead, we prove that the attention mechanism is capable of instantiating latent posterior inference up to an approximation error and the masked objective allows us to learn to infer the latent posterior up to the generalization and optimization errors. Also, it is worth mentioning that \cite{xie2021explanation} focus on the class of transformers like GPT \citep{brown2020language} (``decoder-only''), while we focus on the class of transformers like  BERT \citep{devlin2018bert} and ViT \citep{dosovitskiy2020image} (``encoder-only'').
In comparison to \cite{edelman2021inductive}, we exploit the invariance and equivariance of transformers and establish a generalization error that is independent of input sizes.  

%

\vskip4pt
\noindent{\bf Generalization of  Deep Neural Networks.} Our work is related to the vast body of works that analyze the generalization error of deep neural networks. See, e.g., \citet{jiang2019fantastic,valle2020generalization} for a comprehensive introduction. However, most of them do not exploit invariance and equivariance. As a result, a direct application of such results yields a vacuous bound as input sizes increase. On the other hand, \citet{sokolic2017generalization, sannai2021improved, elesedy2021provably, zhu2021understanding} establish a generalization error that captures the improvement from invariance and equivariance, which, however, is not applicable to the attention mechanism. Our theoretical analysis of the generalization error follows the framework of \citet{bartlett2017spectrally}, which stems from \cite{bartlett1996valid,bartlett2002rademacher}. In addition, the concurrent work \citep{zhang2022relational} provides a PAC-Bayes analysis of the generalization error of the attention mechanism in the context of multiagent reinforcement learning. 

\vskip4pt
\noindent{\bf Optimization of Deep Neural Networks.} Our work is built on the vast body of works that analyze the optimization error of deep neural networks \citep{allen2019learning, allen2019convergence, allen2019convergence2, arora2019fine, du2018gradient, du2019gradient, zhang2019learning, zou2018stochastic, zou2019improved,allen2019convergence, cao2019generalization, li2018learning, chizat2019lazy,mei2018mean, mei2019mean, rotskoff2018parameters, nguyen2019mean, sirignano2020mean}. Most of them focus on overparameterized neural networks in the neural tangent kernel \citep{jacot2018neural} or mean-field regime \citep{mei2018mean}. Our work analyzes the optimization error in the neural tangent kernel regime, which is similar to \cite{malladi2022kernel}. Meanwhile, it is worth mentioning that our theoretical analysis of the approximation and generalization errors is not restricted to the neural tangent kernel regime.


\vskip4pt
\noindent{\bf Invariance and Equivariance in Deep Neural Networks.} Our work is related to a recent line of works on deep neural networks with invariance or equivariance with respect to permutations and other group operations. See, e.g.,    \citet{scarselli2008graph, zaheer2017deep,lee2019set, keriven2019universal, romero2020group, bloem2020probabilistic, hutchinson2021lietransformer, satorras2021n, kossen2021self} and the references therein. 
Also, see \citet{valle2020generalization, han2022geometrically} for a detailed survey. In comparison, we exploit the latent variable model induced by exchangeability to provide a complete
characterization of the representation, inference, and learning aspects of the attention mechanism.

\section{Preliminary}

\noindent{\bf Notations.}
We denote by $[L]$ the index set $\{1, 2, \ldots, L\}$ for any $L \in \NN_+$. For any vector $v \in \RR^L$, we denote by $\softmax(v) = (\exp(v^\ell) / (\sum_{\ell' = 1}^L \exp(v^{\ell'})))_{\ell \in [L]} \in \RR^L$ the softmax function.
We denote by $\|\cdot\|_2$ the spectral norm, which becomes the $\ell_2$-norm when it operates on a vector. We denote by $\|\cdot\|_{\rm F}$ the Frobenius norm.
For any $d \in \NN_+$, we denote by $\SSS^{d-1}  = \{x \in \RR^d \given \norm{x}_2 = 1\}$ the $(d-1)$-dimensional unit sphere.

\vskip4pt

\noindent{\bf Reproducing Kernel Hilbert Space.} Let $\cH_x$ be a Hilbert space over a domain $\fX$, which contains functions $f: \fX\rightarrow \RR$ and is equipped with the inner product $\inp{\cdot}{\cdot}_{\cH_x}$. We say that $\cH_x$ is a reproducing kernel Hilbert space (RKHS) with the kernel function $\fk: \fX \times \fX \rightarrow \RR$ if we have the reproducing property that $
\inp{f}{\fk(x, \cdot)}_{\cH_x} = f(x)$ for any $f \in \cH_x$ and $ x\in \fX.$
An RKHS $\cH_x$ is associated with a feature mapping $\phi: \fX\rightarrow \ell_2$ such that $\fk(x, x') = \phi(x)^\top \phi(x')$ for any $x, x'\in \fX$ \citep{muandet2016kernel}. Here we denote by $\ell_2$ the space of all square-summable series.

\vskip4pt

\noindent{\bf Attention Mechanism.} For an input sequence $X = \{x^\ell\}_{\ell \in [L]}$ with the input tokens $x^\ell \in \RR^d$, we consider the key matrix $K \in \RR^{L \times d_\tp}$ and the value matrix $V \in \RR^{L \times d}$ defined as
\begin{align}
	K &= (k^1, \ldots, k^L)^\top = \bigl( k_\theta(x^1), \ldots,  k_\theta(x^L)\bigr)^\top \in \RR^{L \times d_\tp}, \nonumber\\ 
	V &= (v^1, \ldots, v^L)^\top = \bigl( v_\theta(x^1), \ldots, v_\theta(x^L)\bigr)^\top \in \RR^{L \times d}. \nonumber
\end{align}
Here $k_\theta: \RR^d \rightarrow \RR^{d_\tp}$ and $v_\theta: \RR^d \rightarrow \RR^{d}$ map the $\ell$-th input token $x^\ell$ to the key $k^\ell$ and the value $v^\ell$, respectively, where $\theta \in \Theta$ is the the learnable parameter.
For any query $q\in\RR^{d_\tp}$, we define the attention mechanism as follows,
\#\label{eq:attn}
\att(q, K, V) = V^\top \Norm\bigl(\fk(K, q)\bigr) \in \RR^{d},
\#
where $\fk: \RR^{d_\tp} \times \RR^{d_\tp} \rightarrow \RR$ is a kernel function and we write $\fk(K, q) = (\fk(k^\ell, q))_{\ell \in L} \in \RR^{L}$. Here we denote by $\Norm: \RR^{L} \rightarrow \RR^{L}$ a normalization mapping. 

A common example of the attention mechanism is the softmax attention \citep{vaswani2017attention}, where the kernel function is the exponential kernel $\fk_{\texp}(q, k) =  \exp(q^\top k / \gamma)$ with a fixed $\gamma > 0$ and the normalization mapping is the following softmax normalization,
\begin{align*}
	\Norm_\sm\bigl(\fk(K, q)\bigr) = \bigl(\one^\top \fk(K, q)\bigr)^{-1} \cdot\fk(K, q) .
\end{align*}
The attention mechanism in \eqref{eq:attn} with the exponential kernel and the softmax normalization is the softmax attention \citep{vaswani2017attention}, which takes the following form,
\begin{align*}
	\att_{\sm}(q, K, V) = V^\top \Norm_\sm\bigl(\fk_{\texp}(K, q)\bigr) = \softmax(Kq/\gamma).
\end{align*}



\section{Representation, Inference, and Estimation \\ via Latent Variable Model}
\label{sec:exchange}



\textbf{From Exchangeability to Latent Variable Model.}
We consider the input sequence $X = \{x^\ell\}_{\ell\in[L]}$, where $x^\ell \in \RR^d$ is an input token and $L\in \NN_+$ is the sequence length. In natural language processing (NLP), such a sequence consists of embeddings of words in a paragraph, while in computer vision (CV), such a sequence consists of embeddings of patches in an image.

As is the case in BERT \citep{devlin2018bert} and ViT \citep{dosovitskiy2020image}, the input sequence is exchangeable since it includes positional encodings. Specifically, we say that a random variable sequence $\{x^\ell\}_{\ell \in \NN_+}$ is exchangeable if and only if it holds for any sequence length $L \in \NN_+$ and any index permutation $\pi: [L] \rightarrow [L]$ that
\$
\PP(x^1, x^2, \ldots, x^L) = \PP(x^{\pi(1)}, x^{\pi(2)}, \ldots, x^{\pi(L)}).
\$
In other words, permuting the index order within the random variable sequence does not affect its joint distribution. The following proposition states that the exchangeability of a random variable sequence induces a latent variable.

\begin{proposition}[de Finetti Representation Theorem \citep{de1937prevision}]
	\label{th:definetti}
	Let $\{x^\ell\}_{\ell \in \NN_+}$ be an exchangeable sequence. Then, there exists a latent variable $z$ such that for any sequence length $L \in \NN_+$,
	\begin{align*}
		\PP(x^1, \ldots, x^L) &= \int \prod_{\ell = 1}^L \PP(x^\ell \given z) \cdot \PP(z) \ud z,  \\
		\PP(x^\ell \given x^1, \ldots, x^{\ell-1}, x^{\ell + 1}, \ldots, x^L) &= \int \PP(x^\ell\given z) \cdot \PP(z\given x^1, \ldots, x^{\ell-1}, x^{\ell + 1}, \ldots, x^L) \ud z, \quad \forall \ell \in [L].
	\end{align*}
\end{proposition}

We remark that Proposition 
\ref{th:definetti} requires an infinite-length exchangeable sequence. 
Up to an approximation error, a finite-length exchangeable sequence also induces a latent variable \citep{diaconis1980finite}. 
In what follows, we consider the former case where the input sequence includes positional encodings and is thus exchangeable \citep{devlin2018bert, dosovitskiy2020image}. See Figure \ref{fig:exchange} for an illustration of the exchangeability. 

\begin{figure}[H]
	\centering
	\includegraphics[width=3.8in]{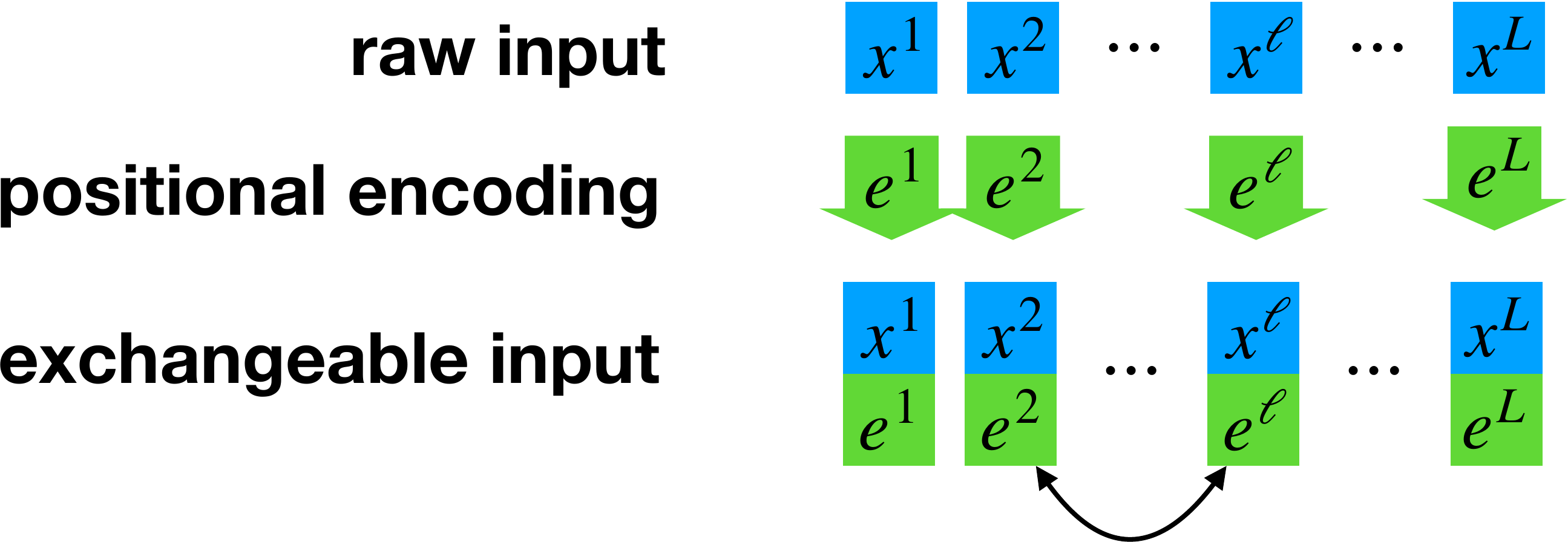}
	\caption{The input sequence (the raw version without positional encodings) becomes exchangeable with  positional encodings. In practice, the positional encoding is incorporated in an additive manner (instead of concatenation). \label{fig:exchange}}
\end{figure}

Proposition~\ref{th:definetti} guarantees the existence of a latent variable, which forms the basis of our theoretical analysis. See Figure \ref{fig:definetti_x} for an illustration.
Intuitively, the latent variable can be viewed as the ``concept'' of the input sequence, which is ``summarized'' over words or patches. For instance, in NLP, the latent variable instantiates the ``meaning'' of a paragraph, while in CV, the latent variable instantiates the ``theme'' of an image. In particular, the latent posterior $\PP(z\given X)$ plays a key role in solving downstream tasks \citep{song2014nonparametric, xie2021explanation}, as it is a desired representation of the input sequence $X$. See Figure \ref{fig:definetti_y} for an illustration.
In the following lemma, we prove that the latent posterior $b_z(X) = \PP(z = \cdot \given X)$ is a minimal sufficient statistic \citep{fisher1922mathematical}.

\begin{lemma}[Minimal Sufficiency of Latent Posterior]
	\label{lem:sufficiency}
	Let $z$ be the latent variable induced by the exchangeability of the input sequence $X$.
	The latent posterior $b_z(X) = \PP(z = \cdot \given X)$ is a minimal sufficient statistic of the input sequence $X$ for the latent variable $z$. Meanwhile, for any target variable $y$ that is independent of the input sequence $X$ conditioning on the latent variable $z$, we assume the invertibility of the operator $\cT$ defined by
	\#\label{eq:def-cT}
	(\cT f)(y) = \int \PP(y \given z) f(z) \ud z.
	\#
	Then, the latent posterior $b_z(X) = \PP(z = \cdot \given X)$ is a minimal sufficient statistic of the input sequence $X$ for the target variable $y$. 
\end{lemma}

\begin{proof}
	See \S\ref{sec:pf-sufficiency} for a detailed proof.
\end{proof}

\noindent{\bf From Latent Variable Model to Learning Objectives.}
In what follows, we consider the prediction task in BERT \citep{devlin2018bert} and ViT \citep{dosovitskiy2020image}. 
Let $y$ be the target variable and $X= \{x^\ell\}_{\ell \in [L]}$ is the input sequence.  In particular, in self-supervised learning (BERT), the target variable $y$ is a masked token of the input sequence, while in supervised learning (ViT), $y$ is the unknown label corresponding to the class encoding. 
We remark that in ViT, the unknown label $y$ corresponds to the masked token in BERT, while the input class encoding corresponds to the mask in BERT.
In both cases, the concatenation $\{y, x^1, \ldots, x^L\}$ is treated as an exchangeable sequence since it includes the positional encodings.
By Proposition \ref{th:definetti}, we have
\begin{align}
	\label{eq:y-exchange}
	\PP(y \given X) = \int \PP(y \given z) \cdot  \PP(z \given X) \rd z,
\end{align}
where $z$ is the latent variable induced by the exchangeability of $X$. See Figure \ref{fig:definetti} for an illustration.

In what follows, we treat $y$ as a target variable that satisfies \eqref{eq:y-exchange}, which specifies that $y$ is independent of the input sequence $X$ conditioning on the latent variable $z$. By Lemma \ref{lem:sufficiency}, the latent posterior $b_z(X)$ is a minimal sufficient statistic of $X$ for $y$. In other words, the latent posterior $b_z(X)$ is a desired representation of the input sequence $X$. 
According to \eqref{eq:y-exchange}, the prediction of the target variable $y$ from the input sequence $X$ (forward pass) takes two implicit steps: i) the inference of the latent posterior $\PP(z \given X)$, and ii) the prediction of $y$ based on the generative distribution $\PP(y\given z)$ integrated with the  latent posterior $\PP(z\given X)$. 

\begin{figure}[htbp]
	\centering
	\begin{subfigure}[t]{0.265\textwidth}
		\centering
		\includegraphics[width=\textwidth]{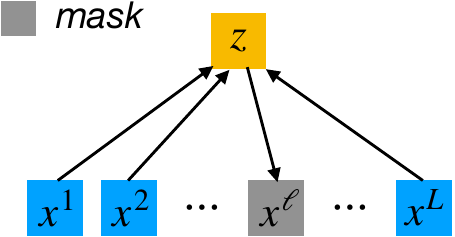}
		\caption{Prediction of a masked token $x^\ell$. \label{fig:definetti_x}}
	\end{subfigure}
	\hspace{30mm}
	\begin{subfigure}[t]{0.3 \textwidth}
		\centering
		\includegraphics[width=\textwidth]{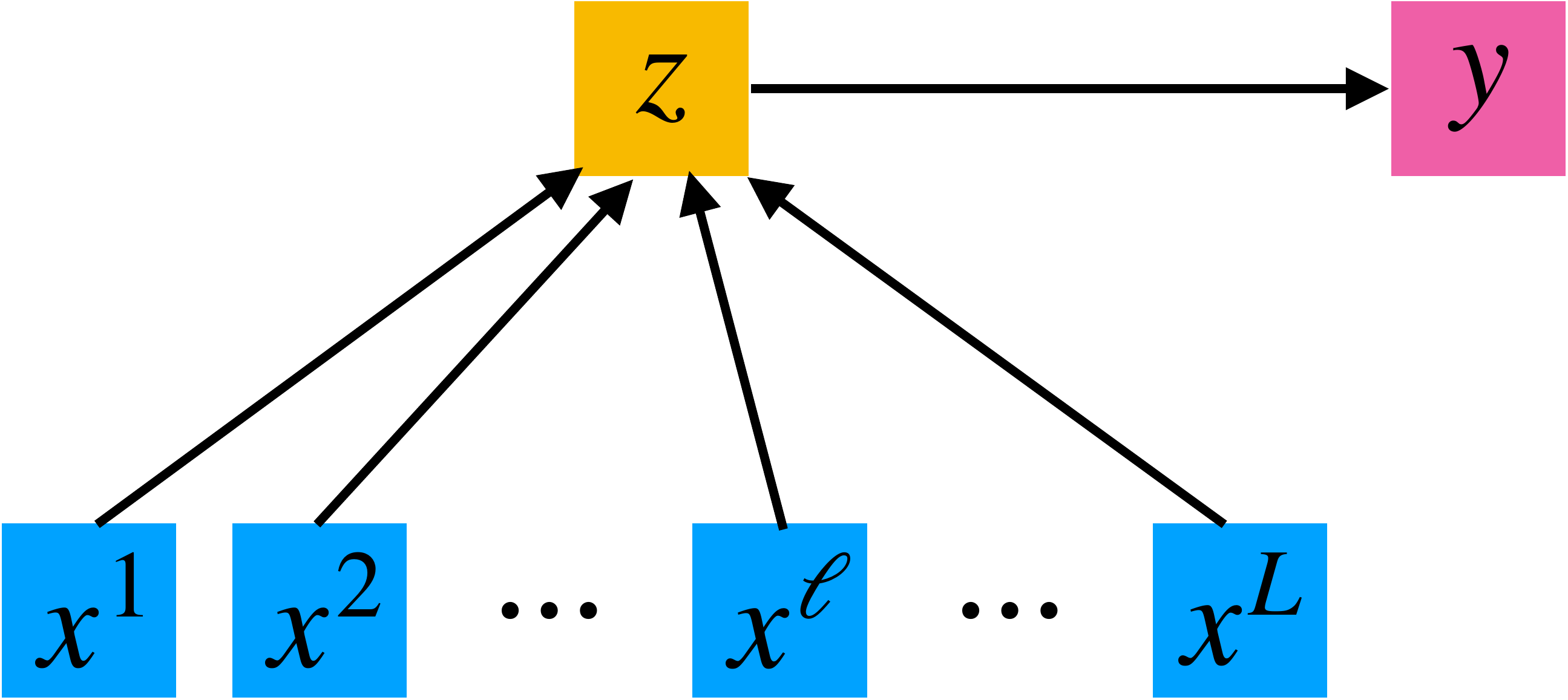}
		\caption{Prediction of the target variable $y$, which can be viewed as a masked token. \label{fig:definetti_y}}
	\end{subfigure}
	\caption{The forward pass for the prediction of the masked token $x^\ell$ and the target variable $y$. The prediction of $y$ takes two steps: i) the inference of the latent posterior $\PP(z \given X)$, and ii) the prediction of $y$ based on the generative distribution $\PP(y\given z)$ integrated with the  latent posterior $\PP(z\given X)$.}
	\label{fig:definetti}
\end{figure}


To construct the learning objective, we consider the distribution of the target variable $y$ conditioning on the input sequence $X$ and parameterize it by $\PP_\theta(y \given X)$, where $\theta$ is the learnable parameter.
Given a dataset $\cD_n = \{(X_i, y_i)\}_{i\in [n] }$, where $X_i$ is the $i$-th input sequence and $y_i$ is the $i$-th target variable, the maximum likelihood estimation (MLE) objective takes the following form,
\begin{align}\label{eq:mle1}
	& \max_\theta \hat\EE_{(X, y) \sim \cD_n} \bigl[\log \PP_\theta (y \given X) \bigr]\\
	&\quad  = \hat\EE_{(X, y) \sim \cD_n} \biggl[ \log \int \PP_{\theta}(y \given z) \PP_{\theta}(z \given X)  \rd z\biggr]. \notag 
\end{align}
We define $\hat\EE_{(X, y) \sim \cD_n}[\,\cdot \,]$ as the empirical expectation with respect to the dataset $\cD_n$. 


\subsection{Preliminary Finite-Dimensional Example} \label{sec:fin-eg}

\vskip4pt

{\noindent \bf Latent Variable Model.} We provide a finite-dimensional Gaussian-distributed example to illustrate the latent variable model and the MLE objective.
Specifically, we consider the setting with the input sequence $X = \{x^\ell\}_{\ell \in [L]}$ and the target variable $y$, where $x^\ell \in \RR^{d}$ and $y \in \RR^d$. For the input sequence $X$, we consider the following example of the latent variable model in \eqref{eq:y-exchange},
\begin{align}
	\label{eq:lvm-fin}
	r^\ell = z c^\ell + \epsilon^\ell,   \qquad \text{where} \quad c^\ell = c_*(x^\ell), \quad r^\ell = r_*(x^\ell), \quad \forall \ell \in [L]. 
\end{align}
Here $z \in \RR^{d_r \times d_c}$ is the latent variable induced by the exchangeability of the input sequence, $c^\ell \in \RR^{d_c}$ and $r^\ell \in \RR^{d_r}$ are the covariate and response, respectively, which are determined by two unknown functions $c_*: \RR^{d} \rightarrow \RR^{d_c} $ and $r_*: \RR^d \rightarrow \RR^{d_r}$, and $\epsilon^\ell \sim N(0, \sigma^2 I)$ is the noise, which is independent of $c^\ell$.
In practice, the covariate $c^\ell$ instantiates the contextual information, while the response $r^\ell$ instantiates the semantic information.
We consider the prediction of the (unknown) target variable $y$ based on its (known) input mask $\mask$. In the self-supervised setting \citep{devlin2018bert}, $y$ is a masked token of the input sequence, while $\mask$ is the positional encoding. In the supervised setting \citep{dosovitskiy2020image}, $y$ is the label of the input sequence, while $\mask$ is the class encoding.
Specifically, corresponding to \eqref{eq:lvm-fin}, 
we consider the prediction model with $y = r^\mask$, 
where $r^\mask$ is the response  corresponding to the covariate $c^\mask$ of the input mask such that
\begin{align}
	\label{eq:lvm-y}
	y = r^\mask = z c^\mask + \epsilon,\qquad \text{where}\quad c^\mask = c_*^\tm(\mask).
\end{align}
Here $c_*^\tm: \RR^d \rightarrow \RR^{d_c}$ is an unknown function and $\epsilon \sim N(0, \sigma^2 I)$ is the noise, which is independent of $c^\mask$. For example, \eqref{eq:lvm-y} holds when we consider an exchangeable sequence $\{x^1, \ldots, x^L, x^{L+1}\}$ satisfying \eqref{eq:lvm-fin} for any $\ell \in [L+1]$ with the input sequence $X = \{x^1, \ldots, x^L\}$, $c^\mask = c^{L+1}$, and the target variable $ y = r^\mask = r^{L+1}$. In the next section, we consider an advanced infinite-dimensional example and show that $c^\mask$, $c^\ell$, and $r^\ell$ correspond to the query, the key, and the value in the attention mechanism, respectively.

Note that the regression model in \eqref{eq:lvm-fin} is a conditional model. Instead of modeling the conditional distribution of $y$ given $X$ as in \eqref{eq:y-exchange}, we model the conditional distribution of $y$ given $X$ and $\mask$. Recall that $y = r^\mask$. Corresponding to \eqref{eq:y-exchange}, the generative distribution takes the following form,
\begin{align}
	\label{eq:gen-dist}
	\PP(y \given \mask, z) \propto \exp\Bigl(-\norm[\big]{ y - zc_*^\tm(\mask) }_2^2 \big/ 2 \sigma^2 \Bigr). 
\end{align}
We take the Gaussian distribution $N(0, \lambda I)$ as the prior of $z$. By \eqref{eq:lvm-fin}, the latent posterior $\PP(z \given X)$ is a Gaussian distribution, which (approximately) takes the following form,
\begin{align}
	\label{eq:lat-post}
	\PP(z \given X) \propto \exp \Bigl(- \norm[\big]{z - \bar z(X)}_2^2 \big/ 2 \iota^2 \Bigr). 
\end{align}
Here the covariance of the latent posterior is approximated by $\iota^2 I$ and the mean $\bar z(X)$ of the latent posterior takes the following form,
\begin{align*}
	\bar z(X) = \EE[z \given X] = R^\top ( C C^\top + \lambda I)^{-1} C,
\end{align*}
where we define $C = (c^1, \ldots, c^L)^\top \in\RR^{L\times d_c}$ and $R = (r^1, \ldots, r^L)^\top \in\RR^{L \times d_r}$.
Combining \eqref{eq:gen-dist} and \eqref{eq:lat-post},  we obtain
\begin{align}
	\label{eq:arch0}
	\PP(y \given \mask, X) = \int \PP(y \given \mask, z) \cdot \PP(z \given X) \rd z \propto  \exp\Bigl( -\norm[\big]{y - \bar z(X) c_*^\tm(\mask) }_2^2 \big/ 2 \tilde \sigma^2 \Bigr), 
\end{align}
which corresponds to \eqref{eq:y-exchange},
Here we approximate the covariance of $y$ conditioning on $X$ and $\mask$ by $\tilde \sigma^2 I$, where $\tilde \sigma$ does not depend on $X$. We remark that \eqref{eq:arch0} is a form of Bayesian model averaging \citep{wasserman2000bayesian} within one data point.

\vskip4pt

\noindent{\bf Parameterization of Latent Variable Model.}
Recall that $c_*$, $r_*$, and $c^\tm_*$ in \eqref{eq:lvm-fin} and \eqref{eq:lvm-y} are unknown. 
We parameterize them with $c_\theta$, $r_\theta$, and $c^\tm_\theta$, where $\theta \in \Theta$ is a learnable parameter. With the ideal parameter $\theta^* \in \Theta$, it holds for any $\ell \in [L]$ that
\begin{align}\label{eq:inv-f-para}
	c_{\theta^*}(x^\ell) = c_*(x^\ell) =  c^\ell, \quad  r_{\theta^*}(x^\ell) =  r_*(x^\ell) = r^\ell, \quad c^\tm_{\theta^*}(\mask) =c^\tm_*(\mask) =  c^\mask.
\end{align}
By \eqref{eq:lat-post}, we parameterize the latent posterior $\PP(z \given X)$ as follows,
\begin{align}
	\label{eq:latent-post-para}
	\PP_{\theta}(z \given X) \propto \exp \Bigl(- \norm[\big]{z - \bar z_\theta(X)}_2^2 \big/ 2 \iota^2 \Bigr),
\end{align}
where $\bar z_\theta(X)$ is calculated as follows,
\begin{align}
	\bar z_\theta(X) =  r_\theta(X)^\top \bigl( c_\theta(X) c_\theta(X)^\top + \lambda I \bigr)^{-1} c_\theta(X). \nonumber
\end{align}
Here $c_\theta(X) = (c_\theta(x^1), \ldots, c_\theta(x^L))^\top \in \RR^{L\times d_c}$ and $r_\theta(X) = (r_\theta(x^1), \ldots, r_\theta(x^L))^\top \in \RR^{L\times d_r}$. By \eqref{eq:gen-dist}, we parameterize the generative distribution $\PP(y \given \mask, z)$ as follows,
\begin{align*}
	\PP_\theta(y \given \mask, z) \propto \exp\Bigl(-\norm[\big]{ y - z c^\tm_\theta(\mask) }_2^2 \big/ 2 \sigma^2 \Bigr).
\end{align*}
By \eqref{eq:arch0}, we define the conditional likelihood $\PP(y \given \mask, X)$ as follows,
\begin{align}
	\label{eq:arch00}
	\PP_\theta(y \given \mask, X) \propto \exp\Bigl( -\norm[\big]{y - \bar z_\theta(X) c_\theta^\tm(\mask) }_2^2 \big/ 2 \tilde \sigma^2 \Bigr). 
\end{align}

\vskip4pt

\noindent{\bf Training and Testing.}
In the training phase, given the dataset $\cD_n = \{(X_i, y_i)\}_{i \in [n]}$, we aim to maximize the MLE objective in \eqref{eq:mle1}. By \eqref{eq:arch00}, maximizing the MLE objective is equivalent to minimizing the mean-squared error as follows,
\begin{align}\label{eq:para}
	\min_\theta \hat \EE_{(X, y) \sim \cD_n}\Bigl[\norm[\big]{y - \bar z_\theta(X) c^\tm_\theta(\mask)}_2^2 \Bigr].
\end{align}
Note that the learnable parameter $\theta$ is estimated across different data points in the dataset $\cD_n$ through the backward pass, while the latent variable $z$ is inferred within one data point $(X_i, y_i)$ through the forward pass. We remark that by learning $\theta$, the model learns to perform Bayesian model averaging.
Suppose that we solve \eqref{eq:para} and obtain the estimator $\hat \theta$.
In the testing phase, given an input sequence $X_\dagger$ and an input mask $\mask_\dagger$, we predict the target variable $y_\dagger$ by maximizing the posterior of $y$,
\begin{align}
	\label{eq:predict}
	\hat y = \argmax_y \PP_{\hat\theta}(y \given  \mask_\dagger, X_\dagger) = \EE[r^\mask \given \mask_\dagger, X_\dagger] = \bar z_{\hat \theta}(X_\dagger) c^\tm_{\hat \theta}(\mask_\dagger).
\end{align} 
We remark that the learning process for the attention mechanism involves two aspects. In the forward pass, within one data point, we infer the latent posterior $\PP(z \given X)$ to predict the target variable $y$. In the backward pass, we estimate the learnable parameter $\theta$ across different data points. See Figure \ref{fig:learning} for an illustration.



\begin{figure}[H]
	\centering
	\includegraphics[width=5in]{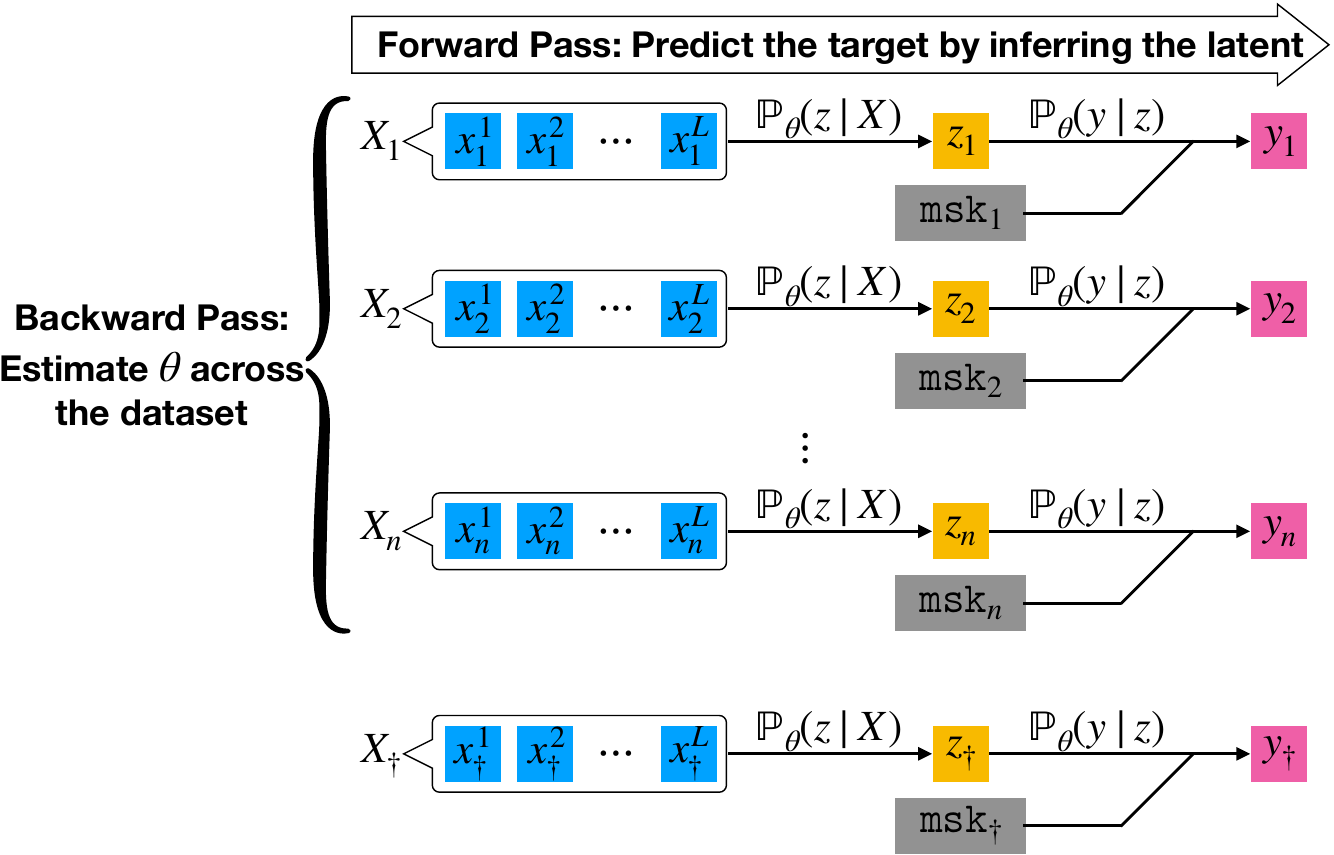}
	\caption{Forward pass: within one data point $(X, y)$, we infer the latent posterior $\PP_\theta(z \given X)$ by \eqref{eq:latent-post-para}. We predict $y_\dagger$ by $\hat y$ in \eqref{eq:predict}. Backward pass: across different data points in the dataset $\cD_n$, we estimate the learnable parameter $\theta$ by \eqref{eq:para}.  \label{fig:learning}}
\end{figure}

The finite-dimensional example illustrates a ``greybox'' approach to modeling, that is, it combines the handcrafted architecture in \eqref{eq:arch0} (``whitebox''), and the learnable parameter in \eqref{eq:para}, which is estimated in an end-to-end manner through empirical risk minimization (``blackbox'').
As shown in Figure \ref{fig:sl-ssl}, the forward pass first infers the latent variable $z$ and then utilizes the latent variable $z$ to predict the masked token (in the self-supervised setting) or the unknown label (in the supervised setting). Meanwhile, the backward pass estimatess the learnable parameter.
In the following section, we extend the finite-dimensional example to the infinite-dimensional setting, which recovers the attention mechanism. In particular, we demonstrate that the attention mechanism infers the latent posterior within a data point. Also, we show that the covariate corresponds to the query and key and that the response corresponds to the value  in the attention mechanism.

\begin{figure}[H]
	\centering
	\includegraphics[width=5in]{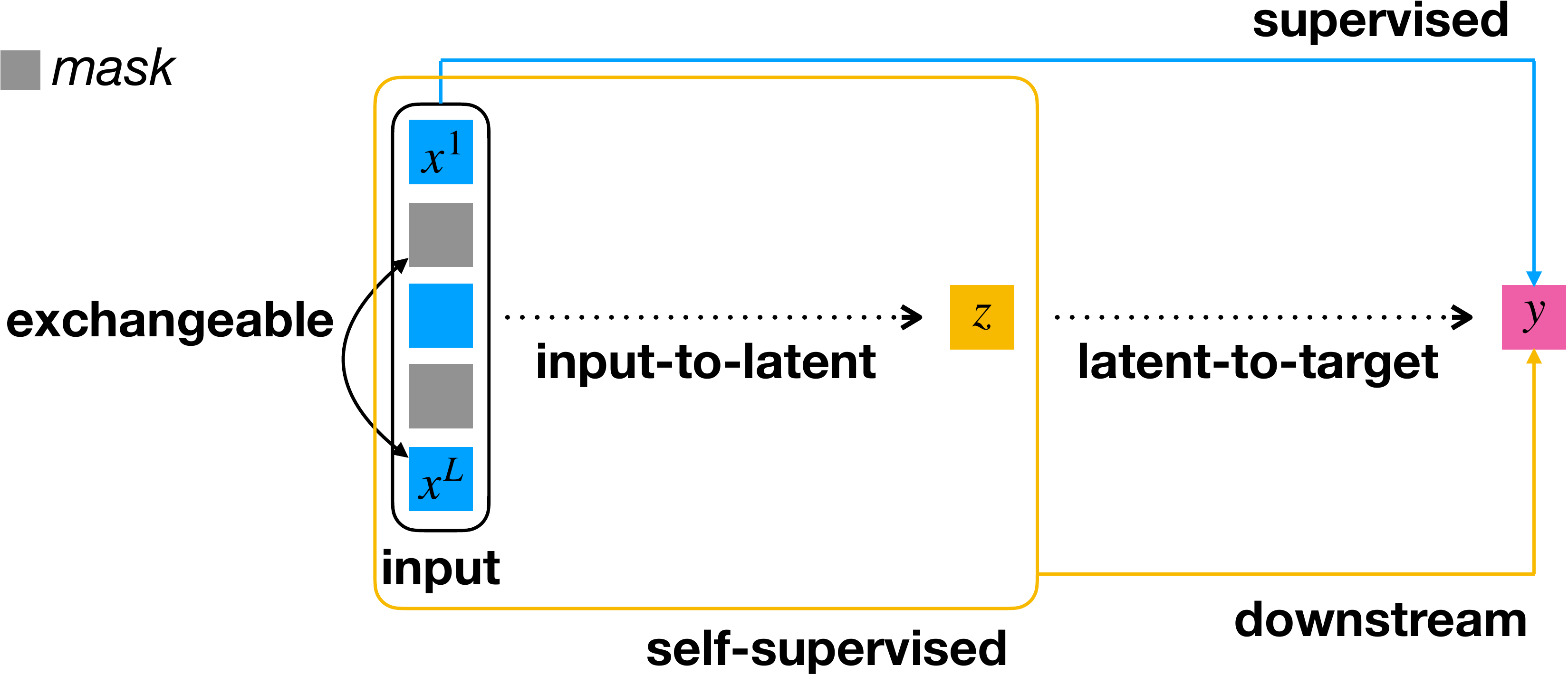}
	\caption{The forward and backward passes in transformers. Dotted arrows stand for forward passes (input$\rightarrow$latent$\rightarrow$target). Solid arrows stand for backward passes (training). Masks (grey tokens) are only used to illustrate the self-supervised setting (yellow box). \label{fig:sl-ssl}}
\end{figure}

\section{Attention as Latent Posterior Inference} \label{sec:attn-cme-kde}

In what follows, we demonstrate how the attention mechanism performs latent posterior inference for the latent variable model, which is induced by the exchangeability of the input sequence.
In \S\ref{sec::Gaussian_Posterior}, we extend the finite-dimensional example in \S\ref{sec:exchange} to an RKHS to induce a variant of the softmax attention, namely, the conditional mean embedding (CME) attention. In particular, we prove that it infers the latent posterior in the forward pass. In \S\ref{sec::KDE_intuition}, we prove that the softmax attention has the same limit as the CME attention when the sequence length goes to infinity, which implies that the softmax attention approximately infers the latent posterior.

\subsection{Attention as Kernel Conditional Mean Embedding}
\label{sec::Gaussian_Posterior}

\noindent{\bf Advanced Infinite-Dimensional Example.}
We present an infinite-dimensional version of the preliminary example in \S\ref{sec:fin-eg}, which motivates us to study the CME attention. Similarly to \eqref{eq:lvm-fin}, we consider the following model for the input sequence $X = \{x^\ell\}_{\ell \in [L]}$ with input token $x^\ell \in \RR^d$,
\begin{align}\label{eq:structure-inf}
	r^\ell &= z \phi(c^\ell) + \epsilon^\ell,  \qquad \text{where}\quad c^\ell = c_*(x^\ell), \quad r^\ell = r_*(x^\ell), \quad \forall \ell \in [L].
\end{align}
Here $c^\ell \in \RR^{d_\tp}$ and $r^\ell \in \RR^{d}$ are the covariate and the response, respectively, which are determined by two unknown functions $c_*: \RR^{d} \rightarrow \RR^{d_\tp} $ and $r_*: \RR^d \rightarrow \RR^{d}$, $\phi: \RR^{d_\tp}\rightarrow \cH_c$ is the feature mapping of the RKHS $\cH_c$, $z: \cH_c \rightarrow \RR^{d}$ is a linear mapping, which is viewed as the latent variable induced by the exchangeability of the input sequence $X$, and $\epsilon^\ell \sim N(0, \sigma^2 I)$ is the Gaussian noise, which is independent of the covariate $c^\ell$. Note that $\phi(c^1)^\top \phi(c^2) = \fk(c^1, c^2)$ for any $c^1, c^2 \in \RR^{d_\tp}$, where $\fk: \RR^{d_\tp} \times \RR^{d_\tp} \rightarrow \RR$ is the kernel function of the RKHS $\cH_c$.
A common example is the Gaussian radial basis function (RBF) kernel $\fk_\rbf(q, k) = \exp( - \norm{q - k}_2^2 / 2 \gamma)$ with $\gamma > 0$.
Similarly to \eqref{eq:lvm-y}, the (unknown) target variable $y$ is determined by its (known) input mask $\mask$, which satisfies 
\begin{align}
	\label{eq:lvm-y-inf}
	y = r^\mask = z\phi(c^\mask) + \epsilon, \qquad \text{where} \quad c^\mask = c_*^\tm (\mask).
\end{align}
Here we denote by $c^\mask$ and $r^\mask$ the covariate and the response corresponding to the input $\mask$, respectively,  $c_*^\tm : \RR^{d} \rightarrow \RR^{d_\tp}$ is an unknown function, and $ \epsilon \sim N(0, \sigma^2 I)$ is a Gaussian noise, which is independent of $c^\mask$.
To simplify the presentation, we view the RKHS $\cH_c$ as a vector space $\RR^{d_\phi}$ with $d_\phi = \infty$. Correspondingly, we view the latent variable $z$ as a matrix in $\RR^{d\times d_\phi}$. We present a rigorous characterization of $z$ in \S\ref{sec:gp-regress} based on Gaussian process regression.
Similarly to \eqref{eq:inv-f-para}, we parameterize $c_*$, $r_*$, and $c^\tm_*$ by $c_\theta$, $r_\theta$, and $c^\tm_\theta$, where $\theta \in \Theta$ is a learnable parameter. Similarly to \eqref{eq:predict},  we predict $r^\mask$ in the forward pass via
\begin{align}
	\label{eq:cme-att}
	\hat r & = \EE[r^\mask \given \mask, X] \nonumber \\
	&= r_\theta(X)^\top \Bigl( \phi\bigl(c_\theta(X)\bigr) \phi\bigl(c_\theta(X)\bigr)^\top + \lambda I \Bigr)^{-1}\phi\bigl(c_\theta(X)\bigr) \phi\bigl(c_\theta^\tm(\mask)\bigr) \nonumber\\
	& = r_\theta(X)^\top \Bigl( \fk\bigl(c_\theta(X), c_\theta(X)\bigr) + \lambda I \Bigr)^{-1}\fk\bigl(c_\theta(X), c_\theta^\tm(\mask)\bigr),
\end{align} 
where we define  $\fk(c_\theta(X), c_\theta(X)) = (\fk(c_\theta(x^i), c_\theta(x^j)))_{i, j \in [L]} \in\RR^{L\times L}$, $\fk(c_\theta(X), c_\theta^\tm(\mask)) = (\fk(c_\theta(x^\ell), c_\theta^\tm(\mask)))_{\ell \in [L]} \in \RR^{L}$, $\phi(c_\theta(X)) = (\phi(c_\theta(x^1)), \ldots, \phi(c_\theta(x^L)))^\top \in\RR^{L \times d_\phi}$, and $r_\theta(X) = (r_\theta(x^1), \ldots, r_\theta(x^L))^\top \in\RR^{L \times d}$. 
We remark that \eqref{eq:cme-att} recovers the empirical version of the kernel conditional mean embedding of $\PP_{\cR \given \cC}$ \citep{song2009hilbert}, where we denote by $\PP_{\cR \given \cC}$ the conditional distribution of $r^\ell$ given $c^\ell$ (as two random variables) within one data point, and $\cH_r = (\RR^{d})^* = \RR^{d}$ is the dual space of $\RR^{d}$ equipped with the Euclidean kernel $ \inp{\cdot}{\cdot}$.

\vskip4pt

\noindent{\bf From Latent Variable Model to Attention.}
Recall that the attention mechanism is defined in \eqref{eq:attn} with $q \in \RR^{d_\tp}$, $K = (k^1, \ldots, k^L)^\top \in \RR^{L\times d_\tp}$, and $V = (v^1, \ldots, v^L)^\top \in \RR^{L\times d}$. 
The kernel conditional mean embedding in \eqref{eq:cme-att} motivates us to consider the following CME normalization, 
\begin{align}
	\label{eq:att-norm-cme}
	\Norm_\cme\bigl(\fk(K, q) \bigr) = \bigl(\fk(K, K) + \lambda I \bigr)^{-1}\fk(K, q),
\end{align}
where we write $\fk(K, q) = (\fk(k^\ell, q))_{\ell \in [L]} \in \RR^{L}$ and $\fk(K, K) = (\fk(k^i, k^j))_{(i, j) \in [L]\times[L]}\in\RR^{L\times L}$
We call the attention mechanism with the CME normalization in \eqref{eq:att-norm-cme} the CME attention and denote it by $\att_{\cme}$. In particular, the CME attention takes the following form,
\begin{align}\label{eq:cme-att-def}
	\att_\cme(q, K, V) = V^\top \Norm_\cme\bigl(\fk(K, q) \bigr) = V^\top \bigl( \fk(K, K) + \lambda I\bigr)^{-1}\fk(K, q) \in \RR^{d}.
\end{align}
We see that the CME attention recovers \eqref{eq:cme-att} when 
\begin{align}
	\label{eq:cme-att-conn}
	q = c_\theta^\tm(\mask), \qquad k^\ell = c_\theta(x^\ell), \qquad v^\ell = r_\theta(x^\ell), \qquad \forall \ell \in [L].
\end{align}
We remark that \eqref{eq:cme-att-conn} establishes a connection between the latent variable model and the attention mechanism. 
In other words, the covariate $c^\mask$ of the input mask $\mask$ corresponds to the query $q$, the covariate $c^\ell$ of the input token $x^\ell$ corresponds to the key $k^\ell$ for $\ell \in [L]$, and the response $r^\ell$ of the input token $x^\ell$ corresponds to the value $v^\ell$ for $\ell \in [L]$. 
In the attention mechanism, we denote by $q_\theta: \RR^{d} \rightarrow \RR^{d_\tp}$, $k_\theta: \RR^{d} \rightarrow \RR^{d_\tp}$, and $v_\theta: \RR^{d} \rightarrow \RR^{d}$ the mappings from the input token to the query, the key, and the value, respectively, with the learnable parameter $\theta$. In particular, we have the following correspondence,
\begin{align*}
	q_{\theta} = c_{\theta}^\tm, \qquad k_{\theta} = c_{\theta}, \qquad v_{\theta} = r_{\theta}.
\end{align*}
In an common example,  we instantiate $q_\theta$, $k_\theta$, and $v_\theta$ for $x \in \RR^d$ as follows,
\#\label{eq:qkv-nn}
	q_\theta(x) = (W^\tq)^\top \nn(x; A), \qquad k_\theta(x) = (W^\tk)^\top\nn(x; A), \qquad v_\theta(x) = (W^\tv)^\top \nn(x; A),
\#
where $W^\tq , W^\tk\in \RR^{d\times d_\tp}$ and $ W^\tv \in \RR^{d\times d}$ are learnable parameters. Here we denote by $\nn(\cdot; A): \RR^d \rightarrow \RR^{d}$ the feedforward neural network with the learnable parameter $A$ and summarize the learnable parameter as $\theta = (A, W^\tq, W^\tk, W^\tv)$. 
Similarly to \eqref{eq:para}, the MLE objective takes the following form,
\begin{align}
	\label{eq:learn-attn}
	\min_\theta \hat \EE_{(X, y) \sim \cD_n}\biggl[\norm[\Big]{y - \att_{\cme}\bigl(q_\theta(\mask), k_\theta(X), v_\theta(X)\bigr)}_2^2 \biggr],
\end{align}
where we write $k_\theta(X) = (k_\theta(x^1), \ldots, k_\theta(x^L))^\top \in \RR^{L\times d_\tp}$ and $v_\theta(X) = (v_\theta(x^1), \ldots, v_\theta(x^L))^\top \in \RR^{L\times d}$.


\vskip4pt

\noindent{\bf Limit of CME Attention with $L\rightarrow \infty$.}
Given an input sequence $X = \{x^\ell\}_{\ell\in [L]}$, we consider the key-value pairs $\{(k^\ell, v^\ell)\}_{\ell \in [L]}$ obtained from $k^\ell = k_\theta(x^\ell)$ and $v^\ell = v_\theta(x^\ell)$ for a fixed $\theta$.
For notational simplicity, we denote by $\cK$ and $\cV$ the random variables with the same distribution as $(k^\ell, v^\ell)$ within one data point.
Recall that we define the CME attention in \eqref{eq:cme-att-def}. Also, we define the covariance operator $\cC_{\cK\cK} = \EE[\fk(\cK, \cdot) \otimes \fk(\cK, \cdot)]$.
In the following proposition, we prove that the CME attention  approximates the kernel conditional mean embedding of $\PP_{\cV\given \cK}$ as $L\rightarrow \infty$. Note that the following proposition does not depend on the latent variable model in \eqref{eq:structure-inf}.

\begin{proposition}[CME Attention Converges to Kernel Conditional Mean Embedding]
	\label{prop:attn-cme}
	Let $\fk$ be a positive definite kernel function. We assume that $\{x^\ell\}_{\ell \in [L]}$ in the input sequence $X$ are independent and identically distributed (within one data point) and the value $\norm{v^\ell}_2$ is upper bounded by 1 for any $\ell \in [L]$. 
	It holds with probability at least $1-\delta$ that
	\begin{align*}
		\norm[\big]{\att_\cme(q, K, V) - \EE[\cV \given \cK = q]}_2 = \cO\biggl( \sqrt{\frac{L}{\lambda}} \cdot \biggl( \frac{2}{\lambda} + \sqrt{\frac{\Gamma(L^{-1}\lambda)}{\lambda}} \biggr)  \log\frac{1}{\delta} + \lambda L^{-1}\biggr).
	\end{align*}
Here $\Gamma(L^{-1}\lambda)$ is the effective dimension of the covariance operator $\cC_{\cK\cK}$, which is defined as $\Gamma(L^{-1}\lambda) = \tr((L^{-1} \lambda \cI + \cC_{\cK\cK})^{-1} \cC_{\cK\cK})$.
\end{proposition}

\begin{proof}
	See \S\ref{sec:pf-thm-attn-cme} for a detailed proof.
\end{proof}

We remark that when we use the Gaussian RBF kernel $\fk_\rbf$ in the CME attention, it holds that $\Gamma(L^{-1}\lambda) \le \cO(L / \lambda)$ \citep{zhang2015divide}. We then have $\norm[\big]{\att_\cme(q, K, V) - \EE[\cV \given \cK = q]}_2 \le \cO(L \cdot \lambda^{-3/2}\cdot \log(1/\delta) + \lambda L^{-1})$.
Note that the CME attention $\att_\cme$ is a variant of the softmax attention \citep{vaswani2017attention} with a different normalization. 
In the following section, we prove that the softmax attention has the same limit as the CME attention when the sequence length $L$ goes to infinity.

\subsection{Softmax Attention Infers Latent Posterior}
\label{sec::KDE_intuition}
In \S\ref{sec::Gaussian_Posterior}, we demonstrate how the latent variable model motivates the design of the CME attention.
Recall that we consider the attention mechanism in the form of  \eqref{eq:attn} with $q \in \RR^{d_\tp}$, $K \in \RR^{L\times d_\tp}$, and $V \in \RR^{L\times d}$.
In practice, a common normalization is defined as follows, 
\begin{align}\label{eq:norm_sm}
	\Norm_\sm(\fk(K, q)) = \bigl(\mathbf{1}_L^\top \fk(K, q)\bigr)^{-1} \cdot \fk(K, q), 
\end{align}
where $\mathbf{1}_L \in\RR^{L}$ is the $L$-dimensional all-one vector and we recall that $\fk(K, q) = (\fk(k^\ell, q))_{\ell \in [L]}\in \RR^{L}$. We denote by $\att_\sm$ the attention mechanism with the normalization in \eqref{eq:norm_sm}. 
When the kernel function is the exponential kernel $\fk_{\texp}(q, k) = \exp(k^\top q / \gamma)$ for any $q, k \in \RR^{d_\tp}$ and a fixed $\gamma > 0$, the attention mechanism in \eqref{eq:attn} takes the following form,
\#\label{eq::dop_prod_attn}
\att_\sm(q, K, V) = V^\top\Norm_\sm \bigl(\fk_{\texp}(K, q)\bigr) = V^\top\softmax(Kq / \gamma),
\#
which recovers the softmax attention in \cite{vaswani2017attention}. 
In what follows, we prove that as the sequence length $L$ goes to infinity, the softmax attention $\att_{\sm}$ has the same limit as the CME attention $\att_\cme$.




\vskip4pt

\noindent{\bf Softmax Attention Has the Same Limit as CME Attention with $L\rightarrow \infty$.} We  demonstrate that the softmax attention in \eqref{eq::dop_prod_attn} is a conditional kernel density estimator of $\PP_{\cV\given \cK}$.
We define the conditional kernel density estimator (KDE) as follows,
\begin{align}\label{eq::kernel_cond_density}
\hat \PP_{\cV \given \cK}^{\fk}(v\given q) = \iota \cdot \frac{\sum^L_{\ell=1}\fk(k^\ell, q)\cdot\fk(v^\ell, v)}{\sum^L_{\ell=1}\fk(k^\ell, q)},
\end{align}
where $\iota > 0$ is the normalization factor such that $\int \hat \PP_{\cV \given \cK}^{\fk}(v\given q) \ud v = 1$.
We remark that although the definition of the KDE in \eqref{eq::kernel_cond_density} involves the kernel function $\fk(\cdot, \cdot)$, it is not associated with any RKHS.
A common choice of the kernel function is the Gaussian RBF kernel $\fk_\rbf(q, k) = \exp( - \norm{q - k}_2^2 / 2 \gamma)$.
In what follows, we normalize the query $q$, the key $k$, and the value $v$ so that $q, k \in \SSS^{d_\tp - 1}$ and $v \in \SSS^{d-1}$, where $\SSS^{d_\tp - 1}$ and $\SSS^{d-1}$ are the $(d_\tp - 1)$-dimensional and $(d - 1)$-dimensional unit spheres, respectively. 
On the unit sphere, the exponential kernel is equivalent to the Gaussian RBF kernel. Specifically, it holds for a given rescaling $\gamma >0 $ that $\fk_{\texp}(q, k) = \exp(q^\top k / \gamma) = C \cdot \exp(-\norm{q - k}_2^2 / 2\gamma) = C \cdot \fk_\rbf(q, k)$ for any $q, k \in \RR^{d_\tp}$, where $C > 0 $ is an absolute constant.
Moreover, when we use the exponential kernel in \eqref{eq::kernel_cond_density}, the value of $\iota$ does not depend on $q$. 
To see this, note that $\int_{\SSS^{d - 1}} \fk_{\texp}(v^1, v) \ud v = \int_{\SSS^{d - 1}} \fk_{\texp}(v^2, v) \ud v$ for any $v^1, v^2 \in \SSS^{d-1}$ due to the symmetry. 
The following proposition proves that the attention mechanism  in \eqref{eq::dop_prod_attn} outputs the conditional kernel density estimator in \eqref{eq::kernel_cond_density} and has the same limit as the CME attention as $L\rightarrow\infty$.

\begin{proposition}[Softmax Attention Converges to Kernel Conditional Mean Embedding]\label{prop:kernel_attn_story1}
	Recall that the softmax attention is defined in \eqref{eq::dop_prod_attn}. It holds for any $q \in \SSS^{d_\tp - 1}$ that
	\$
	\att_\sm(q, K, V) = C \int_{\SSS^{d - 1}} v \cdot  \hat \PP_{\cV \given \cK}^{\fk}(v\given q) \ud v,
	\$
	where $C> 0$ is an absolute constant.
	Meanwhile, under the condition that $\hat \PP_{\cV \given \cK}^{\fk}(v\given k) \rightarrow \PP_{\cV\given \cK}(v \given k)$ uniformly for any $k$ as $L\rightarrow \infty$,
	it holds for $L\rightarrow \infty$ that
	\#
	\nonumber
	\att_\sm(q, K, V) \rightarrow C \cdot\EE[ \cV \given \cK = q].
	\#
\end{proposition}

\begin{proof}
	See \S\ref{sec::pf_lem_kernel_attn_story1} for a detailed proof.
\end{proof}

We remark that the uniform convergence $\hat \PP_{\cV \given \cK}^{\fk}(v\given k) \rightarrow \PP_{\cV\given \cK}(v \given k)$ holds when the density of $\PP_{\cK}$ is bounded from below \citep{de2003conditional}.
As shown in Propositions \ref{prop:attn-cme} and \ref{prop:kernel_attn_story1}, the softmax attention $\att_\sm$ and the CME attention $\att_\cme$ have the same limit as $L\rightarrow \infty$. Since the CME attention captures the latent posterior, which is proved in \S\ref{sec::Gaussian_Posterior}, we conclude that the softmax attention also captures the latent posterior approximately. Moreover, in terms of the limiting expectation $\EE[ \cV \given \cK = q]$, we highlight that it implies the necessity of using the multiple heads and connects the attention mechanism with causal inference. See \S\ref{sec:impl-limit} for a detailed discussion.

\begin{figure}[H]
	\centering
	\includegraphics[width=3.2in]{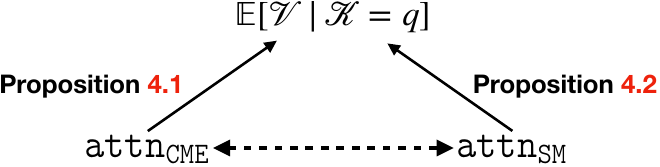}
	\caption{As shown in Propositions \ref{prop:attn-cme} and \ref{prop:kernel_attn_story1}, the softmax attention $\att_\sm$ and the CME attention $\att_\cme$ have the same limit $\EE[ \cV \given \cK = q]$ as $L\rightarrow \infty$. \label{fig:triangle}}
\end{figure}

\section{Excess Risk Analysis}
To demonstrate the theoretical benefit of incorporating latent posterior inference into the transformer architecture, we present a compact version of excess risk analysis for one-layer single-head softmax attention neural networks without skip connections. See \S\ref{appendix:gen} for a detailed analysis of the complete setup of the transformer architecture.


\noindent{\bf Attention Neural Network.} We specify the feedforward neural network in \eqref{eq:qkv-nn} as $\nn(x; A) = \relu(Ax)$\footnote{Here, for ease of presentation, we consider feedforward neural networks without bias terms.}, where $\relu(\cdot)$ is the rectified linear unit (ReLU) activation that operates elementwise. In the rest of the paper, we consider the attention neural networks with a final aggregation layer to allow for the proper scaling of the outputs in the supervised setting and the transfer capability to diverse downstream tasks in the self-supervised setting, which is discussed in \S\ref{sec:ssl}. We define the following function class of attention neural networks,
\#\label{eq:trans-class-simple}
\cF_\att = \Bigl\{ \agg_{\theta_0} \circ \att_\sm\bigl(q_\theta(\mask), k_\theta(X), v_\theta(X)\bigr) : \theta = (\theta_0, A, W^\tq, W^\tk, W^\tv) \in \Theta\Bigr\},
\#
where $\agg_{\theta_0}: \RR^{d} \to \RR^{d_\ty}$ is the aggregation layer parameterized by $\theta_0$ and $\att_\sm$ is the softmax attention defined in \eqref{eq::dop_prod_attn} with the learnable parameters $(A, W^\tq, W^\tk, W^\tv)$ defined in \eqref{eq:qkv-nn}. To characterize the excess risk, we specify the parameter space as follows, which grants $\cF_\att$ a finite capacity.

\begin{assumption}[Parameter Space]\label{assumption:constraint-simple}
We assume  for all $\theta = (\theta_0, A, W^\tq, W^\tk, W^\tv) \in \Theta$ that
 \$
\|W^{\tq}\|_{2} & \leq \omega^\tq, & \|W^{\tk}\|_{2} & \leq \omega^\tk, & \|W^{\tv}\|_{2}  & \leq \omega^\tv,  & \|A\|_{2}  & \leq {\alpha}^\nn,\notag\\
\|W^{\tq}\|_{\rm F} &  \leq R^{\tq}, & \|W^{\tk}\|_{\rm F} & \leq R^{\tk}, & \|W^{\tv}\|_{\rm F} & \leq R^{\tv},  & \|A\|_{\rm F}  & \leq {R}^{\nn},
 \$
where $\omega^\tq, \omega^\tk, \omega^\tv, \alpha^\nn, R^\tq, R^\tk, R^\tv, R^\nn >0$.
\end{assumption}

\noindent{\bf Excess Risk.} Following \eqref{eq:learn-attn}, we consider the learning objective $\cL((X, y), f) = \| y - f(X) \|_2^2$ for $f \in \cF_\att$, where $y$ is the target variable. We make the following assumption on the training dataset.
\begin{assumption}[Data Distribution]\label{asu::data}
We assume that the training dataset $\cD_n = \{(X_i, y_i)\}_{i \in [n]}$ is independently and identically drawn from the data distribution $\cD$, which is supported on the product space $\mathfrak{X}^{L}\times \mathfrak{Y}$, where 
\#\label{eq:data-supp}
\mathfrak{X}^{L} = \bigl\{X \in \RR^{L\times d}: {\textstyle \max_{\ell \in [L]}}\|x^\ell\|_2 \leq R\bigr\}, \qquad \mathfrak{Y} = \bigl\{y \in \RR^{d_\ty}: \|y\|_2 \leq  1/2\bigr\}.
\#
\end{assumption}
We consider the excess risk $\cE = \EE[\cL((X, y), \hat f)] - \EE[ \cL((X, y), f^*) ]$, where $\EE[\,\cdot\,]$ is the population expectation over the data distribution $\cD$. Here $\hat f\in\cF_\att$ is the attention neural network obtained from minimizing the empirical risk $\hat\EE[ \cL((X, y),  f)]$ in the training process, where $\hat{\EE}[\,\cdot\,]$ is the empirical expectation over the training dataset $\cD_n$. Here $f^*(X) = \EE[y\given X]$ is the regression function that we aim to approximate. In other words, $f^*$ is the optimal model that minimizes the population risk $\EE[\cL((X, y), f)]$.

To analyze the excess risk $\cE$, we decompose it into three terms,
\#\label{eq::risk_decomp}
\cE & = \underbrace{\EE\Bigl[\cL\bigl((X, y), \hat f\bigr)\Bigr] - \hat\EE\Bigl[ \cL\bigl((X, y), \hat f\bigr) \Bigr] + \hat\EE\Bigl[\cL\bigl((X, y), \tilde f\bigr)\Bigr] - \min_{f\in\cF_\att} \EE\Bigl[ \cL\bigl((X, y),  f\bigr) \Bigr]}_{\displaystyle \cE_{\text{gen}}\text{: Generalization Error}} \\
&\qquad + \underbrace{\min_{f\in\cF_\att} \EE\Bigl[ \cL\bigl((X, y),  f\bigr) \Bigr] - \EE\Bigl[ \cL\bigl((X, y), f^*\bigr) \Bigr]}_{\displaystyle \cE_{\text{approx}} \text{: Approximation Error}} + \underbrace{\hat\EE\Bigl[\cL\bigl((X, y), \hat f\bigr)\Bigr] - \hat\EE\Bigl[ \cL\bigl((X, y), \tilde f\bigr) \Bigr]}_{\displaystyle \cE_{\text{opt}}\text{: Optimization Error}}.\notag
\#
where $\tilde f = \argmin_{f\in\cF_\att} \hat\EE[ \cL((X, y),  f) ]$ is the attention neural network that minimizes the empirical risk over $\cF_\att$.

In \S\ref{sec:gen}-\ref{sec:opt}, we analyze the three terms on the right-hand side of \eqref{eq::risk_decomp} in the supervised setting. In \S\ref{sec:ssl}, we extend the following analysis of the approximation error $\cE_{\text{approx}}$ to the self-supervised setting.

\subsection{Generalization Error Analysis}\label{sec:gen}
Recall that the softmax attention in \eqref{eq::dop_prod_attn} is instantiated via the exponential kernel, which is equivalent to the Gaussian RBF kernel when $q$ and $k$ are on the unit sphere $\SSS^{d_\tp-1}$. Also, note that vector $\ell_2$-norm scales with the dimension $d_\tp$ at the rate of $\sqrt{d_\tp}$. In the rest of the paper, we consider the Gaussian RBF kernel with inputs rescaled by $1/\sqrt{d_\tp}$, i.e., 
\#\label{eq:ker-rbf}
\fk_\rbf(q, k) = \exp\bigl(-\|q/\sqrt{d_\tp} - k/\sqrt{d_\tp}\|_2^2/2\bigr) = \exp\bigl(-\|q - k\|_2^2/2d_\tp\bigr).
\#
Under Assumption \ref{assumption:constraint-simple}, we define
       \$
    \gamma = \max\{ \alpha^\nn, \omega^\tv\}, \qquad \kappa = \max\biggl\{\frac{R^\nn}{\alpha^\nn}, \frac{R^\tv}{{\omega^\tv}}, \frac{{R}^\tk + R^\tq}{(\omega^\tq + \omega^\tk)\cdot \omega^\tv}\biggr\}, \qquad \zeta = \frac{(\omega^\tq + \omega^\tk)^2 \cdot R^\tv}{\omega^\tv}.
    \$
Recall that $\cF_\att$ is the family of attention neural networks defined in \eqref{eq:trans-class-simple}. Let $\agg_{\theta_0, j}$ be the $j$-th entry of the aggregation layer $\agg_{\theta_0}$ with $j \in [d_\ty]$. We provide the following characterization of the generalization error $\cE_{\text{gen}}$.
\begin{theorem}[Generalization Error]\label{thm:gen-simple}
Let $D = \max\{d, d_\tp, d_\ty\}$. Suppose that Assumptions \ref{assumption:constraint-simple}-\ref{asu::data} hold. We assume that $\agg_{\theta_0}$ has the output range within $\mathfrak{Y}$ and $\agg_{\theta_0, j}$ is $1$-Lipschitz with respect to the $\|\cdot\|_{\rm F}$-norm for all $j \in [d_\ty]$. Then, for any $\delta > 0$, it holds with probability at least $1 - \delta$ that
\$
\cE_{\text{gen}}  = O\biggl( \frac{D^2}{\sqrt{n}} \cdot \bigl[\sqrt{\log(1 + \gamma)} +  \sqrt{\log(1 + \zeta R)} + \sqrt{\log( 1+ \kappa/\zeta)}\bigr]+ \sqrt{\frac{\log (1/\delta)}{n}}\biggr).
\$
\end{theorem}
\begin{proof}
See \S\ref{appendix:gen} for a detailed proof.
\end{proof}

An important implication of Theorem \ref{thm:gen-simple} is that the generalization error for attention neural networks does not degrade as the sequence length $L$ goes to infinity. It is also worth mentioning that the constants $\alpha$, $\omega$,$R^\nn$, $R^\att$, and $R$ play crucial roles in the theoretical analysis of the generalization error and justify the architecture design of the original transformer. In specific, we observe that (i) skip connections help reducing $\alpha$, $\omega$, $R^\nn$, and $R^\att$, and (ii) layer normalizations help reducing $R$ when there is multilayer composition of many attention mechanisms. See \S\ref{appendix:gen} for a more involved analysis of the generalization error of the complete setup of the transformer architecture and the related discussion.

\subsection{Approximation Error Analysis}\label{sec:approx}
In what follows, we characterize the approximation error in the supervised setting. 
\vskip4pt
\noindent{\bf Approximation Target.} We aim to approximate the regression function $f^*(X) = \EE[y \given X]$ with the attention neural network $f_\theta \in \cF_\att$, which is defined in \eqref{eq:trans-class-simple}. The regression function $f^*(X)$ is the optimal model in the sense that it minimizes the population risk $\EE[\cL((X, y), f)]$. By Lemma \ref{lem:sufficiency}, when the input mask $\mask$ and the latent variable $z$ are given, the target variable $y$ is independent of the input sequence $X$. Thus, the regression function $f^*(X)$ can be decomposed as follows,
\#\label{eq:latent-target}
f^*(X) = \EE[y\given X] & = \EE_{z \given X}\bigl[\EE[y \given \mask, z]\bigr]\notag\\
 &= \int \underbrace{\EE[y \given \mask, z]}_{\displaystyle g^*(z; \mask): \text{latent-to-target}} \cdot \PP(z \given X) \ud z.
\#
Here the latent-to-target mapping $g^*(z; \mask)$ can be viewed as a decoding function, which maps the latent variable $z$ to the target variable $y$ given the input mask $\mask$. On the other hand, the latent posterior $\PP(z\,|\, X)$ encodes the input sequence $X$ into the latent variable $z$. We note that the input mask $\mask$ describes the prediction task and is fixed throughout. For example, the input mask $\mask$ corresponds to the class encoding in the supervised setting or the positional encoding in the self-supervised setting.

From \eqref{eq:latent-target}, we see that approximating the regression function $f^*(X)$ involves capturing (i) the latent posterior $\PP(z \given X)$ and (ii) the latent-to-target mapping $g^*(z; \mask)$. Corresponding to (i), the latent variable $z$ summarizes the ``concept''  of the input sequence $X$, while corresponding to (ii), the target variable $y$ and the input mask $\mask$ specify the prediction task. In the following, we demonstrate the central role of the latent-to-target mapping $g^*(z; \mask)$, which attention neural networks aim to approximate. 
\vskip4pt
\noindent{\bf Approximation Surrogate.} We define the reweighted CME attention
\#\label{eq:f-att-cme}
f^{\dagger}_W(X; \mask) = W^\top\att_\cme\bigl(q_*(\mask), k_*(X), v_*(X)\bigr)
\#
as a surrogate function for approximating the regression function $f^*(X)$ in \eqref{eq:latent-target}. Here the reweighting parameter $W \in \RR^{d \times d_\ty}$ satisfies $\|W\|_{\rm F} < \infty$. In the sequel, we demonstrate that the latent-to-target function contained in $f^{\dagger}_W(X; \mask)$ approximates the latent-to-target mapping $g^*(z; \mask)$, which is a key component of the regression function $f^*(X)$. By \eqref{eq:cme-att} and \eqref{eq:cme-att-def}, we have
\#\label{eq:exact-sm}
W^\top \att_\cme\bigl(q_*(\mask), k_*(X), v_*(X)\bigr)  & =  W^\top \EE[v^\mask\,|\, \mask, X]\notag\\
& = W^\top\int \underbrace{\EE[v^\mask\,|\, \mask, z]}_{\displaystyle \psi(z; \mask) \text{: latent-to-value}} \cdot \PP(z\,|\, X) \ud z,
\#
where $q_*(\mask)$ and $v^\mask$ replace $c_*^\tm(\mask)$ and $r^\mask$ in \eqref{eq:lvm-y-inf}, respectively. Taking \eqref{eq:exact-sm} into \eqref{eq:f-att-cme}, we obtain
\#\label{eq:f-att-cme-g}
f^{\dagger}_W(X; \mask)  & =   \int \underbrace{W^\top\psi(z; \mask)}_{\displaystyle g_W^\dagger(X; \mask)} \cdot \PP(z\,|\, X) \ud z =  \EE_{z \given X}\bigl[ g^\dagger_W(z; \mask)\bigr],
\#
where $g^\dagger_W(z; \mask)$ is a latent-to-target function parameterized by $W \in \RR^{d \times d_\ty}$.

Following the infinite-dimensional counterpart of \eqref{eq:lat-post}, the reweighted CME attention captures the latent posterior $\PP(z \given X)$ under the latent variable model in \eqref{eq:structure-inf}, where the latent prior is Gaussian. Comparing \eqref{eq:latent-target} and \eqref{eq:f-att-cme-g}, we see that the reweighted CME attention $f_{W}^\dagger(X; \mask)$ performs the latent-to-target decoding via $g_{W}^\dagger(z; \mask)$, which plays the same role as the latent-to-target mapping $g^*(z; \mask)$. Thus, it remains to characterize the expressity of the function class
\#\label{eq:g-class}
\cG^\dagger = \bigl\{g^\dagger_W(z; \mask) = W^\top \psi(z; \mask): W \in \RR^{d \times d_\ty}, \|W\|_{\rm F} < \infty\bigr\}
\#
in terms of approximating the latent-to-target mapping $g^*(z; \mask)$ in \eqref{eq:latent-target}.

To characterize the function class $\cG^\dagger$ defined in \eqref{eq:g-class}, we define the function class 
\#\label{eq:g-class-i}
\cG_i^\dagger = \{g^\dagger_{W, i}(z; \mask) = w_i^\top \psi(z; \mask): w_i \in \RR^{d}, \|w_i\|_2 < \infty \bigr\},
\# 
which is formed by the $i$-th entry of the latent-to-target function $g^\dagger_W(z; \mask) \in \cG^\dagger$. Here $i \in [d_\ty]$ and $W = [w_1, \ldots, w_{d_\ty}]^\top$. Note that the function class $\cG_i^\dagger$ is the RKHS $\cH_\ltv$ induced by the kernel function $\fk_{\ltv}(z, z'; \mask) = \psi(z; \mask)^\top \psi(z'; \mask)$, which is a reproducing kernel. Here the latent-to-value ($\ltv$) mapping $\psi(z; \mask)$ is defined in \eqref{eq:exact-sm}. See \S\ref{sec:ltv-rkhs} for a detailed discussion. See Figure \ref{fig:H_mask} for a visualization of the construction of $\cH_{\ltv}$.
\begin{figure}[H]
	\centering
	\includegraphics[width=4.4in]{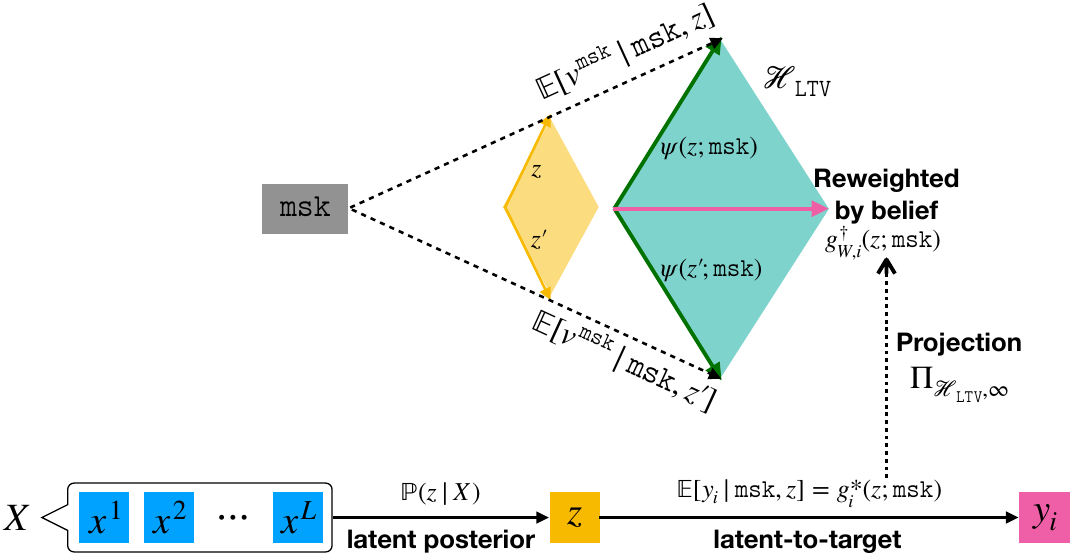}
	\caption{The RKHS $\cH_{\ltv}$ induced by the latent-to-value mapping $\psi(z; \mask)$. The input mask $\mask$ describes the prediction task and determines the RKHS $\cH_\ltv$. \label{fig:H_mask}}
\end{figure}

Therefore, the reweighted CME attention $f_W^\dagger(X; \mask)$ in \eqref{eq:latent-target} aims to capture the $i$-th entry $g^*_i(z; \mask)$ of the latent-to-target mapping $g^*(z; \mask)$ within the RKHS $\cH_{\ltv}$. To this end, we make the following assumption on the fundamental hardness of the recovery task.
\begin{assumption}[Recovery Gap]\label{asu::approx_err_sl}
For any fixed input mask $\mask$, let 
\$
g^\dagger_{W, i}(\cdot; \mask) = \Pi_{\cH_{\ltv}, \infty}\bigl(g^*_i(\cdot; \mask)\bigr) = \argmin_{g_i(\cdot; \mask) \in \cH_{\ltv}} \bigl\|g^*_i(\cdot; \mask) - g_i(\cdot; \mask)\bigr\|_\infty
\$ 
be the $\ell_\infty$-norm projection of the $i$-th entry $g^*_i(z; \mask)$ of the latent-to-target mapping $g^*(z; \mask)$ onto the RKHS $\cH_{\ltv}$. We assume that there exists $\epsilon_{g}(\mask) \in [0, +\infty)$ such that
\$
\sum_{i  = 1}^{d_\ty}\bigl\|g_i^*(\cdot; \mask) - g^\dagger_{W, i}(\cdot; \mask)\bigr\|^2_{\infty} \leq \epsilon_{g}^2(\mask).
\$
Here the $\ell_\infty$-norm is taken over the latent variable $z$.
\end{assumption}
Recall that the function class of attention neural networks $\cF_\att$ is defined in \eqref{eq:trans-class-simple}. We have the following theorem characterizing the approximation error $\cE_{\rm approx}$ defined in \eqref{eq::risk_decomp}.
\begin{theorem}[Approximation Error]
\label{lem::approx_error}
Let $\{g^\dagger_{W, i}(z; \mask) = w_i^\top \psi(z; \mask)\}_{i \in [d_\ty]}$ be a function class satisfying Assumption \ref{asu::approx_err_sl}. We define $W = [w^\top_1, \cdots, w^\top_{d_\ty}]^\top$. Suppose that there exists $f_\theta \in \cF_\att$ and $\epsilon_\att \in [0, +\infty)$ such that
\#\label{eq:appprox-sm-cme}
\sup_{X \in \mathfrak{X}^{L}}\Bigl\|f_\theta(X; \mask) - W ^\top\att_\cme\bigl(q_*(\mask), k_*(X), v_*(X)\bigr)\Bigr\|_2 \leq \epsilon_\att,
\#
where $\mathfrak{X}^{L}$ is defined in \eqref{eq:data-supp}.
Then, we have 
\$
\cE_{\text{approx}} \leq 2\epsilon^2_{g}(\mask) +  2\epsilon^2_{\att}.
\$
\end{theorem}
\begin{proof}
See \S\ref{sec::pf_lem_approx_error} for a detailed proof.
\end{proof}
The approximation error bound in Theorem \ref{lem::approx_error} involves the recovery gap $\epsilon_g(\mask)$ and the surrogate approximation error $\epsilon_\att$. Since the latent posterior $\PP(z \given X)$ is captured by the reweighted CME attention, the recovery gap $\epsilon_g(\mask)$ between the function class $\cG^\dagger$ in \eqref{eq:g-class} and the latent-to-target mapping $g^*(z; \mask)$ in \eqref{eq:latent-target} plays the central role in the approximation error bound. On the other hand, the approximation error $\epsilon_\att$ between attention neural networks in $\cF_\att$ and the reweighted CME attention is characterized in Proposition \ref{prop:kernel_attn_story1}.


\subsection{Optimization Error Analysis}\label{sec:opt}
Since the learning objective of attention neural networks is nonconvex with respect to the parameter $\theta$, we consider the property of the stationary points. Let $\hat \theta =(\hat{\theta}_0, \hat{A}, \hat{W}^\tq, \hat{W}^\tk, \hat{W}^\tv)$ be the stationary point of the empirical risk $ \hat{\EE}[\cL((X, y), f)]$, that is,  
\begin{align}\label{eq:statw}
\biggl
\la\nabla_\theta \hat{\EE}\Bigl[\cL\bigl((X, y), {f}_{\hat \theta}\bigr)\Bigr], \theta - \hat \theta\biggr\ra \ge 0, \quad \forall \theta \in \Theta,
\end{align}
which is the learnable parameter obtained in the training process, i.e., $\hat f = f_{\hat \theta}$. Recall that the regression function $f^*(X) = \EE[y\given X]$ is the minimizer of the population risk ${\EE}[\cL((X, y), f)]$. We have the following proposition characterizing the optimization error $\cE_{\rm opt}$, which is defined in \eqref{eq::risk_decomp}.

\begin{proposition}[Optimization Error]\label{prop:opt-stat}
	Suppose that Assumption \ref{asu::data} holds.
Then, it holds that
	\begin{align}\label{eq:opt}
		\cE_{\rm opt} \leq 2\cdot \min_{\theta \in \Theta}\hat \EE\Bigl[\bigl\|f_{\hat \theta}(X) + \nabla_\theta f_{\hat \theta}(X)^\top (\theta - \hat \theta) - f^*(X)\bigr\|_2\Bigr].
	\end{align}
\end{proposition}
\begin{proof}
See \S\ref{appendix:opt-ap} for a detailed proof.
\end{proof}

The right-hand side of \eqref{eq:opt} quantifies the expressity of the function class defined by the local linearization,
\$
\bigl\{f_{\hat \theta}(X) + \nabla_\theta f_{\hat \theta}(X)^\top(\theta - \hat \theta) : \theta \in \Theta \bigr\}.
\$
In the neural tangent kernel (NTK) regime \citep{yang2020tensor, yang2021tensor, jacot2018neural}, it is known that,
\$
	f^*(X) = f_{\hat \theta}(X) + \nabla_\theta f_{\hat \theta}(X)^\top (\theta - \hat \theta) + o(1), \quad \forall X \in \RR^{L \times d},
\$
where the $o(1)$ error captures the local linearization error in the NTK-based analysis. As a consequence, the optimization error satisfies $\cE_{\rm opt} = o(1)$, that is, the stationary point $\hat \theta$ is (approximately) global optimal. Such a result shows the theoretical benefit of incorporating feedforward neural networks in the architecture design. While NTK-based analysis involves a random initialization in the supervised setting, \cite{malladi2022kernel} provide an NTK-based analysis for the downstream training of the transformer with a pretrained initialization in the self-supervised setting.


\section{From Supervised Learning to Self-Supervised Learning}\label{sec:ssl}
An important aspect of the attention mechanism is that one can obtain a sequence embedding by pretraining in a self-supervised manner, which gives rise to the transfer capability for diverse downstream tasks. 
\vskip4pt
\noindent{\bf Self-Supervised Learning.} The attention mechanism enables embedding learning and downstream prediction via the self-supervised learning (SSL) process as follows.
\begin{itemize}
    \item[($\pre$)] \underline{P}re\underline{t}raining process: We train an attention neural network $\hat f_\pre(\overbar X; \mask_\pre) =  f_{\hat \theta_\pre}(\overbar X; \mask_\pre)\in \cF_\pre$ with the learned parameter $\hat \theta_\pre$ to predict the masked token $x^L \in \RR^d$, which is denoted by $y_\pre$, from the truncated input sequence $\overbar X = \{x^\ell\}_{\ell \in [L-1]}$ and the input mask $\mask_\pre$. Here the function class of attention neural networks for the pretraining process is defined as follows,
\#\label{eq:f-pre}
\cF_\pre = \Bigl\{\agg^\pre_{\theta} \circ \att_\sm\bigl(q_{\theta}(\mask_\pre), k_{\theta}(\overbar X), v_{\theta}(\overbar X)\bigr): \theta \in \Theta_\pre\Bigr\},
\#
where $\agg^\pre_{\theta}: \RR^{d} \to \RR^{d}$ is the aggregation layer. For the pretraining process, the input mask $\mask_\pre$ is the positional encoding of the masked token $x^L$.
    \item[($\dt$)] \underline{D}own\underline{s}tream task: We freeze the learned parameter $\hat \theta_\pre$ and train another attention neural network $\hat f_\dt(\overbar X; \mask_\dt) =  f_{\hat \theta_\dt}(\overbar X; \mask_\dt)\in \cF_\dt$ with the learned parameter $\hat \theta_\dt$ to predict another target variable $y_\dt \in \RR^{d_\ty}$ from the truncated input sequence $\overbar X = \{x^\ell\}_{\ell \in [L-1]}$ and another input mask $\mask_\dt$.  Here the function class of attention neural networks for the downstream task is defined as follows,
\#\label{eq:f-dt}
\cF_\dt = \Bigl\{\agg^\dt_{\theta} \circ \att_\sm\bigl(q_{\hat \theta_\pre}(\mask_\dt), k_{\hat \theta_\pre}(\overbar X), v_{\hat \theta_\pre}(\overbar X)\bigr): \theta \in \Theta_\dt\Bigr\},
\#
which means that the aggregation layer $\agg^\dt_{\hat \theta_\dt}: \RR^{d} \to \RR^{d_\ty}$ replaces the aggregation layer $\agg^\pre_{\hat \theta_\pre}: \RR^{d} \to \RR^{d}$ obtained in the pretraining process. For the downstream task, the input mask $\mask_\dt$ is the class encoding of the target variable $y_\dt$.
\end{itemize}
With the full input sequence $X$ replaced by the truncated input sequence $\overbar X$, the attention neural network $\hat f_\dt(\overbar X; \mask)$ obtained in the SSL process has the same decomposition of the excess risk as that in \eqref{eq::risk_decomp}. In the risk decomposition, for the SSL process, we have the same characterization of the generalization error and the optimization error as those in the supervised setting. When the downstream task is trained using the same set of truncated input sequences as that in the pretraining process, our previous analysis of the generalization error in the supervised setting is applicable to the SSL process. On the other hand, when the downstream task is trained using an independent set of truncated input sequences, we can modify our previous analysis to prove that the generalization error only scales with the complexity measure (e.g., the covering number) of the function class $\{\agg^\dt_\theta: \theta \in \Theta_\dt\}$ of aggregation layers without depending on that of the attention mechanism, as $\hat \theta_\pre$ is frozen. Also, the attention neural network $\hat f_\pre(\overbar X; \mask)$ obtained in the pretraining process has the same approximation error as that in the supervised setting. To characterize the approximation error for the SSL process, we analyze the approximation error for the downstream task by connecting it to the approximation error for the pretraining process.

\vskip4pt
\noindent{\bf Approximation Error.} In parallel to the supervised setting, we define the regression function and the latent-to-target mapping for the pretraining process as follows,
\#\label{eq::def_hu_ssl}
 f_\pre^*(\overbar X) = \EE[y_\pre \given \overbar{X}], \qquad g^*_\pre(z; \mask_\pre) = \EE[y_\pre \given \mask_\pre, z].
\#
Correspondingly, we defined the regression function and the latent-to-target mapping for the downstream task as follows,
\#\label{eq:fg-dt}
f_\dt^*(\overbar X) = \EE[y_\dt \given \overbar{X}], \qquad g_\dt^*(z; \mask_\dt) = \EE[y_\dt \given \mask_\dt, z].
\#
In parallel to the reweighted CME attention defined in \eqref{eq:f-att-cme}, we defined the surrogate functions for the pretraining process and the downstream task as follows,
\#\label{eq:surrogate-ssl}
f_{W_\pre}^\dagger(\overbar X; \mask_\pre) & = W_\pre^\top \att_\cme\bigl(q_*(\mask_\pre), k_*(\overbar X), v_*(\overbar X)\bigr),\notag\\
f_{W_\dt}^\dagger(\overbar X; \mask_\dt) & = W_\dt^\top \att_\cme\bigl(q_*(\mask_\dt), k_*(\overbar X), v_*(\overbar X)\bigr),
\#
where $W_\pre \in \RR^{d \times d}$ and $W_\dt \in \RR^{d \times d_\ty}$ are the reweighting parameters. We use the surrogate function to bridge the regression function and the attention neural network, which is illustrated in Figure \ref{fig:ssl}. In parallel to the latent-to-value mapping $\psi(z; \mask)$ defined in \eqref{eq:exact-sm}, we define the latent-to-value mappings for the pretraining process and the downstream task as follows,
\$
\psi_\pre(z; \mask_\pre) = \EE[v^{\mask_\pre}\,|\, \mask_\pre, z], \qquad \psi_\dt(z; \mask_\dt) = \EE[v^{\mask_\dt}\,|\, \mask_\dt, z],
\$
where $v^{\mask_\pre}$ and $v^{\mask_\dt}$ replace $r^\mask$ in \eqref{eq:lvm-y-inf}. The latent-to-value mappings induce the kernel functions as follows,
\$
\fk_\pre(z, z'; \mask_\pre) & = \psi_\pre(z; \mask_\pre)^\top \psi_\pre(z'; \mask_\pre),\\ \fk_\dt(z, z'; \mask_\dt) & = \psi_\dt(z; \mask_\dt)^\top \psi_\dt(z'; \mask_\dt),
\$
which induce the RKHSs $\cH_\pre$ and $\cH_\dt$. Corresponding to \eqref{eq:f-att-cme-g}, we have
\#
f^{\dagger}_{W_\pre}(\overbar X; \mask_\pre)  & =   \int \underbrace{W_\pre^\top\psi_\pre(z; \mask_\pre)}_{\displaystyle g_{W_\pre}^\dagger(\overbar X; \mask_\pre)} \cdot \PP(z\,|\, \overbar X) \ud z =  \EE_{z \given \overbar X}\bigl[ g^\dagger_{W_\pre}(z; \mask_\pre)\bigr],\notag\\
f^{\dagger}_{W_\dt}(\overbar X; \mask_\dt)  & =   \int \underbrace{W_\dt^\top\psi_\dt(z; \mask_\dt)}_{\displaystyle g_{W_\dt}^\dagger(\overbar X; \mask_\dt)} \cdot \PP(z\,|\, \overbar X) \ud z =  \EE_{z \given \overbar X}\bigl[ g^\dagger_{W_\dt}(z; \mask_\dt)\bigr].\label{eq:f-att-cme-g-ssl}
\#
Note that $f^{\dagger}_{W_\pre}(\overbar X; \mask_\pre)$ and $f^{\dagger}_{W_\dt}(\overbar X; \mask_\dt)$ share the same latent posterior since the attention mechanism is frozen for the downstream task. By our previous arguments following \eqref{eq:f-att-cme-g}-\eqref{eq:g-class-i}, it remains to characterize how the reweighted CME attentions in \eqref{eq:surrogate-ssl} recover the latent-to-target mappings in \eqref{eq::def_hu_ssl}-\eqref{eq:fg-dt} within the RKHSs $\cH_\pre$ and $\cH_\dt$. See Figure \ref{fig:ssl} for an illustration of the construction of the RKHSs $\cH_\pre$ and $\cH_\dt$.

\begin{figure}[H]
	\centering
	\includegraphics[width=5.4in]{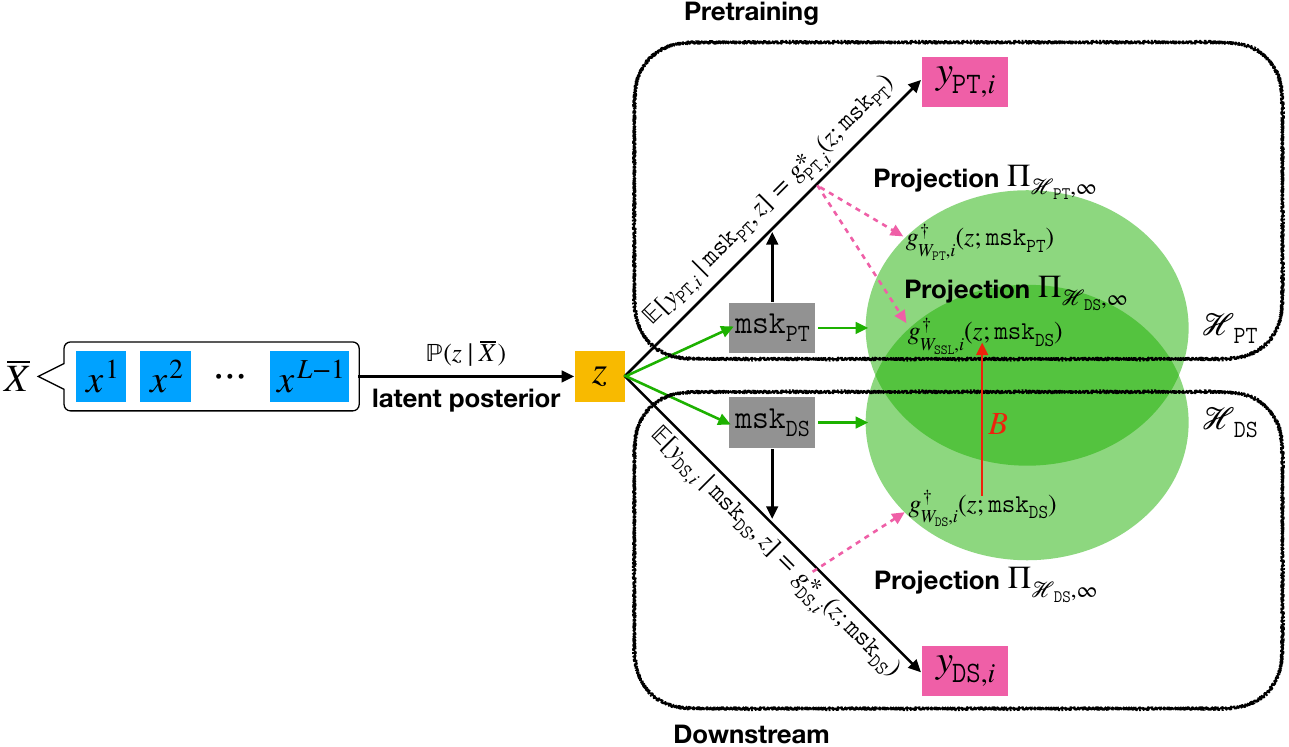}
	\caption{The RKHSs $\cH_\pre$ and $\cH_\dt$ induced by the latent-to-value mappings $\psi_\pre(z; \mask_\pre)$ and $\psi_\dt(z; \mask_\dt)$, respectively. The input masks $\mask_\pre$ and $\mask_\dt$ describe the pretraining process and the downstream task, respectively, and determine the RKHSs correspondingly. The $\ell_\infty$-norm projections $\Pi_{\cH_\pre, \infty}$ and $\Pi_{\cH_\dt, \infty}$ are defined in Assumption \ref{asu::approx_err_ss}. \label{fig:ssl}}
\end{figure}

In parallel to Assumption \ref{asu::approx_err_sl}, we introduce the following assumption on the fundamental hardness of approximating the latent-to-target mappings within the RKHSs $\cH_\pre$ and $\cH_\dt$.


\begin{assumption}[SSL Recovery Gap]\label{asu::approx_err_ss}
For any fixed input masks $\mask_\pre$ and $\mask_\dt$, let 
\$
g^\dagger_{W_\pre, i}(\cdot; \mask_\pre) & = \Pi_{\cH_{\pre}, \infty}\bigl(g^*_{\pre, i}(\cdot; \mask_\pre)\bigr) = \argmin_{g_i(\cdot; \mask_\pre) \in \cH_{\pre}} \bigl\|g^*_{\pre, i}(\cdot; \mask_\pre) - g_i(\cdot; \mask_\pre)\bigr\|_\infty,\\
g^\dagger_{W_\dt, i}(\cdot; \mask_\dt) & = \Pi_{\cH_{\dt}, \infty}\bigl(g^*_{\dt, i}(\cdot; \mask_\dt)\bigr) = \argmin_{g_i(\cdot; \mask_\dt) \in \cH_{\dt}} \bigl\|g^*_{\dt, i}(\cdot; \mask_\dt) - g_i(\cdot; \mask_\dt)\bigr\|_\infty,\\
g^\dagger_{W_\ssl, i}(\cdot; \mask_\dt) & = \Pi_{\cH_{\dt}, \infty}\bigl(g^*_{\pre, i}(\cdot; \mask_\pre)\bigr) = \argmin_{g_i(\cdot; \mask_\dt) \in \cH_{\dt}} \bigl\|g^*_{\pre, i}(\cdot; \mask_\pre) - g_i(\cdot; \mask_\dt)\bigr\|_\infty
\$ 
be the $\ell_\infty$-norm projections of the $i$-th entry $g^*_{\pre, i}(z; \mask_\pre)$ of the latent-to-target mapping $g^*_\pre(z; \mask_\pre)$ onto the RKHS $\cH_{\pre}$, the $i$-th entry $g^*_{\dt, i}(z; \mask_\dt)$ of the latent-to-target mapping $g^*_\dt(z; \mask_\dt)$ onto the RKHS $\cH_{\dt}$, and the $i$-th entry $g^*_{\pre, i}(z; \mask_\pre)$ of the latent-to-target mapping $g^*_\pre(z; \mask_\pre)$ onto the RKHS $\cH_{\dt}$, respectively. We assume the following statements hold.
\begin{itemize}
\item[($\pre$)] There exists $\epsilon_{g}(\mask_\pre) \in [0, +\infty)$ such that
\#\label{eq:eps-pre}
\sum_{i  = 1}^{d}\bigl\|g_{\pre, i}^*(\cdot; \mask_\pre) - g^\dagger_{W_\pre, i}(\cdot; \mask_\pre)\bigr\|^2_{\infty} \leq \epsilon_g^2(\mask_\pre).
\#
\item[($\dt$)] There exists $\epsilon_{g}(\mask_\dt) \in [0, +\infty)$ such that
\#\label{eq:eps-dt}
\sum_{i  = 1}^{d_\ty}\bigl\|g_{\dt, i}^*(\cdot; \mask_\dt) - g^\dagger_{W_\dt, i}(\cdot; \mask_\dt)\bigr\|^2_{\infty} \leq \epsilon_g^2(\mask_\dt).
\#
\item[($\ssl$)] There exists $\epsilon_\ssl(\mask_\pre, \mask_\dt) \in [0, +\infty)$ such that
\#\label{eq:eps-ssl}
\sum_{i  = 1}^{d}\bigl\|g_{\pre, i}^*(\cdot; \mask_\pre) - g^\dagger_{W_\ssl, i}(\cdot; \mask_\dt)\bigr\|^2_{\infty} \leq \epsilon_\ssl^2(\mask_\pre, \mask_\dt).
\#
\end{itemize}
Here the $\ell_\infty$-norms are taken over the latent variable $z$.
\end{assumption}

Intuitively, the feature $\psi_\pre(z; \mask_\pre)$ is obtained in the pretraining process, while the feature $\psi_\dt(z; \mask_\dt)$ is desired by the downstream task. Meanwhile, \eqref{eq:eps-ssl} characterizes the fundamental hardness of recovering the latent-to-target mapping $g_\pre^*(z; \mask_\pre)$ for the pretraining process within the RKHS $\cH_\dt$. Thus, the transfer error $\epsilon_\ssl(\mask_\pre, \mask_\dt)$ captures the transfer capability of the sequence embedding obtained in the pretraining process to the downstream task. In other words, when the pretraining process is sufficiently related to the downstream task, the transfer error $\epsilon_\ssl(\mask_\pre, \mask_\dt)$ is small, which allows us to approximate the $i$-th entry of the latent-to-target mapping $g^*_\pre(z; \mask_\pre)$ within the RKHS $\cH_\dt$ up to the approximation error $\epsilon_\ssl(\mask_\pre, \mask_\dt)$.

We introduce the following assumption on the condition number that characterizes the alignment between the reweighting parameter desired by the downstream task  and the reweighting parameter obtained in the pretraining process. 
\begin{assumption}[SSL Condition Number]\label{asu::effec_ssl_param}
Let $\{g_{W_\dt, i}^\dagger(z; \mask_\dt) = w_{\dt, i}^\top \psi_\dt(z; \mask_\dt)\}_{i \in [d]}$ and $\{g^\dagger_{W_\ssl, i}(z; \mask_\dt) = w_{\ssl, i}^\top \psi_\dt(z; \mask_\dt)\}_{i \in [d_\ty]}$ be the function classes satisfying \eqref{eq:eps-dt} and \eqref{eq:eps-ssl} in Assumption \ref{asu::approx_err_ss}, respectively. Also, let $W_\dt = [w_{\dt, 1}, \ldots, w_{\dt, d_\ty}]^\top \in \RR^{d \times d_\ty}$, $W_\ssl = [w_{\ssl, 1}, \ldots, w_{\ssl, d}]^\top \in \RR^{d \times d}$, and\footnote{For ease of presentation, we assume that $W_\ssl W^\top_\ssl \in \RR^{d \times d}$ is invertible. When $W_\ssl W^\top_\ssl$ is not invertible, our subsequent analysis can be generalized using the pseudoinverse of $W_\ssl W^\top_\ssl$.}
\#\label{eq:b}
B = W_\dt^\top (W_\ssl W^\top_\ssl)^{-1} W_\ssl \in \RR^{d_\ty \times d}.
\#
We assume that there exists $\mu \in [0, +\infty)$ such that $\|B\|_2^2 \leq \mu$.
\end{assumption}

The condition number $\mu$ plays a critical role in our subsequent analysis. To see the intuition behind $\mu$, let $W_\dt = W_\ssl$, which implies that $B$ is a projection matrix and $\mu = 1$. Also, let the row vectors of $W_\ssl$ be an orthonormal basis of $\RR^{d}$, which implies that $W_\ssl W_\ssl^\top = I_{d_\tp}$ and $B = W_\dt^\top W_\ssl$. In this case, $B$ measures the subspace alignment between the reweighting parameter $W_\dt$ desired by the downstream task and the reweighting parameter $W_\ssl$ obtained in the pretraining process. In general cases where $W_\ssl$ is nonorthonormal, we have a similar interpretation through the eigenvalue decomposition of $W_\ssl W_\ssl^\top$.

Recall that $\hat f_\pre(\overbar X; \mask_\pre)$ is the attention neural network obtained in the pretraining process. For any $U \in \RR^{d_\ty \times d}$, we define the following quantity that characterizes the expressity of the function class $\{\agg^\dt_\theta: \theta \in \Theta_\dt\}$ of aggregation layers for the downstream task,
\#\label{eq:approx-sm-cme-ssl}
\epsilon_\agg(U) & = \inf_{f_\dt \in \cF_\dt}\sup_{\overbar X \in \mathfrak{X}^{L-1}}\bigl\|f_\dt(\overbar X; \mask_\dt) - U\hat f_\pre(\overbar X; \mask_\pre)\bigr\|_2\notag\\
& = \inf_{\theta \in \Theta_\dt}\sup_{\overbar X \in \mathfrak{X}^{L-1}}\bigl\|\agg^\dt_{\theta} \circ \att_\sm\bigl(q_{\hat \theta_\pre}(\mask_\dt), k_{\hat \theta_\pre}(\overbar X), v_{\hat \theta_\pre}(\overbar X)\bigr)\notag\\
& \quad\qquad\qquad\qquad  - U\agg^\pre_{\hat \theta_\pre} \circ \att_\sm\bigl(q_{\hat \theta_\pre}(\mask_\pre), k_{\hat \theta_\pre}(\overbar X), v_{\hat \theta_\pre}(\overbar X)\bigr)\bigr\|_2,
\# 
where $\mathfrak{X}^{L-1}$ is defined in \eqref{eq:data-supp}. Since the attention mechanism is frozen for the downstream task, the trainable part of the attention neural network is the aggregation layer $\agg^\dt_\theta$. Thus, the aggregation approximation error $\epsilon_\agg(U)$ characterizes the expressity of the function class $\{\agg^\dt_\theta: \theta \in \Theta_\dt\}$ of aggregation layers in terms of approximating the composition of (i) the linear transformation $U\agg^\pre_{\hat \theta_\pre}$ of the aggregation layer obtained in the pretraining process and (ii) the output variation induced by switching the input mask $\mask_\pre$ to another input mask $\mask_\dt$ in the attention mechanism, which is frozen. To see the intuition behind $\epsilon_\agg(U)$, let $\mask_\pre = \mask_\dt$, which implies that $\epsilon_\agg(U) = 0$ as long as the function class of aggregation layers takes the form of $\{\agg^\dt_\theta = U\agg^\pre_{\hat \theta_\pre}: \theta = U \in \RR^{d_\ty \times d}\}$. In this case, $\epsilon_\agg(U)$ characterizes the compatibility between $\agg_\theta^\dt$ and $\agg^\pre_{\hat \theta_\pre}$ under a linear transformation parameterized by $\theta$. In general cases where $\mask_\pre \neq \mask_\dt$, $\epsilon_\agg(U)$ additionally characterizes the capability of $\agg_\theta^\dt$ to capture the output variation induced by switching the input mask.

Recall that the function class $\cF_\pre$ of attention neural networks for the pretraining process is defined in \eqref{eq:f-pre}. Let
\#\label{eq:approx-pre}
\cE_{\rm approx}^\pre = \min_{f \in \cF_\pre}\EE\Bigl[\cL\bigl((\overbar X, y_\pre), f\bigr)\Bigr] - \EE\Bigl[\cL\bigl((\overbar X, y_\pre), f_\pre^*\bigr)\Bigr]
\#
be the approximation error for the pretraining process, which is characterized in Theorem \ref{lem::approx_error}. Recall that the function class $\cF_\dt$ of attention neural networks for the downstream task is defined in \eqref{eq:f-dt}. For the downstream task, the approximation error in \eqref{eq::risk_decomp} takes the following form,
\#\label{eq:approx-ssl}
\cE_{\rm approx} = \min_{f \in \cF_\dt}\EE\Bigl[\cL\bigl((\overbar X, y_\dt), f\bigr)\Bigr] - \EE\Bigl[\cL\bigl((\overbar X, y_\dt), f_\dt^*\bigr)\Bigr].
\#
The following theorem characterizes the approximation error $\cE_{\rm approx}$ for the SSL process.

\begin{theorem}[SSL Approximation Error]
\label{lem::approx_ssl}
Under Assumptions \ref{asu::approx_err_ss} and \ref{asu::effec_ssl_param}, it holds that
\$
\cE_{\text{approx}} = O\Bigl(\mu\cdot \bigl(\cE_{\rm approx}^{\pre} + \epsilon_\ssl^2(\mask_\pre, \mask_\dt)\bigr) + \epsilon_g^2(\mask_\dt) + \epsilon_\agg^2(B) \Bigr),
\$
where $\cE^\pre_{\rm approx}$, $\epsilon_\ssl(\mask_\pre, \mask_\dt)$, $\epsilon_g(\mask_\dt)$, and $\epsilon_\agg(B)$ are defined in \eqref{eq:approx-pre}, \eqref{eq:eps-ssl}, \eqref{eq:eps-dt}, and \eqref{eq:approx-sm-cme-ssl}, respectively, and $B$ is defined in \eqref{eq:b}.
\end{theorem}
\begin{proof}
See \S\ref{sec::pf_lem_approx_ssl} for a detailed proof.
\end{proof}

Theorem \ref{lem::approx_ssl} demonstrates that the attention neural network enables the transfer capability to diverse downstream tasks, where the approximation error is subsume from that in the supervised setting with a few extra error terms. We interpret the approximation error in Theorem \ref{lem::approx_ssl} as follows.
\begin{itemize}
\item[(i)] The condition number $\mu$ characterizes the the alignment between the reweighting parameter $W_\dt$ desired by the downstream task and the reweighting parameter $W_\ssl$ obtained in the pretraining process. When $W_\dt = W_\ssl$, we have $\mu = 1$.
\item[(ii)] The approximation error $\cE_{\rm approx}^\pre$ for the pretraining process is characterized in Theorem \ref{lem::approx_error}. Specifically, $\cE_{\rm approx}^\pre$ involves the pretraining recovery gap $\epsilon_g(\mask_\pre)$ defined in \eqref{eq:eps-pre}, which characterizes the fundamental hardness of approximating the $i$-th entry $g_{\pre, i}^*(z; \mask_\pre)$ of the latent-to-target mapping $g_\pre^*(z; \mask_\pre)$ defined in \eqref{eq::def_hu_ssl} within the RKHS $\cH_\pre$, and the attention approximation error $\epsilon_\att$ defined in \eqref{eq:appprox-sm-cme}, which is characterized in Proposition \ref{prop:kernel_attn_story1}.
\item[(iii)] The transfer error $\epsilon_\ssl(\mask_\pre, \mask_\dt)$ captures the transfer capability of the sequence embedding obtained in the pretraining process to the downstream task. By our previous arguments following Assumption \ref{asu::approx_err_ss}, $\epsilon_\ssl(\mask_\pre, \mask_\dt)$ is small as long as the pretraining process is sufficiently related to the downstream task. 
\item[(iv)] The downstream recovery gap $\epsilon_g(\mask_\dt)$ defined in \eqref{eq:eps-dt} characterizes the fundamental hardness of approximating the $i$-th entry $g_{\dt, i}^*(z; \mask_\dt)$ of the latent-to-target mapping $g^*_\dt(z; \mask_\dt)$ defined in \eqref{eq:fg-dt} within the RKHS $\cH_\dt$. 
\item[(v)] The aggregation approximation error $\epsilon_\agg(B)$ measures the expressity of the function class of aggregation layers for the downstream task. By our previous arguments following \eqref{eq:approx-sm-cme-ssl}, $\epsilon_\agg(B)$ is small as long as the aggregation layer $\agg^\dt_\theta$ for the downstream task can approximate the composition of the linear transformation $B\agg^\pre_{\hat \theta_\pre}$ of the aggregation layer obtained in the pretraining process and the variation induced by switching the input mask $\mask_\pre$ to another input mask $\mask_\dt$.
\end{itemize}

%
%



\section{Conclusion}\label{sec:conclusion}
The attention mechanism in transformers demonstrates significant empirical successes in natural language processing and computer vision, but there is a lack of understanding about how it works and why it is effective. To this end, we answer three questions about the attention mechanism: (i) what makes a good representation, (ii) how the attention mechanism produces this representation during the forward pass, and (iii) how the attention mechanism learns to produce this representation during the backward pass. Through the lens of exchangeability, we provide a theoretical characterization of the attention mechanism as a combination of a ``whitebox'' design guided by a latent variable model and a ``blackbox'' design allowing for learnable parameters. Also, we establish the approximation, generalization, and optimization guarantees of the attention mechanism. Several challenging questions remain open, e.g., what is the theoretical benefit of a multilayer composition of many attention mechanisms and how it affects the approximation guarantee.

\section{Acknowledgment}

Zhaoran Wang acknowledges National Science Foundation (Awards 2048075, 2008827, 2015568, 1934931), Simons Institute (Theory of Reinforcement Learning), Amazon, J.P. Morgan, and Two Sigma for their supports.



\newpage
\bibliographystyle{ims}
\bibliography{ref.bib}

\newpage
\appendix
\newpage


%
%

\section{Conditional Mean Embedding}
\label{sec:cme}
We introduce the conditional mean embedding \citep{song2009hilbert}, which embeds a conditional distribution to an element in an RKHS. Let $\cH_x$ and $\cH_y$ be the two RKHSs over the spaces $\fX$ and $\fY$ with the kernels $\fk$ and $\fl$, respectively. We denote by $\phi: \fX \rightarrow \ell_2$ and $\varphi: \fY \rightarrow \ell_2$ the feature mappings associated with $\cH_x$ and $\cH_y$, respectively. In other words, it holds for any $x, x' \in \fX$ and $y, y' \in \fY$ that
\begin{align}
	\label{eq:feature-kernel}
	\phi(x)^\top \phi(x') = \fk(x, x'), \qquad \varphi(y)^\top \varphi(y) = \fl(y, y').
\end{align}
Let $\PP_{\cX, \cY}$ be the joint distribution of the two random variables $\cX$ and $\cY$ taking values in $\fX$ and $\fY$, respectively. The conditional mean embedding $\cme(x, \PP_{\cX, \cY}) \in \cH_y$ of the conditional distribution $\PP_{\cY \given \cX}$ is defined as 
\begin{align*}
	\cme(x, \PP_{\cX, \cY}) = \EE \bigl[\fl(\cY, \cdot) \biggiven \cX= x\bigr].
\end{align*}
By the reproducing property, it holds that
\begin{align*}
	\EE \bigl[g(\cY) \biggiven \cX = x \bigr] = \inp[\big]{g}{\cme(x, \PP_{\cX, \cY})}_{\cH_y}, \quad \forall g \in \cH_y, x \in \fX.
\end{align*}
Correspondingly, the conditional mean embedding operator $\cC_{\cY \given \cX}: \cH_x \rightarrow \cH_y$ is a linear operator such that
\$
	\cC_{\cY \given \cX} \fk(x, \cdot) = \cme(x, \PP_{\cX, \cY}),
\$
for any $x \in \fX$.
We define the (uncentered) covariance operator $\cC_{\cX\cX}: \cH_x \rightarrow \cH_x$ and the (uncentered) cross-covariance operator $\cC_{\cY\cX} : \cH_x \rightarrow \cH_y$ as follows, 
\$
	\cC_{\cX\cX} = \EE \bigl[ \fk(\cX, \cdot) \otimes \fk(\cX, \cdot)\bigr], \qquad \cC_{\cY\cX} = \EE \bigl[ \fl(\cY, \cdot) \otimes \fk(\cX, \cdot)\bigr].
\$
Here $\otimes$ is the tensor product. As shown in \cite{song2009hilbert}, it holds that $\cC_{\cY \given \cX} = \cC_{\cY\cX} \cC_{\cX\cX}^{-1}$. Thus, we have that
\begin{align}\label{eq:cme-cov}
	\cme(x, \PP_{\cX, \cY}) = \cC_{\cY\cX} \cC_{\cX\cX}^{-1} \fk(x, \cdot).
\end{align}

To derive the empirical estimation of $\cC_{\cY \given \cX}$, we consider the following regularized least-squares problem,
\begin{align}
	\label{eq:reg-int-emp}
	\min_{\cC : \cH_x \rightarrow \cH_y} \hat \cE(\cC) =  \sum_{\ell = 1}^{L} \norm[\big]{ \fl(y^\ell, \cdot) - \cC \fk(x^\ell, \cdot) }^2_{\cH_y} + \lambda \cdot \norm{\cC}^2_{\mathrm{HS}},
\end{align}
where $\{(x^\ell, y^\ell)\}_{\ell \in [L]}$ are independently and identically sampled from $\PP_{\cX, \cY}$, $\norm{\cdot}_{\mathrm{HS}}$ denotes the Hilbert-Schmidt norm, and $\lambda > 0$ is the regularization parameter.
Recall from \eqref{eq:feature-kernel} that $\phi$ and $\varphi$ are the feature mappings associated with the RKHSs $\cH_x$ and $\cH_y$. To ease the presentation, we view the space $\ell_2$ as an (infinite-dimensional) vector space and consider the feature mappings $\phi: \fX \rightarrow \RR^{d_\phi}$ and $\varphi: \fY \rightarrow \RR^{d_\varphi}$, where $d_\phi$ and $d_\varphi$ can be infinity.
We write $\phi(X) = (\phi(x^1), \ldots, \phi(x^L))^\top \in \RR^{L \times d_\phi}$ and $\varphi(Y) = (\phi(y^1), \ldots, \phi(y^L))^\top \in \RR^{L \times d_\varphi}$. 
Also, we define the the (uncentered) empirical  covariance operator $\hat \cC_{\cX\cX}$ and (uncentered) empirical cross-covariance operator $\hat \cC_{\cY\cX}$ as follows, 
\begin{align}
	\label{eq:emp-cov}
	\hat \cC_{\cX\cX} &= L^{-1} \sum_{\ell=1}^{L}\phi(x^\ell) \phi(x^\ell)^\top = L^{-1}\phi(X)^\top \phi(X) \in \RR^{d_\phi\times d_\phi} \nonumber \\
	\hat \cC_{\cY\cX} &= L^{-1} \sum_{\ell=1}^{L} \varphi(y^\ell)\varphi(x^\ell)^\top = L^{-1}\varphi(Y)\phi(X)^\top \in \RR^{d_\varphi\times d_\phi}.
\end{align}
Then, the solution to \eqref{eq:reg-int-emp} is 
\begin{align*}
	\hat \cC_{\cY \given \cX}^\lambda = \varphi(Y)^\top \phi(X)\bigl(\phi(X)^\top \phi(X) + \lambda \cI \bigr)^{-1} =  \hat \cC_{\cY\cX}(\hat \cC_{\cX\cX} + L^{-1}\lambda \cI)^{-1} \in \RR^{d_\varphi \times d_\phi}.
\end{align*}
We denote by $\hat \cme_\lambda(x, \PP_{\cX,\cY}) = \hat \cC_{\cY \given \cX} \phi(x)\in \RR^{d_\varphi}$ the empirical conditional mean embedding.
Note that 
\begin{align*}
	\phi(X) \bigl(\phi(X)^\top \phi(X) + \lambda \cI\bigr)^{-1} = \bigl(\phi(X) \phi(X)^\top + \lambda I\bigr)^{-1} \phi(X).
\end{align*}
Thus, it holds that
\begin{align}
\label{eq:reg-int-emp2}
	\hat \cme_\lambda(x, \PP_{\cX,\cY}) & = \hat \cC_{\cY \given \cX}^\lambda \phi(x) \nonumber\\ 
	& = \hat \cC_{\cY\cX}(\hat \cC_{\cX\cX} + L^{-1}\lambda \cI)^{-1} \fk(x, \cdot) \nonumber \\
	& =\varphi(Y)^\top\phi(X) \bigl( \phi(X)^\top \phi(X) + \lambda \cI \bigr)^{-1} \phi(x) \nonumber \\
	& = \varphi(Y)^\top \bigl(\phi(X) \phi(X)^\top + \lambda I\bigr)^{-1} \phi(X)\phi(x) \nonumber \\
	& = \varphi(Y)^\top (\fk(X, X) + \lambda I)^{-1} \fk(X, x).
\end{align}
Here $\fk(X, X) = \phi(X)\phi(X)^\top = (\fk(x^i, x^j) )_{i, j \in [L]} \in \RR^{L\times L}$ is the Gram matrix and $\fk(X, x) = \phi(X) \phi(x) = (\fk(x^1, x), \ldots, \fk(x^L, x)) \in \RR^L$. 









\section{Attention Recovers Latent Posterior}

\subsection{Gaussian Process Regression} \label{sec:gp-regress}

\noindent{\bf Gaussian Process Regression.}
We say that $f$ follows a Gaussian process $\gp(\mu, \fk)$ on $\RR^d$ if for any $x^1, \ldots, x^L$, $(f(x^1), \ldots, f(x^L))$ follows a Gaussian distribution with mean $(\mu(x^1), \ldots, \mu(x^L))$ and covariance $(\fk(x^i, x^j))_{i, j \in [L]}$. 
Here $\mu(x)= \EE[f(x)]$ is the mean function and  $\fk(x, x') = \EE[(f(x) - \mu(x))(f(x') - \mu(x'))] $ is the covariance (or kernel) function, where $f$ is random. 
We take $\gp(0, \fk)$ as the prior of $f$. 
Given a dataset $\cD = \{(x^\ell, y^\ell)\}_{\ell \in [L]}$ from the regression model $ y^\ell = f(x^\ell) + \epsilon^\ell$ with $\epsilon^\ell \sim N(0, \lambda I)$, the posterior of $f$ is a Gaussian process with mean $\mu_\cD(x)$ and covariance $\fk_\cD(x, x')$ \citep{schulz2018tutorial} as follows,
\begin{align*}
	\mu_\cD(x) &= \fk(x, X) \bigl(\fk(X, X) + \lambda I\bigr)^{-1} Y, \\
	\fk_\cD(x, x') & = \fk(x, x') - \fk(x, X) \bigl(\fk(X, X) + \lambda I \bigr)^{-1} \fk(X, x').
\end{align*}
Here $\fk(x, X) = (\fk(x, x^\ell))_{\ell \in [L]}^\top \in \RR^{1 \times L}$, $\fk(X, X) = (\fk(x^i, x^j))_{i,j \in [L]} \in \RR^{L\times L}$, $\fk(X, x') = (\fk(x^\ell, x'))_{\ell \in [L]} \in \RR^{L}$, and $Y = (y^1, \ldots, y^L) \in \RR^L$

\vskip4pt

\noindent{\bf Rigorous Characterization of Latent Variable Model.}
We  provide a rigorous characterization of the advanced infinite-dimensional example of the latent variable model in \S\ref{sec::Gaussian_Posterior}.We consider the following model,
\begin{align}
	\label{eq:lvm-gp}	
	r^\ell = f(c^\ell) + \epsilon^\ell, \qquad 	r^\mask = f(c^\mask) + \epsilon.
\end{align}
Here $f = (f_1, \ldots, f_{d})$ with $f_i \sim \gp(0, \fk(\cdot, \cdot))$ for any $i \in [d]$ and $\epsilon^\ell$ and $\epsilon$ are independent Gaussian noises drawn from $N(0, \lambda I)$. Then, following the Gaussian process regression, we recover \eqref{eq:cme-att} as the mean of the posterior of the Gaussian process.

\subsection{Implication of Convergence with $L\rightarrow \infty$} \label{sec:impl-limit}

\noindent{\bf Necessity of Multiple Heads.}
Based on the definition of the attention mechanism $\att$ in \eqref{eq:attn}, we define the multihead attention as
\#\label{eq:mha}
\mha(q, X; W) = \sum_{i= 1}^h \head_i \in \RR^{d}.
\#
Here $h\in \NN_+$ is the head number, $W = \{(W_i^\tq, W_i^\tk, W_i^\tv)\}_{i = 1}^h$ with $W_i^\tq \in \RR^{d \times d_\tp}$, $W_i^\tk\in \RR^{d \times d_\tp}$, and $W_i^\tv\in \RR^{d \times d}$ is the learnable parameter, and
\$
\head_i = \att(q, K_i, V_i) \in \RR^{d}, \qquad \text{where} \quad K_i = X W_i^\tk \in \RR^{L\times d_\tp}, \quad V_i = XW_i^\tv \in \RR^{L\times d}.
\$
In the multihead attention, we set $d = d_\tp \cdot h$, where $h$ is the head number and $d_\tp$ is the dimension of the key and the query.
We remark that the multihead attention defined in \eqref{eq:mha} is written in the summation form, which is equivalent to the concatenation form \citep{vaswani2017attention}. To see this, we consider the concatenation form of the multihead attention,
\begin{align}\label{eq:mha2}
	\tilde \mha(q, X; \tilde W) = \bigl((W^{\mathrm{o}}_1)^\top, \ldots, (W^{\mathrm{o}}_h)^\top \bigr) \begin{pmatrix}
		\tilde \head_1 \\
		\ldots \\
		\tilde \head_h
	\end{pmatrix}
 = \sum_{i = 1}^{h} (W^{\mathrm{o}}_i)^\top \tilde \head_i ,
\end{align}
where $W^{\mathrm{o}} \in \RR^{d_\tp \times d}$ is a learnable parameter with the $i$-th block $W^{\mathrm{o}}_i$ and the $i$-th head $\tilde \head_i$ is obtained via
\$
\tilde \head_i = \att(q, K_i, \tilde V_i) \in \RR^{d_\tp}, \qquad \text{where} \quad K_i = X W_i^\tk \in \RR^{L\times d_\tp}, \quad V_i = X \tilde W_i^\tv \in \RR^{L\times d_\tp}.
\$
Here $\tilde W_i^\tv \in \RR^{d \times d_\tp}$.
We see that the \eqref{eq:mha} and \eqref{eq:mha2} are equivalent when $\head_i = (W^{\mathrm{o}}_i)^\top \tilde \head_i$, which holds when $ W_i^\tv = \tilde W_i^\tv W^{\mathrm{o}}_i$.


We use $\EE[\cV\given \cK = q]$ to demonstrate the necessity of using multiple heads in the multihead attention. 
Note that the key and value are obtained by the following mappings,
\begin{align*}
k^\ell =  (W^\tk)^\top x^\ell, \qquad v^\ell = (W^\tv)^\top x^\ell,
\end{align*}
where $x^\ell \in \RR^d$ is the input token and $W^\tk \in \RR^{d\times d_\tp}$, $W^\tv \in \RR^{d\times d}$ are the learnable parameters. 
We consider a single-head attention, where $h = 1$, $d_\tp =d $, and $W^\tk\in \RR^{d\times d}$ is invertible.
We denote by $\cK$, $\cV$, and $\cX$ the random variable with the same distribution as $k^\ell$, $v^\ell$, and $x^\ell$, respectively. By Propositions \ref{prop:attn-cme} and \ref{prop:kernel_attn_story1}, we have 
\begin{align*}
	\att(q, K, V) \approx \EE[\cV \given \cK = q] = \EE\bigl[ (W^\tv)^\top \cX \,\big|\, (W^\tk)^\top \cX = q \bigr] = \bigl((W^\tk)^{-1} W^\tv \bigr)^\top q,
\end{align*}
which is a linear mapping and fails to capture the nonlinear interaction query $q$ and the input sequence $X$.
In other words, the single-head attention  becomes a linear mapping in the limit with $L\rightarrow \infty$. 
In contrast, when  $h > 1$, we have $d_\tp = d / h < d$, which implies that the matrix $W^\tk \in \RR^{d \times d}$ is not invertible.
Thus, using multiple heads avoid the degenerating issue.

%
%
%

\vspace{4pt}

\noindent{\bf Connection to Instrumental Variable.} 
We draw a connection from the attention mechanism to the instrumental variable model.
Instrumental variable regression estimates the causal relationship between the input $\cX$ and the output $\cY$. 
Specifically, when $(\cX, \cY)$ is confouneded, an instrumental variable $\cW$ is effective in identifying the causal relationship between $\cX$ and $\cY$. 
Intuitively, $\cW$ is an instrumental variable if it influences $\cY$ only through $\cX$ which is formalized as follows.
\begin{assumption}[Instrumental Variable Model]
    \label{asp:causal}
    Let $(\cX, \cY, \cW)$ be a random variable on the space $\fX\times \fY \times \fW$ with joint distribution $\PP_{\cX, \cY, \cW}$. We assume that 
    \begin{itemize}
    \item[(i)] $\cY = g(\cX) + \epsilon$ and $\EE[\epsilon \given \cW = w] = 0$ for any $w\in \fW$, and
    \item[(ii)] $\PP_{\cX\given \cW}(x \given w)$ does not remain when $w$ varies.
\end{itemize}
\end{assumption}

Under Assumption \ref{asp:causal}, $\cW$ is an instrumental variable.
Specifically, (i) of Assumption \ref{asp:causal} is the exclusion restriction, where function $g$ is the structural function of interest and $\epsilon$ is the confounding noise.  
Also, (ii) of Assumption \ref{asp:causal} is the relevance condition, which ensures that $\cW$ is informative in the sense that it depends on $w$ in a nontrivial manner. 
We remark that the instrumental variable model generalizes the standard regression model.
To see this, when $\cX= \cW$, the estimation of $g$ reduces to standard regression of unconfounded inputs and it holds that $g(\cdot) = \EE[\cY \given \cX = \cdot]$. In particular, the instrumental variable model allows that $\cX$ and $\epsilon$ are confounded, i.e., $\cX$ and $\epsilon$ are dependent. 
By Assumption \ref{asp:causal}, we have the following estimation equation
\begin{align}
    \label{eq:iv1}
    \EE [\cY\given \cW = w] = \EE\bigl[g(\cX) \given \cW = w\bigr].
\end{align}
The right-hand side of \eqref{eq:iv1} provides a two-stage method for estimating the function $g$. At the first stage, we estimate the conditional mean mean embedding of $\PP_{\cX\given \cW}$. Then, at the second stage, we estimate the function $g$ via regressing $\cY$ on the empirical conditional mean mean embedding of $\PP_{\cX\given \cW}$ \citep{singh2019kernel}.

To ease the presentation, we consider the following mapping,
\$
	g_\theta \circ \att(q, K, V) \in \RR^{d},
\$
where $g_\theta$ is a function approximator with a learnable parameter $\theta$. For example, $g_\theta$ is a linear or kernel function.
By Proposition \ref{prop:attn-cme} and Proposition \ref{prop:kernel_attn_story1}, it holds that
\begin{align*}
	g_\theta \circ \att(q, K, V) \approx g_\theta \bigl(\EE[\cV \given \cK =q]\bigr), \qquad \text{as} \quad L \rightarrow \infty.
\end{align*}
Let the target variable be $y$. Then, the learning objective takes the following form,
\begin{align*}
	\min_\theta \hat \EE\Bigl[\norm[\big]{y - g_\theta(\EE[\cV \given \cK = q])}_2^2 \Bigr],
\end{align*}
which corresponds to the second stage of estimating the instrumental variable model. Note that $\EE[\cV \given \cK = q]$ is the conditional mean embedding of $\PP_{\cV \given \cK}$.
Thus, the key $\cK$ can be viewed as the instrumental variable \citep{pearl2009causality}, which handles the endogeneity.
We provide an alternative view on how the attention mechanism performs relational reasoning as a causal inference procedure.

\subsection{Proof of Lemma \ref{lem:sufficiency}} \label{sec:pf-sufficiency}

\begin{proof}
First, we prove the statement that $b_z(X) = \PP(z = \cdot \given X)$ is a minimal sufficient statistic of $X$ for $z$. To see the sufficiency of $b_z(X)$ for $z$, note that
\begin{align*}
	\PP(z \given X) = \PP\bigl(z \biggiven \PP(z=\cdot \given X)\bigr) = \PP\bigl(z \biggiven b_z(X) \bigr).
\end{align*}
To see $b_z(X)$ is the minimal sufficient statistic, let $U(X)$ be another sufficient statistic of $X$ for $z$. Then, we have
\$
\PP(z \,|\,X) = \PP\bigl(z \biggiven U(X) \bigr),
\$
which implies that $b_z(X)$ is a function of $U(X)$. Thus, $b_z(X)$ is minimal.

Second, we prove the statement that $b_z(X)$ is a minimal sufficient statistics of $X$ for $y$. To see the sufficiency of $b_z(X)$ for $y$, note that
\begin{align*}
	\PP(y \given X) = \int \PP(y\given z) \cdot 
	\PP(z \given X)\ud z,
\end{align*}
which implies that $\PP(y \given X) = \PP(y\given b_z(X))$ since $\PP(y \given X)$ only depends on $X$ through $b_z(X)$.
Suppose that $U(X)$ is a sufficient statistic of $X$ for $y$. We have
\$
\int \PP(y\,|\,z) \cdot 
\PP(z \,|\,X)\ud z=
\PP(y \,|\,X) &= \PP\bigl(y \biggiven U(X) \bigr)=
\int \PP(y\,|\,z) \cdot 
\PP\bigl(z \biggiven U(X) \bigr)\ud z.
\$
By the definition of $\cT$ in \eqref{eq:def-cT}, we then have that 
\$
b_z(X) = \PP(z = \cdot \given X)=\cT^{-1}\Bigl( \int \PP(y=\cdot\,|\,z) \cdot 
\PP\bigl (z \,\big|\,U(X) \bigr)\ud z \Bigr),
\$
which implies that $b_z(X)$ is a function of $U(X)$. Thus, $b_z(X)$ is minimal.
\end{proof}

\subsection{Proof of Proposition \ref{prop:attn-cme}}
\label{sec:pf-thm-attn-cme}

\begin{proof}
For notational simplicity, we denote by $\norm{\cdot}$ the RKHS norm for elements in an RKHS and the operator norm for operators between two RKHSs. Also, we denote by $\cH_k$ and $\cH_v$ the RKHSs for the key and the value with the kernel functions $\fk$ and $\fl$, respectively. 
Note that we consider the Euclidean kernel $\fl(v, v') = v^\top v'$ for the value, which uses the identity mapping $\varphi$ as the feature mapping.
Recall the definition of the empirical covariance operator and the empirical cross-covariance operator in \eqref{eq:emp-cov}. Correspondingly, we write
\begin{align*}
	\hat \cC_{\cK\cK} &= L^{-1}\phi(K)^\top \phi(K) \in \RR^{d_\phi\times d_\phi} \nonumber \\
	\hat \cC_{\cV\cK} &= L^{-1}\varphi(V)^\top\phi(K) \in \RR^{d_\varphi\times d_\phi}, \\
	\hat \cC_{\cV\cV} &= L^{-1}\varphi(V)^\top \varphi(V) \in \RR^{d_\varphi\times d_\varphi}.
\end{align*}
Here $\phi(K) = (\phi(k^1), \ldots, \phi(k^L))^\top \in \RR^{L\times d_\phi}$ and $\varphi(V) = (\phi(v^1), \ldots, \phi(v^L))^\top \in \RR^{L\times d_\varphi}$
By the definition of the CME attention in \eqref{eq:cme-att-def} and the equality in \eqref{eq:reg-int-emp2}, we have that
\begin{align*}
	\att_{\cme}(q, K, V) = \hat\cC_{\cV\cK} (\hat \cC_{\cK\cK} + L^{-1}\lambda \cI)^{-1} \phi(q),
\end{align*}
which implies that $\att_\cme$ recovers the empirical conditional mean embedding.
By \eqref{eq:cme-cov}, it holds that
    \begin{align}\label{eq:tac1}
		&\norm[\big]{\att(q, K, V) - \cme(q, \PP_{\cK, \cV})} \noend 
		&\quad \le \underbrace{\norm[\big]{ \hat\cC_{\cV\cK} (\hat \cC_{\cK\cK} + L^{-1}\lambda \cI)^{-1} \phi(q) - \cC_{\cV\cK} (\cC_{\cK\cK} + L^{-1}\lambda \cI)^{-1} \phi(q)}}_{\displaystyle\text{(i)}} \noend 
		&\qquad + \underbrace{\norm[\big]{ \cC_{\cV\cK} (\cC_{\cK\cK} + L^{-1}\lambda \cI)^{-1} \fk(q, \cdot) - \cC_{\cV\cK} \cC_{\cK\cK}^{-1} \fk(q, \cdot)}}_{\displaystyle\text{(ii)}}.
	\end{align}

\vskip4pt

\noindent{\bf Upper bounding term (i) of \eqref{eq:tac1}.} 
We adapt the proof from \cite{song2009hilbert}.
It suffices to upper bound $\norm{ \hat\cC_{\cV\cK} (\hat \cC_{\cK\cK} + L^{-1}\lambda \cI)^{-1} - \cC_{\cV\cK} (\cC_{\cK\cK} + L^{-1}\lambda \cI)^{-1}}$. It holds that 
\begin{align}\label{eq:tac11}
    &\norm[\big]{ \hat\cC_{\cV\cK} (\hat \cC_{\cK\cK} + L^{-1}\lambda \cI)^{-1} - \cC_{\cV\cK} (\cC_{\cK\cK} + L^{-1}\lambda \cI)^{-1}} \\
    & \quad \le \norm[\Big]{ \hat \cC_{\cV\cK} \bigl((\hat \cC_{\cK\cK} + L^{-1}\lambda \cI)^{-1} - (\cC_{\cK\cK} + L^{-1}\lambda \cI)^{-1} \bigr) } + \norm[\big]{ (\hat\cC_{\cV\cK} - \cC_{\cV\cK}) ( \cC_{\cK\cK} + L^{-1}\lambda \cI)^{-1} } \noend
    & \quad  =  \norm[\big]{ \hat \cC_{\cV\cK} (\hat \cC_{\cK\cK} + L^{-1}\lambda \cI)^{-1}(\hat \cC_{\cK\cK} - \cC_{\cK\cK} ) (\cC_{\cK\cK} + L^{-1}\lambda \cI)^{-1} } + \norm[\big]{ (\hat\cC_{\cV\cK} - \cC_{\cV\cK}) ( \cC_{\cK\cK} + L^{-1}\lambda \cI)^{-1} }. \nonumber
\end{align}
For the first term on the right-hand side of \eqref{eq:tac11}, we have the operator decomposition that $\hat\cC_{\cV\cK} = \hat\cC_{\cV\cV}^{1/2} \cW \hat \cC_{\cK\cK}^{1/2}$ for $\cW$ such that  $\norm{\cW} \le 1$.
Then, we have that
\begin{align}\label{eq:tac12}
    & \norm[\big]{ \hat \cC_{\cV\cK} (\hat \cC_{\cK\cK} + L^{-1}\lambda \cI)^{-1}(\hat \cC_{\cK\cK} - \cC_{\cK\cK} ) (\cC_{\cK\cK} + L^{-1}\lambda \cI)^{-1} } \noend
    & \quad \le \norm{\hat\cC_{\cV\cV}}^{1/2} \cdot \norm[\big]{ \hat \cC_{\cK\cK}^{1/2} (\hat \cC_{\cK\cK} + L^{-1}\lambda \cI)^{-1/2} } \cdot \norm[\big]{(\hat \cC_{\cK\cK} + L^{-1}\lambda \cI)^{-1/2}} \cdot \norm[\big]{ (\hat \cC_{\cK\cK} - \cC_{\cK\cK} ) (\cC_{\cK\cK} + L^{-1}\lambda \cI)^{-1} } \noend
    & \quad \le (L^{-1}\lambda)^{-1/2} \cdot \norm[\big]{ (\hat \cC_{\cK\cK} - \cC_{\cK\cK} ) (\cC_{\cK\cK} + L^{-1}\lambda \cI)^{-1} },
\end{align}
where the last inequality follows from 
\begin{align*}
	\norm{\hat\cC_{\cV\cV}}^2 = L^{-1} \sum_{\ell = 1}^{L} \norm{v^\ell}_2^2 \le 1, \quad 
	\hat \cC_{\cK\cK} (\hat \cC_{\cK\cK} + L^{-1}\lambda \cI)^{-1} \le \cI, \quad 
	(\hat \cC_{\cK\cK} + L^{-1}\lambda \cI)^{-1} \le (L^{-1}\lambda)^{-1} \cI.
\end{align*}
Plugging \eqref{eq:tac12} into \eqref{eq:tac11}, we have
\begin{align}\label{eq:tac13}
    &\norm[\big]{ \hat\cC_{\cV\cK} (\hat \cC_{\cK\cK} + L^{-1}\lambda \cI)^{-1} - \cC_{\cV\cK} (\cC_{\cK\cK} + L^{-1}\lambda \cI)^{-1}} \\
    &\quad \le (L^{-1}\lambda)^{-1/2} \cdot \norm[\big]{ (\hat \cC_{\cK\cK} - \cC_{\cK\cK} ) (\cC_{\cK\cK} + L^{-1}\lambda \cI)^{-1} } + \norm[\big]{ (\hat\cC_{\cV\cK} - \cC_{\cV\cK}) ( \cC_{\cK\cK} + L^{-1}\lambda \cI)^{-1} }. \nonumber
\end{align}
In what follows, we upper bound the second term on the right-hand side of \eqref{eq:tac13} using Lemma \ref{lem:cme-concen}. We define $\xi: \RR^{d_\tp} \times \RR^d \rightarrow  \cH_k \otimes \cH_v$ as follows,
	\$
		\xi(k, v) = \varphi(v) \phi(k)^\top (\cC_{\cK\cK} + L^{-1}\lambda \cI)^{-1}.
	\$
	Since $\norm[\big]{(\cC_{\cK\cK} + L^{-1}\lambda \cI)^{-1}} \le (L^{-1}\lambda)^{-1}$, we have that
	\begin{align*}
		\norm[\big]{\xi(k, v)} = \norm[\big]{(\cC_{\cK\cK} + L^{-1}\lambda \cI)^{-1}} \cdot \norm[\big]{\varphi(v)} \cdot \norm[\big]{\phi(k)} \le C \cdot (L^{-1}\lambda)^{-1},
	\end{align*}
	where $C >0$ is an absolute constant.
	In addition, we have that
	\begin{align*}
		\EE\Bigl[\norm[\big]{\xi(k, v)}^2\Bigr] & = \EE\Bigl[ \norm[\big]{  \phi(k)^\top(  \cC_{\cK\cK} + L^{-1}\lambda \cI)^{-1}}^2 \cdot \norm[\big]{\varphi(v)}^2 \Bigr] \noend
		                                  & \le  \EE\Bigl[ \norm[\big]{(  \cC_{\cK\cK} + L^{-1}\lambda \cI)^{-1}  \phi(k)}^2 \Bigr] \noend
		                                  & = \EE\Bigl[ \inp[\big]{(  \cC_{\cK\cK} + L^{-1}\lambda \cI)^{-2}  \phi(k)}{  \phi(k)} \Bigr] \\
		                                  & \le (L^{-1}\lambda)^{-1} \cdot \EE\Bigl[ \inp[\big]{(  \cC_{\cK\cK} + L^{-1}\lambda \cI)^{-1}  \phi(k)}{  \phi(k)} \Bigr].
	\end{align*}
	Using the trace operator, we have
	\begin{align*}
		\EE\Bigl[\norm[\big]{\xi(k, v)}^2\Bigr] & \le \EE\Bigl[ \Tr\bigl((  \cC_{\cK\cK} + L^{-1}\lambda \cI)^{-2}  \phi(k)  \phi(k)^\top \bigr) \Bigr] \noend
		                                  & = \Tr\bigl((  \cC_{\cK\cK} + L^{-1}\lambda \cI)^{-2}  \cC_{\cK\cK} \bigr) \noend
		                                  & \le (L^{-1}\lambda)^{-1} \cdot  \Tr\bigl((  \cC_{\cK\cK} + L^{-1}\lambda \cI)^{-1}  \cC_{\cK\cK}\bigr) \\
		                                  &= (L^{-1}\lambda)^{-1} \cdot \Gamma(L^{-1}\lambda).
	\end{align*}
	Here $\Gamma(L^{-1}\lambda)$ is the effective dimension of $\cC_{\cK\cK}$, which is defined as follows,
	\begin{align*}
		\Gamma(L^{-1}\lambda) =  \Tr\bigl((  \cC_{\cK\cK} + L^{-1}\lambda \cI)^{-1}  \cC_{\cK\cK}\bigr).
	\end{align*}
	Applying Lemma \ref{lem:cme-concen} with $B = C (L^{-1}\lambda)^{-1}$ and $\sigma^2 = (L^{-1}\lambda)^{-1} \cdot \Gamma(L^{-1}\lambda)$, we have with probability at least $1- \delta$ that
	\begin{align}
		\label{eq:ce20}
		\norm[\big]{\hat \cC_{\cV\cK} (  \cC_{\cK\cK} + L^{-1}\lambda \cI)^{-1} -   \cC_{\cV\cK} (  \cC_{\cK\cK} + L^{-1}\lambda \cI)^{-1}} \le C \cdot \biggl( \frac{2}{\lambda} + \sqrt{\frac{\Gamma(L^{-1}\lambda)}{\lambda}} \biggr) \log\frac{2}{\delta},
	\end{align}
	where $C>0$ is an absolute constant. Similarly, we have with probability at least $1-\delta$ that
	\begin{align}
		\label{eq:ce221}
		\norm[\big]{\hat \cC_{\cK\cK} (  \cC_{\cK\cK} + L^{-1}\lambda \cI)^{-1} -   \cC_{\cK\cK} (  \cC_{\cK\cK} + L^{-1}\lambda \cI)^{-1}} \le C' \cdot \biggl( \frac{2}{\lambda} + \sqrt{\frac{\Gamma(L^{-1}\lambda)}{\lambda}} \biggr)  \log\frac{2}{\delta}.
	\end{align}
Here $C'>0$ is an absolute constant.
Plugging \eqref{eq:ce20} and \eqref{eq:ce221} into \eqref{eq:tac13}, we have with probability at least $1- \delta$ that
	\begin{align}\label{eq:tac1100}
	    &\norm[\big]{ \hat\cC_{\cV\cK} (\hat \cC_{\cK\cK} + L^{-1}\lambda \cI)^{-1} - \cC_{\cV\cK} (\cC_{\cK\cK} + L^{-1}\lambda \cI)^{-1}} \noend 
	    &\quad \le C'' \cdot \sqrt{\frac{L}{\lambda}} \cdot \biggl( \frac{2}{\lambda} + \sqrt{\frac{\Gamma(L^{-1}\lambda)}{\lambda}} \biggr)  \log\frac{2}{\delta}.
	\end{align}

\vskip4pt	

\noindent{\bf Upper bounding term (ii) of \eqref{eq:tac1}.}
We adapt the proof from \cite{fukumizu2015nonparametric}. 
For any $g \in \cH_k$, it holds that
\begin{align*}
	\inp{\cC_{\cV\cK} g}{ \cC_{\cV\cK} g} & = \EE\bigl[\fl(\cV, \bar \cV) g(\cK) g(\bar \cK)\bigr] \\
	& = \EE\Bigl[ \EE\bigl[\fl(\cV, \bar \cV) \biggiven \cK, \bar \cK \bigr] g(\cK) g(\bar \cK)\Bigr] \\
	& = \inp[\Big]{(\cC_{\cK\cK} \otimes \cC_{\cK\cK})\EE\bigl[\fl(\cV, \bar \cV) \biggiven \cK = \cdot , \bar \cK = \dagger \bigr]}{g \otimes g}.
\end{align*}
Similarly, we have for any $q \in \RR^{d_\tp}$ and any $g \in \cH_k$ that
\begin{align*}
	\inp[\Big]{ \cC_{\cV\cK}}{\EE\bigl[\fl(\cV, \cdot) \biggiven \cK = q\bigr] } &= \inp[\Big]{\EE\bigl[\fl(\cV, \bar\cV) \biggiven \cK = q, \cK = \dagger\bigr]}{\cC_{\cK\cK} g} \\
	& = \inp[\Big]{(\cI \otimes \cC_{\cK\cK}) \EE\bigl[\fl(\cV, \bar \cV) \biggiven \cK = \cdot , \bar \cK = \dagger \bigr]}{\fl(\cdot, q) \otimes g}.
\end{align*}
By setting $g = (\cC_{\cK\cK} + L^{-1}\lambda \cI)^{-1} \fk(q, \cdot)$, we have that
\begin{align*}
    &\norm[\big]{ \cC_{\cV\cK} (\cC_{\cK\cK} + L^{-1}\lambda \cI)^{-1} \fk(q, \cdot) - \cC_{\cV\cK} \cC_{\cK\cK}^{-1} \fk(q, \cdot)}^2 \noend
    & \quad = \inp[\big]{\cC_{\cV\cK} (\cC_{\cK\cK} + L^{-1}\lambda \cI)^{-1} \fk(q, \cdot) - \cC_{\cV\cK} \cC_{\cK\cK}^{-1} \fk(q, \cdot)}{\cC_{\cV\cK} (\cC_{\cK\cK} + L^{-1}\lambda \cI)^{-1} \fk(q, \cdot) - \cC_{\cV\cK} \cC_{\cK\cK}^{-1} \fk(q, \cdot)} \noend 
    & \quad = \bigg\langle \Big( (\cC_{\cK\cK} + L^{-1}\lambda \cI)^{-1} \cC_{\cK\cK} \otimes (\cC_{\cK\cK} + L^{-1}\lambda \cI)^{-1} \cC_{\cK\cK} - \cI \otimes  (\cC_{\cK\cK} + L^{-1}\lambda \cI)^{-1} \cC_{\cK\cK}    \noend
    &   \quad \qquad (\cC_{\cK\cK} + L^{-1}\lambda \cI)^{-1} \cC_{\cK\cK} \otimes \cI + \cI \otimes \cI \Big) \EE\bigl[\fl(\cV, \bar \cV) \biggiven \cK = \cdot , \bar \cK = \dagger \bigr], \fk(q, \cdot) \otimes \fk(q, \dagger) \bigg\rangle. 
\end{align*}
Note that $\EE[ \fl(v, \bar v) \given k = \cdot , \bar k = \dagger ] \in \cH_k \otimes \cH_k$ is in the range of $\cC_{\cK\cK}\otimes \cC_{\cK\cK}$. We define $ \tilde \cC \in \cH_k \times \cH_k$ such that $(\cC_{\cK\cK}\otimes \cC_{\cK\cK})\tilde \cC = \EE[ \fl(v, \bar v) \given k = \cdot , \bar k = \dagger ]$. Let $\{\lambda_i\}_{i = 1}^\infty$ and $\{\varphi_i\}_{i =1}^\infty$ be the eigenvalues and eigenvectors of $\cC_{\cK\cK}$, respectively. Then, we have that
\begin{align*}
	&\norm[\big]{ \cC_{\cV\cK} (\cC_{\cK\cK} + L^{-1}\lambda \cI)^{-1} \fk(q, \cdot) - \cC_{\cV\cK} \cC_{\cK\cK}^{-1} \fk(q, \cdot)}^4 \\
    & \quad \le \bigg\|\Big( (\cC_{\cK\cK} + L^{-1}\lambda \cI)^{-1} \cC_{\cK\cK} \otimes (\cC_{\cK\cK} + L^{-1}\lambda \cI)^{-1} \cC_{\cK\cK} - \cI \otimes  (\cC_{\cK\cK} + L^{-1}\lambda \cI)^{-1} \cC_{\cK\cK}    \noend
    &   \quad \qquad (\cC_{\cK\cK} + L^{-1}\lambda \cI)^{-1} \cC_{\cK\cK} \otimes \cI + \cI \otimes \cI \Big) \EE\bigl[\fl(\cV, \bar \cV) \biggiven \cK = \cdot , \bar \cK = \dagger \bigr] \bigg\|^2 \noend
    & \quad = \sum_{i,j} \biggl( \frac{\lambda_i^2}{\lambda_i + L^{-1}\lambda} \frac{\lambda_j^2}{\lambda_j + L^{-1}\lambda } - \frac{\lambda_i^2 \lambda_j}{\lambda_i + L^{-1}\lambda} - \frac{\lambda_j^2 \lambda_i}{\lambda_j + L^{-1}\lambda} + \lambda_i \lambda_j \biggr)^2 \cdot \inp{\varphi_i \otimes \varphi_j}{ \tilde \cC}^2 \noend
    & \quad = \sum_{i,j} \biggl( \frac{\lambda_i \lambda_j (L^{-1}\lambda)^2}{(\lambda_i + L^{-1}\lambda)(\lambda_j + L^{-1}\lambda)} \biggr)^2 \cdot  \inp{\varphi_i \otimes \varphi_j}{ \tilde \cC}^2 \noend
    & \quad \le (L^{-1}\lambda)^4 \cdot \norm{\tilde \cC}^2.
\end{align*}
Thus, we have 
\#\label{eq:tac223}
    \norm[\big]{ \cC_{\cV\cK} (\cC_{\cK\cK} + \lambda \cI)^{-1} \fk(q, \cdot) - \cC_{\cV\cK} \cC_{\cK\cK}^{-1} \fk(q, \cdot)}_2 \le C\cdot \lambda L^{-1},
\#
where $C > 0$ is an absolute constant.

Plugging \eqref{eq:tac1100} and \eqref{eq:tac223} into \eqref{eq:tac1}, we have with probability at least $ 1- \delta$ that
\begin{align*}
	\norm[\big]{\att(q, K, V) - \cme(q, \PP_{\cK, \cV})} & \le \cO\biggl( \sqrt{\frac{L}{\lambda}} \cdot \biggl( \frac{2}{\lambda} + \sqrt{\frac{\Gamma(L^{-1}\lambda)}{\lambda}} \biggr)  \log\frac{1}{\delta} + \lambda L^{-1}\biggr).
\end{align*}
Thus, we complete the proof of Proposition~\ref{prop:attn-cme}.
\end{proof}

\subsection{Proof of Proposition \ref{prop:kernel_attn_story1}}
\label{sec::pf_lem_kernel_attn_story1}
\begin{proof}
Under the condition that $\hat\PP^{\fk}_{\cV\given \cK}(v\given q) \rightarrow \PP(v \given q)$ uniformly for any $q \in \SSS^{d_\tp -1}$ as $L\rightarrow \infty$, we have 
\$
\int v \hat\PP^{\fk}_{\cV\given \cK}(v\given q) \ud v \rightarrow \EE[\cV \given \cK = q] \qquad \text{as} \quad L\rightarrow \infty.
\$
Moreover, it holds that
\#\label{eq::pf_lem_kernel_story1_eq1}
\int v \hat\PP^{\fk}_{\cV\given \cK}(v\given q) \ud v 
&= \iota \cdot \int_{\SSS^{d - 1}} v\cdot \frac{\sum^L_{\ell=1}\fk(k^\ell, q)\cdot\fk(v^\ell, v)}{\sum^L_{\ell=1}\fk(k^\ell, q)} \ud v \notag\\
&= \frac{\iota \cdot \sum^L_{\ell = 1} \fk(k^\ell, q)\cdot\int_{\SSS^{d - 1}} v \cdot\fk(v^\ell, v) \ud v }{\sum^L_{\ell=1}\fk(k^\ell, q)},
\#
where $\SSS^{d-1}$ is the $(d-1)$-dimensional unit sphere. It suffices to calculate the integration term$\int_{\SSS^{d - 1}} v \cdot\fk(v^\ell, v) \ud v$.
To this end, we utilize the following lemma.
\begin{lemma}
\label{lem::calculation_of_integral}
Let $\fk(a, b) = \exp(a^\top b / \gamma)$ be the exponential kernel with a fixed $\gamma >0$. It holds for any $b\in\SSS^{d - 1}$ that
\$
\int_{\SSS^{d - 1}} a\cdot \fk(a, b)\ud a = C_1 \cdot b,
\$
where $C_1>0$ is an absolute constant.

\end{lemma}
\begin{proof}
See \S\ref{sec::pf_lem_calculation_of_integral} for a detailed proof.
\end{proof}
By Lemma \ref{lem::calculation_of_integral}, it holds for the right-hand side of \eqref{eq::pf_lem_kernel_story1_eq1} that
\$
\iota \cdot C_1 \cdot \frac{\sum^L_{\ell = 1} \fk(k^\ell, q)\cdot v^\ell}{\sum^L_{\ell=1}\fk(k^\ell, q)} = \iota \cdot C_1  \cdot V^\top\softmax(Kq / \gamma)=\iota \cdot C_1 \cdot \att_\sm(q, K, V),
\$
where the first equality follows from the definition of the softmax function and the second equality follows from the definition of the softmax attention in \eqref{eq::dop_prod_attn}.
By setting $C = \iota \cdot C_1$, we complete the proof of Proposition \ref{prop:kernel_attn_story1}.
\end{proof}

\subsubsection{Proof of Lemma \ref{lem::calculation_of_integral}}\label{sec::pf_lem_calculation_of_integral}
\begin{proof}
	Let $a, b$ be two vectors in the $(d - 1)$-dimensional unit sphere $\SSS^{d - 1}$. We first define the following vector,
	\#\label{eq::pf_lem_calc_int_eq1}
	c = (a^\top b)\cdot b-\bigl(a - (a^\top b)\cdot b\bigr) \in\SSS^{d - 1}.
	\#
	By direct calculation, we have the following property of $c$ defined in \eqref{eq::pf_lem_calc_int_eq1},
	\#\label{eq::pf_lem_calc_int_eq2}
	c^\top b &= (a^\top b)\cdot\|b\|^2_2 - a^\top b +  (a^\top b)\cdot\|b\|^2_2 = a^\top b.
	\#
	By \eqref{eq::pf_lem_calc_int_eq1} and \eqref{eq::pf_lem_calc_int_eq2}, we have that
	\#\label{eq::pf_lem_calc_int_eq3}
	a + c = 2(a^\top b) \cdot b = 2(c^\top b) \cdot b = (a^\top b) \cdot b + (c^\top b) \cdot b.
	\#
	We now calculate the desired integration. Note that
	\#\label{eq::pf_lem_calc_int_eq4}
	\int_{\SSS^{d - 1}} a\cdot\exp(a^\top b)\ud a = b \cdot \int_{\SSS^{d - 1}} (a^\top b) \exp(a^\top b) \ud a + \int_{\SSS^{d - 1}} \bigl(a - (a^\top b)\cdot b\bigr)\cdot\exp(a^\top b) \ud a.
	\#
	For the second term on the right-hand side of \eqref{eq::pf_lem_calc_int_eq4}, it follows from \eqref{eq::pf_lem_calc_int_eq1} and  \eqref{eq::pf_lem_calc_int_eq2} and \eqref{eq::pf_lem_calc_int_eq3} that
	\#\label{eq::pf_lem_calc_int_eq5}
	\int_{\SSS^{d - 1}} \bigl(a - (a^\top b)\cdot b\bigr)\cdot\exp(a^\top b) \ud a &= -\int_{\SSS^{d - 1}} \bigl(c - (c^\top b)\cdot b\bigr)\cdot\exp(c^\top b) \ud a\notag\\
	&=-\int_{\SSS^{d - 1}} \bigl(c - (c^\top b)\cdot b\bigr)\cdot\exp(c^\top b) \ud c,
	\#
	where the second equality follows from the fact that 
	\$
	\ud c = 2\|b\|^2_2\ud a - \ud a = \ud a.
	\$
	By replacing $c$ by $a$ on the right-hand side of \eqref{eq::pf_lem_calc_int_eq5}, we have
	\#\label{eq::pf_lem_calc_int_eq6}
	\int_{\SSS^{d - 1}} \bigl(a - (a^\top b)\cdot b\bigr)\cdot\exp(a^\top b) \ud a = -\int_{\SSS^{d - 1}} \bigl(a - (a^\top b)\cdot b\bigr)\cdot\exp(a^\top b) \ud a= 0
	\#
	Finally, by plugging \eqref{eq::pf_lem_calc_int_eq6} into \eqref{eq::pf_lem_calc_int_eq4}, we obtain that
	\$
	\int_{\SSS^{d - 1}} a\cdot\exp(a^\top b)\ud a = b \cdot \int_{\SSS^{d - 1}} (a^\top b) \exp(a^\top b) \ud a.
	\$
	Thus, by setting
	\$
	C_1 = \int_{\SSS^{d - 1}} (a^\top b) \exp(a^\top b) \ud a, \quad \forall b \in\SSS^{d - 1},
	\$
	we complete the proof of Lemma \ref{lem::calculation_of_integral}. Note that here $C_1$ is an absolute constant that does not depend on $b$ due to the symmetry on the unit sphere.
\end{proof}

\section{Generalization Error Analysis}\label{appendix:gen}
In this section, we analyze the generalization error of the complete setup of the transformer architecture, which involves multiple layers, skip connections, and multihead attentions. We collect the notations used throughout this section as follows.
\vskip4pt
\noindent{\bf Notations.} For two positive reals $r$ and $s$ such that $1/r + 1/s = 1$, we call $(r, s)$ a conjugate pair. We denote by $\|\cdot\|_r$ the vector $\ell_r$-norm when it operates on a vector. Let $M = (m_1, \dots, m_{d_2}) \in \RR^{d_1\times d_2}$, where $m_i \in \RR^{d_1}$ with $i \in [d_2]$. We define the matrix $(r, s)$-norm as $\|M\|_{r, s} = (\sum_{i = 1}^{d_2}\|m_i\|^s_r)^{1/s}$. We define the $(r, s)$-operator norm as $\|M\|_{r\to s} = \sup_{u \in \RR^{d_2}}\|Mu\|_{s}/\|u\|_r$. We write $\|\cdot\|_{r} = \|\cdot\|_{r \to r}$ when the $(r, r)$-operator norm operates on a matrix.

\subsection{Complete Setup of Transformer Architecture}
In what follows, we specify the complete setup of a $T$-layer transformer parameterized by $\theta = (\overbar \theta, \theta^{(0)}, \dots, \theta^{(T-1)})$, where the $t$-th layer ($t = 0, \dots, T-1$) is parameterized by $\theta^{(t)} \in \Theta^{(t)}$ and the aggregation layer is parameterized by $\overbar \theta \in \overbar\Theta$. Here $\Theta^{(t)}$ and $\overbar\Theta$ are the parameter spaces for the $t$-th layer and the aggregation layer, respectively. We define a two-layer feedforward neural network (FFN) with a skip connection (and no bias term) as follows,
\#\label{eq:ffn}
\ffn(X; A) = \relu(X A^{\tx})A^{\sigma} + X \in \RR^{L \times d}, 
\#
which is parameterized by $A = (A^\tx, A^\sigma)$. Here $X \in \RR^{L \times d}$,$A^{\tx} \in \RR^{d \times d_\sigma}, A^{\sigma} \in \RR^{d_\sigma \times d}$, and $\relu(\cdot)$ is the rectified linear unit (ReLU) that operates elementwise. Corresponding to \eqref{eq::dop_prod_attn}, we define the sequence-to-sequence counterpart of the softmax attention as follows, 
\#\label{eq:seq2seq_sm}
 \att_{\sm}(Q, K, V) = \Bigl(V^\top\Norm_\sm\bigl(\fk_\rbf(K, q^\ell) \bigr)\Bigr)_{\ell \in [L]}^\top \in \RR^{L \times d}.
\#
Here $Q = (q^\ell)_{\ell \in [L]}^\top \in \RR^{L \times d_\tp}$, $K = (k^\ell)_{\ell \in [L]}^\top \in \RR^{L \times d_\tp}$, $V \in \RR^{L \times d}$, and $\fk_\rbf(K, q^\ell) = (\fk_\rbf(q^\ell, k^{\ell^\prime}))_{\ell' \in [L]}^\top \in \RR^L$ is specified in Assumption \ref{assumption:ker}. Recall that $h$ is the head number of the multihead attention defined in \eqref{eq:mha} and $d = d_\tp \cdot h$. With a slight abuse of notations, we define the sequence-to-sequence counterpart of the multihead attention (MHA) as follows,
\#\label{eq:mha-seq2seq}
\mha(X; W) = \sum_{i= 1}^h \head_i = \sum_{i= 1}^h\att_\sm(Q_i, K_i, V_i) \in \RR^{L \times d},
\#
which is parameterized by $W = \{(W^{\tq}_i, W^{\tk}_i, W^{\tv}_i)\}_{i \in [h]}$. Here $Q_i = XW^\tq_i \in \RR^{L \times d_\tp}$, $K_i = XW^\tk_i \in \RR^{L \times d_\tp}$, and $V_i = XW^\tv_i \in \RR^{L \times d}$ for the attention head $i \in [h]$, where $W^\tq_i, W^\tk_i \in \RR^{d \times d_\tp}$ and $W^\tv_i \in \RR^{d \times d}$.

With $X_\star^{(0)} = X$, let $X_\star^{(t)} \in \RR^{L \times d}$ be the intermediate input of the $t$-th layer ($t = 0, \dots, T-1$) of the transformer architecture, which is defined as follows,
\#\label{eq:trans-arch}
X^{(t)} & = \ffn(X_\star^{(t)}; A^{(t)}), \qquad & A^{(t)}& = (A^{\tx, (t)}, A^{\sigma, (t)}), \notag\\
X_\star^{(t+1)} &= \mha(X^{(t)}; W^{(t)}) + X^{(t)}, \qquad & W^{(t)} & = \bigl\{(W^{\tq, (t)}_i, W^{\tk, (t)}_i, W^{\tv, (t)}_i)\bigr\}_{i \in [h]},
\#
Here $\theta^{(t)} = (A^{(t)}, W^{(t)})$ is the learnable parameter. We compute the output as follows,
\#\label{eq:agg-trans}
\hat y = \overbar{\agg}_{\overbar \theta}(X^{(T)}_\star) \in \RR^{d_\ty},
\#
where the aggregation layer $\overbar{\agg}_{\overbar \theta}: \RR^{L\times d} \to \mathfrak{Y}$ is parameterized by $\overbar \theta$. Here $\mathfrak{Y}$ is defined in \eqref{eq:data-supp}.  Note that $\overbar\agg_{\overbar\theta}$ combines the aggregation layer $\agg_{\theta_0}$ defined in \eqref{eq:trans-class-simple} and the input mask $\mask$. For example, in the complete setup of ViT \citep{dosovitskiy2020image}, the aggregation layer is the function composition $\overbar \agg_{\overbar \theta}(X^{(T)}_\star) = \agg_{\theta_0}(\mha(q_W(\mask), X^{(T)}_\star; W))$ with $\overbar \theta = (\theta_0, W)$, where $\mask$ corresponds to the class encoding. Here the multihead attention $\mha(q_W(\mask), X^{(T)}_\star; W)$ follows the definition in \eqref{eq:mha}.

\vskip4pt
\noindent{\bf Empirical Image Class.} In what follows, we formalize the function class of the transformer architecture and the empirical image class for each layer. We define the base function class as 
\$
\cF_\mha^{L, (0)} = \bigl\{X^{(0)}_\star(X)\bigr\},
\$
which is the function class that only contains the identity mapping. Here we use $X_\star^{(0)}(X) = X$ to denote the identity mapping since $X^{(0)}_\star = X$. In the following, we use $X_\star^{(t)}(X)$ and $X^{(t-1)}(X)$ to denote the functions that map $X$ to $X_\star^{(t)}$ and $X^{(t)}$ for $t = 1, \dots, T-1$, respectively. We define the intermediate function classes recursively as follows,
\$
\cF^{L, (t)}_\ffn & = \Bigl\{\ffn\bigl(X_\star^{(t)}(X); A^{(t)}\bigr): A^{(t)} \in {\mathfrak{A}}^{(t)}, X_\star^{(t)}(X) \in \cF_\mha^{L, (t)}\Bigr\},\\
\cF_\mha^{L, (t+1)} & = \Bigl\{\mha\bigl(X^{(t)}(X); W^{(t)}\bigr) + X^{(t)}: W^{(t)} \in {\mathfrak{W}}^{(t)}, X^{(t)}(X) \in \cF_\ffn^{L, (t)}\Bigr\},
\$
where $0 \leq t \leq T-1$. Here $\Theta^{(t)} = \mathfrak{A}^{(t)} \times \mathfrak{W}^{(t)}$ is the parameter space of the $t$-th layer of the transformer architecture. Correspondingly, we define the function class of the $T$-layer transformer as follows,
\#\label{eq:trans-class}
\cF^{L} = \Bigl\{ \overbar\agg_{\overbar\theta}\bigl(X_\star^{(T)}(X)\bigr): \overbar \theta \in \overbar \Theta,  X_\star^{(T)}(X) \in \cF^{L, (T)}_\mha\Bigr\}.
\#
Correspondingly, we define the empirical image classes as follows,
\#\label{eq:trans-image}
\mathfrak{I}_{\cD_n}(\cF^{L, (t)}_\ffn) & = \Bigl\{\bigl(f(X_i)^\top\bigr)_{i \in [n]} \in \RR^{d \times nL} : f \in \cF^{L, (t)}_\ffn\Bigr\},\notag\\
\mathfrak{I}_{\cD_n}(\cF_\mha^{L, (t+1)}) & = \Bigl\{\bigl(f(X_i)^\top\bigr)_{i \in [n]} \in \RR^{d\times nL} : f \in \cF_\mha^{L, (t+1)}\Bigr\},\notag\\
\mathfrak{I}_{\cD_n}(\cF^L) & = \Bigl\{\bigl(f(X_i)\bigr)_{i \in [n]} \in \RR^{ d_\ty\times n} : f \in \cF^L\Bigr\},
\#
where $0 \leq t \leq T-1$.

\subsection{Generalization Error Analysis for Complete Setup}
In what follows, we present a general version of Theorem \ref{thm:gen-simple}, which allows for the complete setup of the transformer architecture. By specializing it to the single-layer transformer equipped with singlehead attention mechanism and no skip connection, we obtain Theorem \ref{thm:gen-simple}.

In parallel to \eqref{eq:ker-rbf}, we make the following assumption on the Gaussian RBF kernel $\fk_\rbf(q, k)$, which induces the multihead attention $\mha(X; W)$ defined in \eqref{eq:seq2seq_sm} and \eqref{eq:mha-seq2seq}.
\begin{assumption}[Gaussian RBF Kernel]\label{assumption:ker}
Let $s>0$. We assume that the multihead attention $\mha(X; W)$ adopts the Gaussian RBF kernel $\fk_\rbf(q, k) = \exp(-\|q - k\|_2^2/2\sigma^2)$ with $\sigma = (2d_\tp)^{1/s}$.
\end{assumption}
Note that the kernel function $\fk_\rbf(q, k)$ in Assumption \ref{assumption:ker} is a general version of the Gaussian RBF kernel defined in \eqref{eq:ker-rbf}, which corresponds to the special case where $s = 2$. Recall that the output range $\mathfrak{Y}$ is defined in \eqref{eq:data-supp}. We make the following assumption on the aggregation layer $\overbar\agg_{\overbar\theta}$ defined in \eqref{eq:agg-trans}.
\begin{assumption}[Aggregation Layer]\label{assumption:trans}
We assume that the aggregation layer $\overbar\agg_{\overbar\theta}: \RR^{L \times d} \to \RR^{d_\ty}$ has the output range $\mathfrak{Y}$. Let $\overbar \agg_{\overbar \theta, j}: \RR^{L \times d} \to \RR$ be the $j$-th entry ($j \in [d_\ty]$) of the aggregation function $\overbar \agg_{\overbar \theta}$. We assume that for any $\overbar \theta \in \overbar \Theta$, $X_\star, \tilde X_\star \in \RR^{L \times d}$, and $j \in [d_\ty]$, it holds that
\$
\bigl|\overbar\agg_{\overbar\theta, j}(X_\star) - \overbar\agg_{\overbar\theta, j}(\tilde X_\star)\bigr| & \leq \|X_\star^\top - \tilde X_\star^{\top}\|_{r, \infty}.
\$
\end{assumption}
Recall that $\|\cdot\|_r$ denotes the $(r, r)$-operator norm when it operates on a matrix and $\|\cdot\|_{r, s}$ is the matrix $(r, s)$-norm. In parallel to Assumption \ref{assumption:constraint-simple}, we make the following assumption on the parameter space for each layer of the transformer architecture.

\begin{assumption}[Parameter Space]\label{assumption:constraint}
Let $(r, s)$ be a conjugate pair. For $t = 0, \dots, T-1$, we assume that the parameter space $\Theta^{(t)} = {\mathfrak{A}}^{(t)} \times {\mathfrak{W}}^{(t)}$ of the $t$-th layer of the transformer architecture satisfy
\#
{\mathfrak{A}}^{(t)} & = \Bigl\{(A^{\tx, (t)}, A^{\sigma, (t)}) \in \RR^{d \times d_\sigma} \times \RR^{d_\sigma \times d}:\label{eq:set-at}\\
& \qquad  \bigl\|(A^{\tx, (t)})^\top\bigr\|_{r} \leq {\alpha}^{\tx, (t)}, \bigl\|(A^{\tx, (t)})^\top\bigr\|_{r, s} \leq {R}^{\tx, (t)}, \bigl\|(A^{\sigma, (t)})^\top\bigr\|_{r} \leq {\alpha}^{\sigma, (t)}, \bigl\|(A^{\sigma, (t)})^\top\bigr\|_{r, s} \leq R^{\sigma, (t)} \Bigr\},\notag\\
\mathfrak{W}^{(t)} & = \Bigl\{\bigl\{(W^{\tq, (t)}_i, W^{\tk, (t)}_i, W^{\tv, (t)}_i)\bigr\}_{i \in [h]}: (W^{\tq, (t)}_i, W^{\tk, (t)}_i, W^{\tv, (t)}_i) \in \RR^{d \times d_\tp} \times\RR^{d \times d_\tp} \times\RR^{d \times d},\notag\\
& \qquad \bigl\|(W^{\tq, (t)}_i)^\top\bigr\|_{r} \leq \omega^{\tq, (t)}_i, \bigl\|(W^{\tk, (t)}_i)^\top\bigr\|_{r} \leq \omega^{\tk, (t)}_i, \bigl\|(W^{\tv, (t)}_i)^\top\bigr\|_{r} \leq \omega^{\tv, (t)}_i,\label{eq:set-wt}\\
& \qquad \bigl\|(W^{\tq, (t)}_i)^\top\bigr\|_{r, s} \leq R^{\tq, (t)}_i, \bigl\|(W^{\tk, (t)}_i)^\top\bigr\|_{r, s} \leq R^{\tk, (t)}_i, \bigl\|(W^{\tv, (t)}_i)^\top\bigr\|_{r, s} \leq R^{\tv, (t)}_i\Bigr\}.\notag
\#
Also, we assume that the parameter space $\overbar \Theta$ of the aggregation layer takes the form of $\overbar\Theta = \{\overbar\theta \in \RR^{d_\agg}: \|\overbar\theta\|_r \leq 1\}$. Here ${\alpha}^{\tx, (t)}, {R}^{\tx, (t)}, {\alpha}^{\sigma, (t)}, R^{\sigma, (t)}, \omega^{\tq, (t)}_i, \omega^{\tk, (t)}_i, \omega^{\tv, (t)}_i, R^{\tq, (t)}_i, R^{\tk, (t)}_i, R^{\tv, (t)}_i > 0$ with $i \in [h]$ and $t = 0, \dots, T-1$.
\end{assumption}
To ease the presentation, we define the following quantities that combine the parameter bounds across the $h$ heads within the $t$-th layer of the transformer architecture,
\#\label{eq:sum-head}
{\omega}^{\tv, (t)} & = \sum_{i = 1}^h\omega_i^{\tv, (t)}, \qquad \overbar{\omega}^{\tq\tk, (t)} = \max_{i \in [h]} \{\omega_i^{\tq, (t)} + \omega_i^{\tk, (t)}\},\notag\\
R^{\tv, (t)} & = \sum_{i = 1}^h R_i^{\tv, (t)}, \qquad {R}^{\tq\tk, (t)} = \sum_{i = 1}^h(R_i^{\tq, (t)} + R_i^{\tk, (t)}),
\#
where $t = 0, \dots, T-1$. Let
\#\label{eq:mag-xt}
\tilde{\alpha}^{(t)} = 1 + \alpha^{\tx, (t)} \alpha^{\sigma, (t)}, \qquad \tilde{\omega}^{\tv, (t)} = 1 + \omega^{\tv, (t)}, \qquad \gamma^{(t)} = \max\{\tilde \alpha^{(t)}, \tilde \omega^{\tv, (t)}\}. 
\#
Also, let
\#\label{eq:ratio}
\kappa^{(t)} = \max\biggl\{\frac{\alpha^{\tx, (t)} R^{\sigma, (t)} + \alpha^{\sigma, (t)} R^{\tx, (t)}}{\tilde{\alpha}^{(t)}}, \frac{R^{\tv, (t)}}{\tilde{\omega}^{\tv, (t)}}, \frac{{R}^{\tq\tk, (t)}}{ \overbar{\omega}^{\tq\tk, (t)}\omega^{\tv, (t)}}\biggr\}, \qquad \zeta^{(t)} = \frac{(\overbar{\omega}^{\tq\tk, (t)})^2R^{\tv, (t)}}{\tilde{\omega}^{\tv, (t)}}.
\#
Recall that the generalization error $\cE_{\rm gen}$ is defined in \eqref{eq::risk_decomp}. The following theorem characterizes the generalization error of the transformer.
\begin{theorem}[Generalization Error of Transformer]\label{thm:gen}
Let $D = \max\{d, d_\sigma, d_\tp, d_\ty\}$. Suppose that Assumptions \ref{assumption:ker}-\ref{assumption:constraint} and Assumption \ref{asu::data} hold. Then, for any $\delta > 0$, it holds with probability at least $1 - \delta$ that,
\$
\cE_{\rm gen} = O\biggl( \frac{D^2}{\sqrt{n}} \cdot \Bigl[T\sqrt{\log(1+ \gamma)} + \sqrt{T} \sqrt{\log(1 + \zeta R)} + \sqrt{\log( 1+ \kappa/\zeta)}\Bigr] \cdot \sqrt{hT}+ \sqrt{\frac{\log (1/\delta)}{n}}\biggr),
\$
where $R$ is defined in \eqref{eq:data-supp}. Here
\#\label{eq:const-simplify}
\gamma = \max_{0 \leq t \leq T-1}\gamma^{(t)}, \qquad \kappa = \max_{0 \leq t \leq T-1} \kappa^{(t)}, \qquad \zeta = \max_{0 \leq t \leq T-1} \zeta^{(t)},
\#
where $\gamma^{(t)}$, $\kappa^{(t)}$, and $\zeta^{(t)}$ are defined in \eqref{eq:mag-xt}-\eqref{eq:ratio}.
\end{theorem}
\begin{proof}
See \S\ref{appendix:full} for a detailed proof.
\end{proof}

\vskip4pt
\noindent{\bf Highlight.} In comparison with \citet{edelman2021inductive}, we exploit the invariance/equivariance property of the transformer architecture in a fine grained manner. The key observation is that, due to the invariance/equivariance property, the dimensions of the learnable parameters $W^{(t)}$ and $A^{(t)}$ are independent of the sequence length $L$. By such an observation, we characterize the covering number of the function class with the covering numbers of the parameter spaces and propagate them through the $T$ layers. As a consequence, the covering number of the function class is independent of $L$. In contrast, the generalization error in \citet{edelman2021inductive} has a logarithmic dependency on $L$.

%

\vskip4pt
\noindent{\bf Interpretation.} We interpret the generalization error in Theorem \ref{thm:gen} as follows. On the one hand, the $O(1/\sqrt{n}$ dependencies in $O(\sqrt{\log (1/\delta)/n})$ and $O(D^2/\sqrt{n})$ over the sample size $n$ are standard in the literature. On the other hand, the $O(D^2h^{1/2}T^{3/2})$ scaling implies that the transformer architecture requires more training data points to generalize as the dimension $D$ of the parameter space, the number $T$ of layers, and the head number $h$ grow. Also, the generalization error in Theorem \ref{thm:gen} only scales logarithmically in $\gamma$, $\kappa$, and $\zeta$, which implies that the generalization error remains polynomial order as long as $\gamma$, $\kappa$, and $\zeta$ do not scale doubly exponentially with $D$, $h$, or $T$.

\vskip4pt
\noindent{\bf Implication.} Theorem \ref{thm:gen} demonstrates that $\gamma$, $\kappa$, $\zeta$ and $R$ play a crucial role in the generalization error of the transformer. Specifically, we observe that (i) skip connections allow all layers to resemble the identity mapping \citep{bartlett2018representing,bartlett2018gradient, hardt2016identity}, which helps reducing $\gamma$, $\kappa$ and $\zeta$, and (ii) layer normalizations helps controlling the scaling of the intermediate inputs $\{X^{(t)}_\star\}_{0 \leq t \leq T-1}$, which reduces the covering number of the function class.

\vskip4pt
\noindent{\bf Simplification of Theorem \ref{thm:gen} to Theorem \ref{thm:gen-simple}.} Theorem \ref{thm:gen} characterizes the generalization error of the complete setup of the transformer architecture, which includes Theorem \ref{thm:gen-simple} as a special case. In what follows, we specialize Theorem \ref{thm:gen} to obtain Theorem \ref{thm:gen-simple}.
\begin{itemize}
    \item Single-layer transformer with single-head attention: We set $h=1$ and $T = 1$, which implies that
    \begin{itemize}
    \item \eqref{eq:sum-head} becomes
    \$
{\omega}^{\tv, (0)} & = \omega_1^{\tv, (0)}, & \overbar{\omega}^{\tq\tk, (0)} & = \omega_1^{\tq, (0)} + \omega_1^{\tk, (0)},\notag\\
R^{\tv, (0)} & = R_1^{\tv, (0)},  &{R}^{\tq\tk, (0)} &= R_1^{\tq, (0)} + R_1^{\tk, (0)},
\$
\item \eqref{eq:const-simplify} becomes
\$
\gamma = \gamma^{(0)}, \qquad \kappa = \kappa^{(0)}, \qquad \zeta = \zeta^{(0)},
\$
which are defined in \eqref{eq:ratio} but will be redefined in our subsequent simplification, and
\item the generalization error in Theorem \ref{thm:gen} becomes
\$
\cE_{\rm gen} \leq O\biggl( \frac{D^2}{\sqrt{n}} \cdot \Bigl[\sqrt{\log(1+ \gamma)} +  \sqrt{\log(1 + \zeta R)} + \sqrt{\log( 1+ \kappa/\zeta)}\Bigr]+ \sqrt{\frac{\log (1/\delta)}{n}}\biggr).
\$
    \end{itemize}
    \item Skip connections: Since we have no skip connections, we set 
    \$
    \tilde{\alpha}^{(0)} = {\alpha}^{\tx, (0)}{\alpha}^{\sigma, (0)}, \qquad \tilde\omega^{\tv, (0)} = \omega^{\tv, (0)},
    \$
    which appear in the definitions of $\gamma^{(0)}$, $\kappa^{(0)}$, and $\zeta^{(0)}$ in \eqref{eq:mag-xt} and \eqref{eq:ratio}.
    \item Feedforward neural network: Since there is no linear transformation for the second layer of $\nn(X; A) = \relu(XA)$. Specifically, we
    \begin{itemize}
    \item set $\alpha^{\sigma, (0)} = 1$ and $R^{\sigma, (0)} = 0$,
    \item set $d_\sigma = d$ for the intermediate output, and
    \item set $\alpha^\nn = \alpha^{\tx, (0)}$ and $R^\nn = R^{\tx, (0)}$ in Assumption \ref{assumption:constraint-simple}.
   \end{itemize}
   As a consequence, we obtain
       \$
    \gamma = \max\{ \alpha^\nn, \omega^{\tv, (0)}\}, \qquad \kappa = \max\biggl\{\frac{R^\nn}{\alpha^\nn}, \frac{R^{\tv, (0)}}{{\omega}^{\tv, (0)}}, \frac{{R}^{\tq\tk, (0)}}{ \overbar{\omega}^{\tq\tk, (0)}\omega^{\tv, (0)}}\biggr\}.
    \$
    \item Softmax attention: Since we have only one head, we set
    \begin{itemize}
    \item $(\omega^\tq, \omega^\tk, \omega^\tv) = (\omega^{\tq, (0)}, \omega^{\tk, (0)}, \omega^{\tv, (0)})$ and
    \item $(R^\tq, R^\tk, R^\tv)= (R^{\tq, (0)}, R^{\tk, (0)}, R^{\tv, (0)})$.
    \end{itemize}
       As a consequence, we obtain $\overbar\omega^{\tq\tk, (0)} = \omega^\tq + \omega^\tk$, $R^{\tq\tk, (0)} = R^\tq + R^\tk$, and
       \$
    \gamma = \max\{ \alpha^\nn, \omega^\tv\}, \qquad \kappa = \max\biggl\{\frac{R^\nn}{\alpha^\nn}, \frac{R^\tv}{{\omega^\tv}}, \frac{{R}^\tk + R^\tq}{(\omega^\tq + \omega^\tk)\cdot \omega^\tv}\biggr\}, \qquad \zeta = \frac{(\omega^\tq + \omega^\tk)^2 \cdot R^\tv}{\omega^\tv}.
    \$
        \item Spectral norm and Frobenius norm: We use the conjugate pair $(r, s) = (2, 2)$, which implies that
    \begin{itemize}
    \item $\|\cdot\|_r = \|\cdot\|_2$ is the spectral norm of matrices and $\|\cdot\|_{r, s} = \|\cdot\|_{2, 2} = \|\cdot\|_{\rm F}$ is the Frobenius norm, which correspond to Assumption \ref{assumption:constraint-simple}, and 
    \item the Gaussian RBF kernel $\fk_\rbf(q, k)$ in Assumption \ref{assumption:ker} is normalized by $\sigma = (2d_\tp)^{1/s} = (2d_\tp)^{1/2}$, which corresponds to \eqref{eq:ker-rbf}.
    \end{itemize}
\end{itemize}
Therefore, we obtain Theorem \ref{thm:gen-simple}.

\vskip4pt
\noindent{\bf Proof Sketch.} We organize the proof of Theorem \ref{thm:gen} as follows.
\begin{itemize}
    \item [(\S\ref{appendix:rad})] We review how to analyze the generalization error through the Rademacher complexity, which requires a covering number of the function class.
    \item[(\S\ref{appendix:full})] We provide a covering number of the function class and sketch the proof of Theorem \ref{thm:gen}.
    \item[(\S\ref{appendix:cover})] We characterize the covering number of the function class by (i) analyzing the covering number of each MHA and FFN and (ii) analyzing how the covering numbers propagate through the $T$ layers of the transformer architecture.
    \item[(\S\ref{appendix:gen-aux})] We leave the detailed proofs of the intermediate lemmas to \S\ref{appendix:gen-aux}.
\end{itemize}

\subsection{Preliminary of Generalization}\label{appendix:rad}
In this section, we introduce the building blocks for analyzing the generalization error of the transformer architecture.
\subsubsection{Rademacher Complexity}
Suppose that the dataset $\cD_n = \{(X_i, y_i)\}_{i = 1}^n$ is drawn independently and identically from the data distribution $\cD$. Recall that $\cF^L$ is defined in \eqref{eq:trans-class}. Let $\cL((X, y), f)$ be a fixed learning objective. We define the function class $\cL \circ \cF^L$ as follows,
\#\label{eq:l-f}
\cL \circ \cF^L = \Bigl\{\cL\bigl((X, y), f\bigr): f \in \cF^L\Bigr\},
\#
which contains the function compositions of the learning objective $\cL$ and the transformer function $f \in \cF^L$. We define the empirical Rademacher complexity of the function class $\cL\circ \cF^L$ as follows,
\#\label{eq:rad}
{\cR}_{\cD_n}(\cL\circ\cF^L) =   \EE\biggl[\sup_{f \in \cF^L} \frac{1}{n}\sum_{i = 1}^n \epsilon_i \cdot \cL\bigl((X_i, y_i), f\bigr)\biggr],
\#
where the expectation is taken over the independent Rademacher sequence $\{\epsilon_i\}_{i \in [n]}$. The following lemma characterizes the generalization error for learning with the function class $\cF^L$.

\begin{lemma}[Generalization Error via Rademacher Complexity, \citet{mohri2018foundations}]\label{thm:gen-base}
Suppose that the function class $\cL\circ\cF^L$ defined in \eqref{eq:l-f} has the output range $[0, 1]$. For any $\delta > 0$, with probability at least $1 - \delta$ over the independent and identical draw of the dataset $\cD_n = \{(X_i, y_i)\}_{i \in [n]}$ from the data distribution $\cD$, it holds for all $f \in \cF^L$ that,
\$
\biggl|\EE\Bigl[\cL\bigl((X, y), f\bigr)\Bigr] - \hat{\EE}\Bigl[\cL\bigl((X, y), f\bigr)\Bigr]\biggr| \leq 2 {\cR}_{\cD_n}(\cL \circ \cF^L) + 3\sqrt{\frac{\log (2/\delta)}{2n}},
\$
where $\hat{\EE}[\,\cdot\,]$ is the empirical expectation taken over the dataset $\cD_n$ and the empirical Rademacher complexity ${\cR}_{\cD_n}(\cL\circ\cF^L)$ is defined in \eqref{eq:rad}.
\end{lemma}
We define the product function class as follows,
\$
\prod_{j = 1}^{d_\ty}\cF^L_j = \Bigl\{\bigl(f_j(X)\bigr)_{j \in [d_\ty]}^\top: f_j \in \cF^L_j, j \in [d_\ty]\Bigr\},
\$
where $\cF^L_j$ is the function class of the $j$-th entry ($j \in [d_\ty]$) of $f = (f_1, \dots, f_{d_\ty})^\top \in \cF^L$. The following lemma characterizes the empirical Rademacher complexity $\cR_{\cD_n}(\cL\circ\cF^L)$.

\begin{lemma}[Vector Contraction Inequality, \citet{maurer2016vector}]\label{lemma:contraction}
Let $\cF$ be the function class of $f: \RR^{L \times d} \to \mathfrak{Y} \subseteq \RR^{d_\ty}$ and let $\cL_i: \mathfrak{Y} \to \RR$ be a $1$-Lipschitz function with respect to the vector $\ell_2$-norm, where $i \in [n]$. Then, we have
\$
\EE\biggl[\sup_{f\in\cF}\sum_{i = 1}^n \epsilon_i \cdot \cL_i\bigl(f(X_i)\bigr)\biggr] \leq \sqrt{2}\cdot \EE\biggl[\sup_{f\in\cF}\sum_{i = 1}^n\sum_{j = 1}^{d_\ty} \epsilon_{ij} \cdot f_j(X_i)\biggr],
\$
where $\{\epsilon_{i}\}_{i \in [n]}$ is an independent Rademacher sequence, $\{\epsilon_{ij}\}_{i \in [n], j \in [d]}$ is an independent Rademacher sequence that is doubly indexed, $f_j(X_i)$ is the $j$-th entry of $f(X_i)$, and the expectations are taken over the independent Rademacher sequences.
\end{lemma}

Lemma \ref{lemma:contraction} generalizes the Ledoux-Talagrand contraction inequality \citep{ledoux1991probability} to the multivariate setting. Note that $\cF^L \subseteq \prod_{j = 1}^{d_\ty}\cF^L_j$ since all entries of $f \in \cF^L$ share the same parameters. By setting $\cL_i(f(X_i)) = \cL((X_i, y_i), f)$ in Lemma \ref{lemma:contraction}, we have
\#\label{eq:rad-decomp}
\cR_{\cD_n}(\cL \circ \cF^L) & \leq \sqrt{2}\cdot \EE\biggl[\sup_{f \in \cF^L}\frac{1}{n}\sum_{i = 1}^n\sum_{j = 1}^{d_\ty} \epsilon_{ij} \cdot f_j(X_i)\biggr]\notag\\
 &\leq \sqrt{2}\cdot \EE\biggl[\sup_{\{f_j \in \cF^L_j\}_{j \in [d_\ty]}}\frac{1}{n}\sum_{i = 1}^n\sum_{j = 1}^{d_\ty} \epsilon_{ij} \cdot f_j(X_i)\biggr]\notag\\
& \leq \sqrt{2}\cdot\sum_{j = 1}^{d_\ty} \EE\biggl[\sup_{f_j\in\cF_j}\frac{1}{n}\sum_{i = 1}^n \epsilon_{ij} \cdot f_j(X_i)\biggr]\notag\\
&  = \sqrt{2}\cdot \sum_{j = 1}^{d_\ty}\cR_{\cD_n}(\cF^L_j),
\#
which implies that it remains to characterize the empirical Rademacher complexity $\cR_{\cD_n}(\cF^L_j)$, where $j \in [d_\ty]$.

\subsubsection{Rademacher Complexity via Covering Number}
In what follows, we connect the Rademacher complexity $\cR_{\cD_n}(\cF^L_j)$ to the covering number of the empirical image class $\mathfrak{I}_{\cD_n}(\cF^L_j)$, which is defined as
\#\label{eq:image-fj}
\mathfrak{I}_{\cD_n}(\cF^L_j) = \Bigl\{\bigl(f_j(X_i)\bigr)_{i \in [n]} \in \RR^{ 1 \times n} : f_j \in \cF^L_j\Bigr\}.
\#
We define the proper covering number as follows.
\begin{definition}[Proper Covering Number]
Let $N(\cS, \varepsilon, \|\cdot\|)$ be the least cardinality of any subset $\cT \subseteq \cS$ that covers the set $\cS$ at the resolution $\varepsilon$ with respect to the norm $\|\cdot\|$, that is,
\$
N\bigl(\cS, \varepsilon, \|\cdot\|\bigr) = \inf\Bigl\{{\rm card}(\cT): \sup_{S \in \cS}\inf_{T \in \cT}\|S - T\| \leq \varepsilon, \cT \subseteq \cS\Bigr\},
\$
where ${\rm card}(\cT)$ is the cardinality of the set $\cT$.
\end{definition}
To characterize $\cR_{\cD_n}(\cF_j)$, we use the following version of the Dudley entropy integral lemma for the matrix $(r, \infty)$-norm.

\begin{lemma}[Dudley Entropy Integral for $\|\cdot\|_{r, \infty}$-Covering, \citet{mohri2018foundations}]\label{lemma:dudley}
Let $\cF^L_j$ be a function class with the output range $[0, 1/2]$. Suppose that $0 \in \cF^L_j$. Then, we have
\$
{\cR}_{\cD_n}(\cF^L_j) \leq \inf_{\xi \in (0, 1/2)}\biggl(4\xi + \frac{12}{\sqrt{n}}\int_{\xi}^{1/2} \sqrt{\log N\bigl(\mathfrak{I}_{\cD_n}(\cF^L_j), \varepsilon, \|\cdot\|_{r, \infty}\bigr)} \ud \varepsilon\biggr).
\$
\end{lemma}
\begin{proof}
See \S\ref{appendix:dudley} for a detailed proof.
\end{proof}

By Lemmas \ref{thm:gen-base}-\ref{lemma:dudley}, we see that it remains to characterize the covering number of the empirical image class $\mathfrak{I}_{\cD_n}(\cF^L_j)$ with respect to the matrix $(r, \infty)$-norm.

\subsection{Proof of Theorem \ref{thm:gen}}\label{appendix:full}
Recall that the parameter spaces $\mathfrak{A}^{(t)}$ and $\mathfrak{W}^{(t)}$ are defined in \eqref{eq:set-at} and \eqref{eq:set-wt}, respectively. Also, recall that $\tilde \alpha^{(t)}$, $\tilde{\omega}^{\tv, (t)}$, $\overbar{\omega}^{\tq\tk, (t)}$, and ${R}^{\tq\tk, (t)}$ are defined in \eqref{eq:sum-head}-\eqref{eq:mag-xt}. For the dataset $\cD_n = \{(X_i, y_i)\}_{i \in [n]}$, we define
\#\label{eq:rt}
R^{(0)}= \max_{i \in [n]}\|X_i^\top\|_{r, \infty}, \qquad R^{(t)} = R^{(0)}\cdot \prod_{\tau = 0}^{t-1}\tilde{\omega}^{\tv, (\tau)}\tilde{\alpha}^{(\tau)},
\#
which characterize the scaling of the intermediate input $X_\star^{(t)}$ for the $t$-th layer of the transformer architecture, where $t = 0, \dots, T-1$. The following lemma characterizes the covering number of the empirical image class $\mathfrak{I}_{\cD_n}(\cF^L_j)$.

\begin{lemma}[Covering Number of Transformer Architecture]\label{lemma:cover_trans}
Let $D = \max\{d, d_\sigma, d_\tp, d_\ty\}$. Under Assumption \ref{assumption:constraint}, we have for any $j \in [d_\ty]$ that
\$
\log N\bigl(\mathfrak{I}_{\cD_n}(\cF_j^L), \varepsilon, \|\cdot\|_{r, \infty}\bigr) \leq  (4+h)D^2T\cdot\log \biggl(1 + \frac{2{R}^{(T)} R_\trans}{\varepsilon}\biggr).
\$
Here $\mathfrak{I}_{\cD_n}(\cF_j^L)$ is defined in \eqref{eq:image-fj}, $R^{(t)}$ is defined in \eqref{eq:rt}, and 
\#\label{eq:const-trans}
{R}_\trans = \sum_{t = 0}^{T-2} \biggl(\frac{{R}^{(t)}_\mha}{\tilde{\rho}^{(t)}} \cdot \prod_{\tau = t+1}^{T-1}\frac{\tilde{\rho}^{(\tau)}}{\tilde{\omega}^{\tv, (\tau)}}\biggr) + \sum_{t = 0}^{T-1} \biggl( \frac{\alpha^{\tx, (t)} R^{\sigma, (t)} + \alpha^{\sigma, (t)} R^{\tx, (t)}}{\tilde{\alpha}^{(t)}} \cdot \prod_{\tau = t+1}^{T-1}\frac{\tilde{\rho}^{(\tau)}}{\tilde{\omega}^{\tv, (\tau)}}\biggr),
\#
where $\alpha^{\tx, (t)}$, $\alpha^{\sigma, (t)}$, $R^{\tx, (t)}$, and $R^{\sigma, (t)}$ are defined in Assumption \ref{assumption:constraint}, and
\#\label{eq:rho-rmha}
\tilde{\rho}^{(t)} = \tilde{\omega}^{\tv, (t)}  + (\overbar{\omega}^{\tq\tk, (t)})^2\omega^{\tv, (t)} \cdot (R^{(t)})^2,\qquad R^{(t)}_\mha = R^{\tv, (t)} + \overbar{\omega}^{\tq\tk, (t)}{R}^{\tq\tk, (t)}\cdot (R^{(t)})^2.
\#
\end{lemma}
\begin{proof}
See \S\ref{appendix:cover_trans} for a detailed proof.
\end{proof}

\begin{proof}[Proof of Theorem \ref{thm:gen}]
For any $a> 0$, we have
\#\label{eq:int-dudley}
\int_{\xi}^{1/2}\sqrt{\log(1 + a/\varepsilon)} \ud \varepsilon & \leq \int_{\xi}^{1/2}\bigl[\sqrt{\log(1+a)} + \sqrt{\log(1 + 1/\varepsilon)}\bigr] \ud \varepsilon\notag\\
& \leq (1/2 - \xi) \cdot \sqrt{\log(1+a)} + \int_{\xi}^{1/2}1/\sqrt{\varepsilon} \ud \varepsilon\notag\\
& = (1/2 - \xi) \cdot \sqrt{\log(1+a)} + \sqrt{2} - 2\sqrt{\xi},
\#
where the first inequality follows from the fact that $\sqrt{a+b} \leq \sqrt{a} + \sqrt{b}$ for any $a, b \geq 0$ and the second inequality follows from the fact that $\log(1+a) \leq a$ for any $a > -1$. By Lemma \ref{lemma:dudley}, we have for any $j \in [d_\ty]$ that
\#\label{eq:rad-bound1}
&{\cR}_{\cD_n}(\cF^L_j)\\
&\quad \leq \inf_{\xi \in (0, 1/2)}\biggl(4\xi + \frac{12}{\sqrt{n}}\int_{\xi}^{1/2} \sqrt{\log N\bigl(\mathfrak{I}_{\cD_n}(\cF_j^L), \varepsilon, \|\cdot\|_{r, \infty}\bigr)} \ud \varepsilon\biggr)\notag\\
&\quad \leq \inf_{\xi \in (0, 1/2)}\biggl(4\xi + \frac{12D\sqrt{(4+h)T}}{\sqrt{n}}\int_{\xi}^{1/2} \sqrt{\log \biggl(1 + \frac{2{R}^{(T)} R_\trans}{\varepsilon}\biggr)} \ud \varepsilon\biggr)\notag\\
& \quad \leq 12D\sqrt{(4+h)T} \cdot \frac{2\sqrt{2} + \sqrt{\log(1+2{R}^{(T)} R_\trans)} }{2\sqrt{n}}\notag\\
& \quad\qquad + \inf_{\xi \in (0, 1/2)}\biggl[\biggl(4-\frac{12D\sqrt{(4+h)T} \cdot \sqrt{\log(1+2{R}^{(T)}R_\trans)}}{\sqrt{n}}\biggr)\cdot \xi - \frac{24D\sqrt{(4+h)T}}{\sqrt{n}} \cdot \sqrt{\xi}\biggr],\notag
\#
where the second inequality follows from Lemma \ref{lemma:cover_trans} and the last inequality follows from \eqref{eq:int-dudley}. On the right-hand side of \eqref{eq:rad-bound1}, we set $\xi = 0^+$ to obtain for any $j \in [d_\ty]$ that
\#\label{eq:rad-bound2}
{\cR}_{\cD_n}(\cF^L_j) & \leq 12D\sqrt{(4+h)T} \cdot \frac{2\sqrt{2} + \sqrt{\log(1+2{R}^{(T)} R_\trans)} }{2\sqrt{n}}\notag\\
& =  O\biggl(D\sqrt{hT} \cdot \frac{1 + \sqrt{\log(1 + {R}^{(T)} R_\trans)}}{\sqrt{n}}\biggr).
\#
Recall that $\cF^L$ is defined in \eqref{eq:trans-class}. Taking \eqref{eq:rad-bound2} into Lemma \ref{thm:gen-base}, Lemma \ref{lemma:contraction}, and \eqref{eq:rad-decomp}, we obtain with probability at least $1 - \delta$ over the independent and identical draw of the dataset $\cD_n = \{(X_i, y_i)\}_{i \in [n]}$, it holds for any $f \in \cF^L$ that
\#\label{eq:gen-raw}
& \biggl|\EE\Bigl[\cL\bigl((X, y), f\bigr)\Bigr] - \hat{\EE}\Bigl[\cL\bigl((X, y), f\bigr)\Bigr]\biggr|\notag\\
&\quad \leq 2\sqrt{2}\cdot \sum_{j = 1}^{d_\ty} \cR_{\cD_n}(\cF_j^L) + 3\sqrt{\frac{\log(2/\delta)}{2n}}\notag\\
&\quad = O\biggl( D^2\sqrt{hT} \cdot \frac{1 + \sqrt{\log(1+R^{(T)} R_\trans)} }{\sqrt{n}} + \sqrt{\frac{\log (1/\delta)}{n}}\biggr).
\#
By simplifying the first term on the right-hand side of \eqref{eq:gen-raw} using Lemma \ref{lemma:simplify}, we have with probability at least $1 - \delta$ over the independent and identical draw of the dataset $\cD_n = \{(X_i, y_i)\}_{i \in [n]}$ that
\#\label{eq:gen-diff}
& \biggl|\EE\Bigl[\cL\bigl((X, y), f\bigr)\Bigr] - \hat{\EE}\Bigl[\cL\bigl((X, y), f\bigr)\Bigr]\biggr|\\
& \quad = O\biggl( \frac{D^2}{\sqrt{n}} \cdot \Bigl[T\cdot\sqrt{\log(1+ \gamma)} + \sqrt{T} \cdot \sqrt{\log(1 + \zeta R^{(0)})} + \sqrt{\log( 1+ \kappa/\zeta)}\Bigr] \cdot \sqrt{hT}+ \sqrt{\frac{\log (1/\delta)}{n}}\biggr),\notag\\
& \quad = O\biggl( \frac{D^2}{\sqrt{n}} \cdot \Bigl[T\cdot\sqrt{\log(1+ \gamma)} + \sqrt{T} \cdot \sqrt{\log(1 + \zeta R)} + \sqrt{\log( 1+ \kappa/\zeta)}\Bigr] \cdot \sqrt{hT}+ \sqrt{\frac{\log (1/\delta)}{n}}\biggr)\notag
\#
holds for any $f \in \cF^L$, where the last line follows from $R \geq R^{(0)}$ defined in Assumption \ref{asu::data}. Here $R^{(0)}$ is defined in \eqref{eq:rt}. Recall that $\tilde f = \argmin_{f \in \cF^L}\hat \EE[ \cL((X, y),  f)]$ is the empirical risk minimizer, where $\cF^L$ is the $T$-layer version of $\cF_\att$ defined in \eqref{eq:trans-class-simple}. Let $\overbar f = \argmin_{f\in\cF^L}  \EE[ \cL((X, y),  f) ]$ be the population risk minimizer. We have
\$
& \hat\EE\Bigl[\cL\bigl((X, y), \tilde f\bigr)\Bigr] - \min_{f\in\cF^L}  \EE\Bigl[ \cL\bigl((X, y),  f\bigr) \Bigr]\\
&\quad = \min_{f\in\cF^L}\hat\EE\Bigl[\cL\bigl((X, y), f\bigr)\Bigr] -  \EE\Bigl[ \cL\bigl((X, y),  \overbar f\bigr) \Bigr]\\
&\quad = \min_{f\in\cF^L}\hat\EE\Bigl[\cL\bigl((X, y), f\bigr)\Bigr] -  \hat \EE\Bigl[ \cL\bigl((X, y),  \overbar f\bigr) \Bigr] + \hat\EE\Bigl[\cL\bigl((X, y), \overbar f\bigr)\Bigr] -  \EE\Bigl[ \cL\bigl((X, y),  \overbar f\bigr) \Bigr]\\
& \quad \leq \hat\EE\Bigl[\cL\bigl((X, y), \overbar f\bigr)\Bigr] -  \EE\Bigl[ \cL\bigl((X, y),  \overbar f\bigr) \Bigr].
\$
By the definition of the generalization error $\cE_{\rm gen}$ in \eqref{eq::risk_decomp}, we have with probability at least $1 - \delta$ over the independent and identical draw of the dataset $\cD_n = \{(X_i, y_i)\}_{i \in [n]}$ that
\$
\cE_{\rm gen} & = \EE\Bigl[\cL\bigl((X, y), \hat f\bigr)\Bigr] - \hat\EE\Bigl[ \cL\bigl((X, y), \hat f\bigr) \Bigr] + \hat\EE\Bigl[\cL\bigl((X, y), \tilde f\bigr)\Bigr] - \min_{f\in\cF^L}  \EE\Bigl[ \cL\bigl((X, y),  f\bigr) \Bigr]\\
& \leq \EE\Bigl[\cL\bigl((X, y), \hat f\bigr)\Bigr] - \hat\EE\Bigl[ \cL\bigl((X, y), \hat f\bigr) \Bigr] + \hat\EE\Bigl[\cL\bigl((X, y), \overbar f\bigr)\Bigr] -  \EE\Bigl[ \cL\bigl((X, y),  \overbar f\bigr) \Bigr]\\
& = O\biggl( \frac{D^2}{\sqrt{n}} \cdot \Bigl[T \cdot \sqrt{\log(1+ \gamma)} + \sqrt{T}\cdot \sqrt{\log(1 + \zeta R)} + \sqrt{\log( 1+ \kappa/\zeta)}\Bigr] \cdot \sqrt{hT}+ \sqrt{\frac{\log (1/\delta)}{n}}\biggr),
\$
where the last line follows from \eqref{eq:gen-diff}. Therefore, we conclude the proof of Theorem \ref{thm:gen}.
\end{proof}

\subsection{Proof of Lemma \ref{lemma:cover_trans}}\label{appendix:cover}
We organize the proof of Lemma \ref{lemma:cover_trans} as follows.
\begin{itemize}
    \item[(\S\ref{appendix:modular})] We analyze the covering numbers of the empirical image classes of MHA and FFN, respectively, which serve as the building blocks for covering the empirical image class of the transformer architecture.
    \item[(\S\ref{appendix:propagation})] We propagate the covering numbers of MHA and FFN through the $T$ layers of the transformer architecture.
    \item[(\S\ref{appendix:cover_trans})] We combine \S\ref{appendix:modular} and \S\ref{appendix:propagation} to characterize the covering number of the empirical image class of the transformer architecture.
\end{itemize}

\subsubsection{Covering Numbers of FFN and MHA }\label{appendix:modular}
In what follows, we characterize the covering numbers of the empirical image classes of FFN and MHA. The following lemma characterizes the covering number of a set of matrices with respect to the matrix $(r, s)$-norm, which serves as the building block for analyzing the covering number of the empirical image classes of FFN and MHA.

\begin{lemma}[Covering Number of Matrix Set]\label{lemma:cover_matrix_ball}
Let $(r, s)$ be a conjugate pair. We have 
\$
\log N \Bigl(\bigl\{M^\top \in \RR^{d_2 \times d_1}: \|M^\top\|_{r, s} \leq R_M\bigr\}, \varepsilon, \|\cdot\|_{r, s}\Bigr) \leq d_1 d_2 \cdot \log\biggl(1+ \frac{2R_M}{\varepsilon}\biggr),
\$
where $R_M, \varepsilon > 0$.
\end{lemma}
\begin{proof}
See \S\ref{appendix:cover_ball} for a detailed proof.
\end{proof}


\noindent{\bf Empirical Image Class of FFN.} In the sequel, we characterize the covering number of the empirical image class of FFN. In parallel to the parameter space $\mathfrak{A}^{(t)}$ defined in \eqref{eq:set-at}, we define
\#\label{eq:set-a}
{\mathfrak{A}} & = \Bigl\{(A^\tx, A^\sigma) \in \RR^{d \times d_\sigma} \times \RR^{d_\sigma \times d}:\notag\\
& \qquad \bigl\|({A}^\tx)^\top\bigr\|_{r} \leq {\alpha}^\tx, \bigl\|({A}^\tx)^\top\bigr\|_{r, s} \leq {R}^\tx, \bigl\|({A}^\sigma)^\top\bigr\|_{r} \leq {\alpha}^\sigma, \bigl\|({A}^\sigma)^\top\bigr\|_{r, s} \leq R^\sigma\Bigr\},
\#
where $\alpha^\tx, \alpha^\sigma , R^\tx, R^\sigma > 0$, while $d$ and $d_\sigma$ are positive integers. Correspondingly, we define the function class of FFN and the empirical image class of FFN as follows,
\#\label{eq:class-ffn-temp}
\cF_\ffn & = \bigl\{\ffn(X_\star; A): A \in \mathfrak{A}\bigr\},\notag\\
\mathfrak{I}_{\tilde{\cD}_{n\star}}(\cF_\ffn) & = \Bigl\{\bigl(f(\tilde X_{i\star})^\top\bigr)_{i \in [n]} \in \RR^{d \times nL}: f \in \cF_\ffn\Bigr\}.
\#
Here $\tilde{\cD}_{n\star} = \{\tilde X_{i\star}\}_{i \in [n]}$ is the input set of FFN, where $\tilde X_{i\star} \in \RR^{L \times d}$. In the following, we characterize the covering number of the empirical image class $\mathfrak{I}_{\tilde \cD_{n\star}}(\cF_\ffn)$.
\begin{lemma}[Covering Number of FFN]\label{lemma:cover_ffn}
Let $\varepsilon > 0$. Suppose that the input set $\tilde{D}_{n\star} = \{\tilde X_{i\star}\}_{i \in [n]} \subseteq \RR^{L \times d}$ of FFN satisfies $\max_{i \in [n]}\|\tilde X_{i\star}^\top \|_{r, \infty} \leq \tilde R_\star$. Then, we have
\$
\log N\bigl(\mathfrak{I}_{\tilde \cD_{n\star}}(\cF_\ffn), \varepsilon, \|\cdot\|_{r,\infty}\bigr) \leq 2dd_\sigma \cdot \log\biggl(1 + \frac{2(\alpha^\tx R^\sigma + \alpha^\sigma R^\tx) \cdot \tilde R_\star}{\varepsilon}\biggr),
\$
where $\alpha^\tx, \alpha^\sigma, R^\tx, R^\sigma$ are defined in \eqref{eq:set-a}.
\end{lemma}
\begin{proof}
For any $A = ({A}^\tx, {A}^\sigma) \in \mathfrak{A}$, suppose that $\hat{A} = (\hat{A}^\tx, \hat{A}^\sigma) \in \mathfrak{A}$ satisfy
\#\label{eq:eps-ffn-temp}
\bigl\|(A^\tx)^\top - (\hat{A}^\tx)^\top \bigr\|_{r, s} & \leq \varepsilon^\tx, \qquad \bigl\|(A^\sigma)^\top - (\hat{A}^\sigma)^\top\bigr\|_{r, s} \leq \varepsilon^\sigma,
\#
where $\varepsilon^\tx, \varepsilon^\sigma > 0$. For any $i \in [n]$, we have
\$
& \bigl\|\ffn(\tilde X_{i\star}; A)^\top - \ffn(\tilde X_{i\star};\hat{A})^\top\bigr\|_{r, \infty}\\
&\quad = \Bigl\|\bigl(\relu(\tilde X_{i\star} A^\tx)A^\sigma\bigr)^\top + \tilde X_{i\star}^\top - \bigl(\relu( \tilde X_{i\star}\hat{A}^\tx)\hat{A}^\sigma\bigr)^\top - \tilde X_{i\star}^\top\Bigr\|_{r, \infty}\notag\\
&\quad \leq \Bigl\|\bigl((A^\sigma)^\top - (\hat{A}^\sigma)^\top\bigr)\relu(\tilde X_{i\star} A^\tx)^\top\Bigr\|_{r, \infty} + \Bigl\| (\hat{A}^\sigma)^\top\bigl(\relu(\tilde X_{i\star}\hat{A}^\tx) - \relu(\tilde X_{i\star}{A}^\tx)\bigr)^\top\Bigr\|_{r, \infty}\notag\\
&\quad \leq \big\|(A^\sigma)^\top - (\hat{A}^\sigma)^\top\bigr\|_{r, s}\cdot \bigl\|(A^\tx)^\top\bigr\|_{r} \cdot \|\tilde X_{i\star}^\top\|_{r, \infty} + \bigl\| (\hat{A}^\sigma)^\top\bigr\|_r \cdot \bigl\|(\hat{A}^\tx - {A}^\tx)^\top\bigr\|_{r, s}\cdot \|\tilde X_{i\star}^\top\|_{r, \infty}\notag\\
&\quad \leq \varepsilon^\sigma \cdot \bigl\|(A^\tx)^\top\bigr\|_{r} \cdot \|\tilde X_{i\star}^\top\|_{r, \infty} + \varepsilon^\tx \cdot \bigl\| (\hat{A}^\sigma)^\top\bigr\|_r\cdot \|\tilde X_{i\star}^\top\|_{r, \infty}\notag\\
&\quad \leq (\varepsilon^\sigma \alpha^\tx + \varepsilon^\tx \alpha^\sigma)\cdot \tilde R_\star,\notag
\$
where the second line follows from the definition of FFN in \eqref{eq:ffn}, the fourth line follows from Lemma \ref{lemma:norm}, the fifth line follows from \eqref{eq:eps-ffn-temp}, and the last line follows from the definition of $\cF_\ffn$ in \eqref{eq:class-ffn-temp} and the fact that $\max_{i \in [n]}\|\tilde X_{i \star}^\top\|_{r, \infty} \leq \tilde R_\star$. Setting 
\#\label{eq:eps-a}
\varepsilon^\tx = \varepsilon \cdot \frac{R^\tx}{(\alpha^\tx R^\sigma + \alpha^\sigma R^\tx) \cdot \tilde R_\star} , \qquad \varepsilon^\sigma = \varepsilon \cdot \frac{R^\sigma}{(\alpha^\tx R^\sigma + \alpha^\sigma R^\tx) \cdot \tilde R_\star},
\#
we obtain $\|\ffn(\tilde X_{i\star}; A)^\top - \ffn(\tilde X_{i\star};\hat{A})^\top\|_{r, \infty} \leq \varepsilon$ for any $i \in [n]$, which implies
\$
\Bigl\|\bigl(\ffn(\tilde X_{i\star}; A)^\top\bigr)_{i \in [n]}^\top - \bigl(\ffn(\tilde X_{i\star}; \hat A)^\top\bigr)_{i \in [n]}^\top\Bigr\|_{r, \infty} = \max_{i \in [n]}\bigl\|\ffn(\tilde X_{i\star}; A)^\top - \ffn(\tilde X_{i\star}; \hat A)^\top\bigr\|_{r, \infty} \leq \varepsilon.
\$
To cover the empirical image class $\mathfrak{I}_{\tilde \cD_{n\star}}(\cF_\ffn)$ at the resolution $\varepsilon$, it remains to cover the parameter spaces of $A^\tx$ and $A^\sigma$ at the resolutions $\varepsilon^\tx$ and $\varepsilon^\sigma$, respectively, that is,
\$
&\log N\bigl(\mathfrak{I}_{\tilde \cD_{n\star}}(\cF_\ffn), \varepsilon, \|\cdot\|_{r,\infty}\bigr)\\
& \quad \leq \log N\biggl(\Bigl\{(A^\tx)^\top \in \RR^{d_\sigma \times d}: \bigl\|(A^\tx)^\top\bigr\|_{r, s} \leq R^\tx,  \bigl\|(A^\tx)^\top\bigr\|_{r} \leq \alpha^\tx \Bigr\}, \varepsilon^\tx, \|\cdot\|_{r,\infty}\biggr)\\
&\quad\qquad + \log N\biggl(\Bigl\{(A^\sigma)^\top \in \RR^{d \times d_\sigma}: \bigl\|(A^\sigma)^\top\bigr\|_{r, s} \leq R^\sigma, \bigl\|(A^\sigma)^\top\bigr\|_{r} \leq \alpha^\sigma \Bigr\}, \varepsilon^\sigma, \|\cdot\|_{r,\infty}\biggr)\\
& \quad \leq \log N\biggl(\Bigl\{(A^\tx)^\top \in \RR^{d_\sigma \times d}: \bigl\|(A^\tx)^\top\bigr\|_{r, s} \leq R^\tx \Bigr\}, \varepsilon^\tx, \|\cdot\|_{r,\infty}\biggr)\\
&\quad\qquad + \log N\biggl(\Bigl\{(A^\sigma)^\top \in \RR^{d \times d_\sigma}: \bigl\|(A^\sigma)^\top\bigr\|_{r, s} \leq R^\sigma \Bigr\}, \varepsilon^\sigma, \|\cdot\|_{r,\infty}\biggr)\\
& \quad \leq dd_\sigma \cdot \log\biggl(1 + \frac{2R^\tx}{\varepsilon^\tx}\biggr) + dd_\sigma \cdot \log\biggl(1 + \frac{2R^\sigma}{\varepsilon^\sigma}\biggr)\\
& \quad = 2dd_\sigma \cdot \log\biggl(1 + \frac{2(\alpha^\tx R^\sigma + \alpha^\sigma R^\tx) \cdot \tilde R_\star}{\varepsilon}\biggr),
\$
where the third inequality follows from Lemma \ref{lemma:cover_matrix_ball} and the equality follows from \eqref{eq:eps-a}. Therefore, we conclude the proof of Lemma \ref{lemma:cover_ffn}.
\end{proof}

\noindent{\bf Empirical Image Class of MHA.} In the sequel, we characterize the covering number of the empirical image class of MHA. In parallel to the parameter space $\mathfrak{W}^{(t)}$ defined in \eqref{eq:set-wt}, we define
\#\label{eq:set-w}
\mathfrak{W} & = \Bigl\{\bigl\{(W^{\tq}_i, W^{\tk}_i, W^{\tv}_i)\bigr\}_{i \in [h]}: (W^{\tq}_i, W^{\tk}_i, W^{\tv}_i) \in \RR^{d \times d_\tp} \times\RR^{d \times d_\tp} \times\RR^{d \times d},\\
&\quad \qquad \bigl\|(W^{\tq}_i)^\top\bigr\|_{r} \leq \omega^{\tq}_i, \bigl\|(W^{\tk}_i)^\top\bigr\|_{r} \leq \omega^{\tk}_i, \bigl\|(W^{\tv}_i)^\top\bigr\|_{r} \leq \omega^{\tv}_i,\notag\\
&\quad \qquad \bigl\|(W^{\tq}_i)^\top\bigr\|_{r, s} \leq R^{\tq}_i, \bigl\|(W^{\tk}_i)^\top\bigr\|_{r, s} \leq R^{\tk}_i, \bigl\|(W^{\tv}_i)^\top\bigr\|_{r, s} \leq R^{\tv}_i\Bigr\},\notag
\#
where $\omega^{\tq}_i, \omega^{\tk}_i, \omega^{\tv}_i, R^{\tq}_i, R^{\tk}_i, R^{\tv}_i > 0$. Correspondingly, we define the function class of MHA (with a skip connection) and the empirical image class of MHA as follows,
\#\label{eq:class-mha-temp}
\cF_\mha & = \bigl\{\mha(X; W) + X: W \in \mathfrak{W}\bigr\},\notag\\
\mathfrak{I}_{\tilde{\cD}_{n}}(\cF_\mha) & = \Bigl\{\bigl(f(\tilde X_{i})^\top\bigr)_{i \in [n]} \in \RR^{d \times nL}: f \in \cF_\mha\Bigr\}.
\#
Here $\tilde{\cD}_{n} = \{\tilde X_{i}\}_{i \in [n]}$ is the input set of MHA, where $\tilde X_{i} \in \RR^{L \times d}$. Recall that $h$ is the head number of MHA. The following lemma characterizes the Lipschitz continuity of MHA with respect to the parameter in $W = \bigl\{(W^{\tq}_i, W^{\tk}_i, W^{\tv}_i)\bigr\}_{i \in [h]} \in \mathfrak{W}$.
\begin{lemma}[Parameter Lipschitz Continuity of MHA]\label{lemma:lip-mha-param}
Let $(r, s)$ be a conjugate pair. Suppose that $\tilde X \in \RR^{L \times d}$ satisfies $\|\tilde X^\top\|_{r, \infty} \leq \tilde R$. Given any $W = \{(W^{\tq}_i, W^{\tk}_i, W^{\tv}_i)\}_{i \in [h]}\in \mathfrak{W}$, suppose that $\hat{W} = \{(\hat{W}^{\tq}_i, \hat{W}^{\tk}_i, \hat{W}^{\tv}_i)\}_{i \in [h]} \in \mathfrak{W}$ satisfies
\$
\bigl\|(W^{\tq}_i)^\top - (\hat{W}^{\tq}_i)^\top \bigr\|_{r, s} \leq \varepsilon_i^\tq, \quad \bigl\|(W^{\tk}_i)^\top - (\hat{W}^{\tk}_i)^\top\bigr\|_{r, s} \leq \varepsilon_i^\tk, \quad \bigl\|(W^{\tv}_i)^\top - (\hat{W}^{\tv}_i)^\top\bigr\|_{r, s} \leq \varepsilon_i^\tv
\$
for any $i \in [h]$. Then, we have
\$
\bigl\|\mha(\tilde X; W)^\top - \mha(\tilde X; \hat{W})^\top\bigr\|_{r, \infty} \leq \tilde R\cdot \sum_{i = 1}^h \varepsilon_i^\tv +  \tilde R^3\cdot \sum_{i = 1}^h (\omega^{\tq}_i + \omega^{\tk}_i) \cdot (\varepsilon_i^\tq + \varepsilon_i^\tk),
\$
where $\omega^{\tq}_i$, $\omega^{\tk}_i$, and $\omega^{\tv}_i$ are defined in \eqref{eq:set-w}.
\end{lemma}
\begin{proof}
See \S\ref{appendix:lip-mha-param} for a detailed proof.
\end{proof}

The following lemma characterizes the covering number of the empirical image class $\mathfrak{I}_{\tilde \cD_n}(\cF_\mha)$ defined in \eqref{eq:class-mha-temp}.

\begin{lemma}[Covering Number of MHA]\label{lemma:cover_attn}
Let $\varepsilon > 0$. Suppose that the input set $\tilde{\cD}_{n} = \{\tilde X_{i}\}_{i \in [n]} \subseteq \RR^{L \times d}$ of MHA satisfies $\max_{i \in [n]}\|\tilde X_{i}^\top \|_{r, \infty} \leq \tilde R$. Then, we have
\$
\log N\bigl(\mathfrak{I}_{\tilde \cD_n}(\cF_\mha), \varepsilon, \|\cdot\|_{r,\infty}\bigr) \leq (2 + h)\cdot d^2\cdot \log \biggl(1 + \frac{2\tilde R \cdot R_\mha(\mathfrak{W})}{\varepsilon}\biggr).
\$
Here
\#\label{eq:r-mha}
R_\mha(\mathfrak{W}) = \sum_{i = 1}^hR^{\tv}_i + \tilde R^2\cdot \sum_{i = 1}^h (\omega^{\tq}_i + \omega^{\tk}_i)(R^{\tq}_i + R^{\tk}_i),
\#
where $\omega^{\tq}_i$, $\omega^{\tk}_i$, $R^{\tq}_i$, $R^{\tk}_i$, and $R^{\tv}_i$ are defined in \eqref{eq:set-w}.
\end{lemma}
\begin{proof}
Throughout the following proof, we set $\varepsilon^\tq_i, \varepsilon^\tk_i, \varepsilon^\tv_i > 0$ such that
\#\label{eq:eps-mha-temp}
\frac{R^{\tq}_i}{\varepsilon_i^\tq}=\frac{R^{\tk}_i}{\varepsilon_i^\tk}=\frac{R^{\tv}_i}{\varepsilon_i^\tv} = \frac{\tilde R\cdot R_\mha(\mathfrak{W})}{\varepsilon}
\#
for any $i \in [h]$, where $R_\mha(\mathfrak{W})$ is defined in \eqref{eq:r-mha}. By Lemma \ref{lemma:lip-mha-param}, we have for any $i \in [n]$ that
\$
&\bigl\|\mha(\tilde X_i; W)^\top + \tilde X_i^\top - \mha(\tilde X_i; \hat{W})^\top - \tilde X_i^\top\bigr\|_{r, \infty}\\
& \quad \leq \tilde R\cdot \sum_{i = 1}^h \varepsilon_i^\tv + \tilde R^3\cdot \sum_{i = 1}^h (\omega^{\tq}_i + \omega^{\tk}_i) \cdot (\varepsilon_i^\tq + \varepsilon_i^\tk)\\
& \quad = \tilde R\cdot \sum_{i = 1}^h \frac{R^\tv_i\cdot \varepsilon}{\tilde R\cdot R_\mha(\mathfrak{W})} + \tilde R^3\cdot \sum_{i = 1}^h (\omega^{\tq}_i + \omega^{\tk}_i) \cdot \biggl(\frac{R^\tq_i \cdot \varepsilon}{\tilde R\cdot R_\mha(\mathfrak{W})} + \frac{R^\tk_i \cdot \varepsilon}{\tilde R\cdot R_\mha(\mathfrak{W})}\biggr)\\
& \quad = \varepsilon \cdot \frac{\sum_{i = 1}^hR^\tv_i + \tilde R^2\cdot \sum_{i = 1}^h (\omega^{\tq}_i + \omega^{\tk}_i) \cdot (R^\tq_i + R^\tk_i ) }{\tilde R\cdot R_\mha(\mathfrak{W})}  = \varepsilon,
\$
where the third line follows from \eqref{eq:eps-mha-temp} and the last line follows from the definition of $R_\mha(\mathfrak{W})$ in \eqref{eq:r-mha}. Hence, we have
\$
& \Bigl\|\bigl(\mha(\tilde X_i; W)^\top + \tilde X_i^\top\bigr)_{i \in [n]}^\top - \bigl(\mha(\tilde X_i; \hat{W})^\top + \tilde X_i^\top\bigr)_{i \in [n]}^\top\Bigr\|_{r, \infty}\\
&\quad = \max_{i \in [n]}\bigl\|\mha(\tilde X_i; W)^\top + \tilde X_i^\top  - \mha(\tilde X_i; \hat{W})^\top - \tilde X_i^\top\bigr\|_{r, \infty} \leq \varepsilon.
\$
To cover the empirical image class $\mathfrak{I}_{\tilde \cD_{n}}(\cF_\mha)$ at the resolution $\varepsilon$, it remains to cover the parameter spaces of $W^\tq_i$, $W^\tk_i$, and $W^\tv_i$ at the resolutions $\varepsilon^\tq_i$, $\varepsilon^\tk_i$, and $\varepsilon^\tv_i$, respectively, for any $i \in [n]$, that is,
\$
& \log N\Bigl(\mathfrak{I}_{\tilde \cD_n}(\cF_\mha), \varepsilon, \|\cdot\|_{r,\infty}\Bigr)\\
& \quad \leq \sum_{i = 1}^h\biggl(\log N\biggl(\Bigl\{(W^{\tq}_i)^\top \in \RR^{d_\tp \times d}: \bigl\|(W^{\tq}_i)^\top\bigr\|_{r, s} \leq R^{\tq}_i, \bigl\|(W^{\tq}_i)^\top\bigr\|_{r} \leq \omega^{\tq}_i \Bigr\}, \varepsilon_i^\tq, \|\cdot\|_{r, \infty}\biggr)\\
& \quad\qquad + \log N\biggl(\Bigl\{(W^{\tk}_i)^\top \in \RR^{d_\tp \times d}: \bigl\|(W^{\tk}_i)^\top\bigr\|_{r, s} \leq R^{\tk}_i, \bigl\|(W^{\tk}_i)^\top\bigr\|_{r} \leq \omega^{\tk}_i \Bigr\}, \varepsilon_i^\tk, \|\cdot\|_{r, \infty}\biggr)\\
&\quad\qquad + \log N\biggl(\Bigl\{(W^{\tv}_i)^\top \in \RR^{d \times d}: \bigl\|(W^{\tv}_i)^\top\bigr\|_{r, s} \leq R^{\tv}_i, \bigl\|(W^{\tv}_i)^\top\bigr\|_{r} \leq \omega^{\tv}_i \Bigr\}, \varepsilon_i^\tv, \|\cdot\|_{r, \infty}\biggr)\biggr)\\
& \quad \leq \sum_{i = 1}^h\biggl(\log N\biggl(\Bigl\{(W^{\tq}_i)^\top \in \RR^{d_\tp \times d}: \bigl\|(W^{\tq}_i)^\top\bigr\|_{r, s} \leq R^{\tq}_i \Bigr\}, \varepsilon_i^\tq, \|\cdot\|_{r, \infty}\biggr)\\
& \quad\qquad + \log N\biggl(\Bigl\{(W^{\tk}_i)^\top \in \RR^{d_\tp \times d}: \bigl\|(W^{\tk}_i)^\top\bigr\|_{r, s} \leq R^{\tk}_i \Bigr\}, \varepsilon_i^\tk, \|\cdot\|_{r, \infty}\biggr)\\
&\quad\qquad + \log N\biggl(\Bigl\{(W^{\tv}_i)^\top \in \RR^{d \times d}: \bigl\|(W^{\tv}_i)^\top\bigr\|_{r, s} \leq R^{\tv}_i \Bigr\}, \varepsilon_i^\tv, \|\cdot\|_{r, \infty}\biggr)\biggr)\\
& \quad \leq dd_\tp \cdot\sum_{i = 1}^h \biggl( \log \biggl(1 + \frac{2R^{\tq}_i}{\varepsilon_i^\tq}\biggr) +\log \biggl(1 + \frac{2R^{\tk}_i}{\varepsilon_i^\tk}\biggr)\biggr) +  d^2 \cdot \sum_{i = 1}^h\log \biggl(1 + \frac{2R^{\tv}_i}{\varepsilon_i^\tv}\biggr)\\
& \quad =  (2 + h) \cdot d^2 \cdot \log \biggl(1 + \frac{2\tilde R R_\mha(\mathfrak{W})}{\varepsilon}\biggr),
\$
where the third inequality follows from Lemma \ref{lemma:cover_matrix_ball} and the equality follows from \eqref{eq:eps-mha-temp} and the fact that $d = d_\tp\cdot h$. Therefore, we conclude the proof of Lemma \ref{lemma:cover_attn}.
\end{proof}

%
%
%
%

\subsubsection{Propagation of Covering Numbers}\label{appendix:propagation}
Recall that $\mathfrak{W}$ is defined in \eqref{eq:set-w}. The following lemma characterizes the Lipschitz continuity of MHA in the input $X$.
\begin{lemma}[Input Lipschitz Continuity of MHA]\label{lemma:lip-mha-input}
Let $(r, s)$ be a conjugate pair. Suppose that $X\in \RR^{L \times d}$ and $\hat{X} \in \RR^{L \times d}$ satisfy $\|X^\top\|_{r, \infty} \leq \tilde R$ and $\|\hat{X}^\top\|_{r, \infty} \leq \tilde R$, respectively. Then for any $W = \{(W_i^\tq, W_i^\tk, W_i^\tv)\}_{i \in [n]} \in \mathfrak{W}$, we have
\$
\bigl\|\mha(X; W)^\top - \mha(\hat{X}; W)^\top\bigr\|_{r, \infty} \leq \rho(\mathfrak{W})\cdot \|X^\top - \hat{X}^{\top}\|_{r, \infty}.
\$
Here
\#\label{eq:rho}
\rho(\mathfrak{W}) = \sum_{i = 1}^h\omega^{\tv}_i  + \tilde R^2\cdot\sum_{i = 1}^h(\omega^{\tq}_i+ \omega^{\tk}_i)^2\omega^{\tv}_i,
\#
where $\omega^\tq_i$, $\omega^\tk_i$, and $\omega^\tv_i$ are defined in \eqref{eq:set-w}.
\end{lemma}
\begin{proof}
See \S\ref{appendix:lip-mha-input} for a detailed proof.
\end{proof}

Recall that the empirical image classes $\mathfrak{I}_{\cD_n}(\cF_\ffn^{L , (t)})$, $\mathfrak{I}_{\cD_n}(\cF_\mha^{L, (t+1)})$, and  $\mathfrak{I}_{\cD_n}(\cF_j^{L})$ are defined in \eqref{eq:trans-image} and \eqref{eq:image-fj}. Also, recall that the parameter space $\Theta^{(t)}$ is specified in Assumption \ref{assumption:constraint}. The following lemma characterizes the propagation of the covering numbers of the empirical image classes of FFN and MHA through the $T$ layers of the transformer architecture.

\begin{lemma}[Propagation of Covering Number]\label{lemma:prop}
Suppose that Assumption \ref{assumption:constraint} holds. For any $j \in [d_\ty]$, we have
\$
\log N\bigl(\mathfrak{I}_{\cD_n}(\cF_j^{L}), \varepsilon, \|\cdot\|_{r, \infty}\bigr) & \leq \sum_{t = 0}^{T-2} \sup_{\{\theta^{(\tau)} \in \Theta^{(\tau)}\}_{0 \leq \tau \leq t}} \log N\bigl(\mathfrak{I}_{\cD_n}(\cF^{L, (t+1)}_\mha), \varepsilon^{(t)}_\mha, \|\cdot\|_{r, \infty}\bigr)\\
&\quad\qquad +  \sum_{t = 0}^{T-1} \sup_{\{\theta^{(\tau)} \in \Theta^{(\tau)}\}_{0 \leq \tau \leq t}} \log N\bigl(\mathfrak{I}_{\cD_n}(\cF^{L, (t)}_\ffn), \varepsilon^{(t)}_\ffn, \|\cdot\|_{r, \infty}\bigr).
\$
With the conventions $\prod_{\tau = T-1}^{T-1} \cdot \equiv 1$ and $\prod_{\tau = T}^{T-1} \cdot \equiv 1$, the covering resolution $\varepsilon$ is defined as follows,
\#\label{eq:epsilon}
\varepsilon = \sum_{t = 0}^{T-1} \biggl(\tilde{\rho}^{(t)} \varepsilon^{(t)}_\ffn\cdot \prod_{\tau = t+1}^{T-1}\tilde{\rho}^{(\tau)}\tilde{\alpha}^{(\tau)} \biggr) + \sum_{t = 0}^{T-2} \biggl(\varepsilon^{(t)}_\mha\cdot \prod_{\tau = t+1}^{T-1}\tilde{\rho}^{(\tau)}\tilde{\alpha}^{(\tau)} \biggr),
\#
where $\tilde{\alpha}^{(t)}$ and $\tilde{\rho}^{(t)}$ are defined in \eqref{eq:mag-xt} and \eqref{eq:rho-rmha}, respectively.
\end{lemma}
\begin{proof}

\vskip4pt
Throughout the following proof, we fix the dataset $\cD_n = \{(X_i, y_i)\}_{i \in [n]}$ and the parameters $\overbar \theta$ and $\{\theta^{(t)} = (W^{(t)}, A^{(t)})\}_{0 \leq t \leq T-1}$. By \eqref{eq:trans-arch} and \eqref{eq:agg-trans}, the intermediate inputs $\{X_i^{(t)}\}_{i \in [n]}$ and $\{X_{i\star}^{(t)}\}_{i \in [n]}$ and the outputs $\{\hat y_{i} = (\hat y_{i, j})_{j \in [d_\ty]}^\top\}_{i \in [n]}$ are fixed.
\vskip4pt
\noindent{\bf Perturbed Intermediate Inputs.} For all $t = 0, \dots, T-1$, we denote by $\mathfrak{N}_\ffn^{(t)}$ and $\mathfrak{N}^{(t+1)}_\mha$ the covering sets of the empirical image classes $\mathfrak{I}_{\cD_n}(\cF_\ffn^{L , (t)})$ and $\mathfrak{I}_{\cD_n}(\cF_\mha^{L, (t+1)})$ at the resolutions $\varepsilon_\ffn^{(t)}$ and $\varepsilon_\mha^{(t)}$ with respect to the matrix $(r, \infty)$-norm, respectively. Starting from $((X_{i\star}^{(0)})^\top)_{i \in [n]} = ((\tilde X_{i\star}^{(0)})^\top)_{i \in [n]} =  (X_i^\top)_{i \in [n]}$, we construct the perturbed intermediate inputs in a recursive manner as follows,
\#\label{eq:xz-hat}
    \bigl((\tilde{X}^{(t)}_i)^\top\bigr)_{i \in [n]} & \in \Bigl\{(\tilde X_i^\top)_{i \in [n]} \in \mathfrak{N}_\ffn^{(t)}: \bigl\|\ffn(\tilde{X}^{(t)}_{i\star}; A^{(t)})^\top - \tilde X_i^\top\bigr\|_{r, \infty} \leq \varepsilon^{(t)}_\ffn\Bigr\},\\
    \bigl((\tilde{X}^{(t+1)}_{i\star})^\top\bigr)_{i \in [n]} & \in \Bigl\{(\tilde X_{i\star}^\top)_{i \in [n]} \in \mathfrak{N}_\mha^{(t+1)}: \bigl\|\mha(\tilde{X}^{(t)}_i; W^{(t)})^\top + (\tilde{X}_i^{(t)})^\top - \tilde X_{i\star}^\top\bigr\|_{r, \infty} \leq \varepsilon^{(t)}_\mha\Bigr\},\notag\\
        \bigl((\tilde{X}^{(T-1)}_i)^\top\bigr)_{i \in [n]} & \in \Bigl\{(\tilde X_i^\top)_{i \in [n]} \in \mathfrak{N}_\ffn^{(T-1)}: \bigl\|\ffn(\tilde{X}^{(T-1)}_{i\star}; A^{(T-1)})^\top - \tilde X_i^\top\bigr\|_{r, \infty} \leq \varepsilon^{(T-1)}_\ffn\Bigr\},\notag
\#
where $t = 0, \dots, T-2$. For any $i \in [n]$ and any $j \in [d_\ty]$, let 
\#\label{eq:xz-hat-final}
\tilde{X}^{(T)}_{i\star} & = \mha(\tilde{X}^{(T-1)}_i; W^{(T-1)}) + \tilde{X}_i^{(T-1)},\\
\tilde y_{i, j} & = \overbar \agg_{\overbar \theta, j}(\tilde X^{(T)}_{i\star}),\notag
\#
which implies
\$
|\hat y_{i, j} - \tilde y_{i, j}| = \bigl|\overbar \agg_{\overbar\theta, j}(X^{(T)}_{i\star}) - \overbar \agg_{\overbar\theta, j}(\tilde X^{(T)}_{i\star})\bigr| \leq \bigl\|(X^{(T)}_{i\star})^\top - (\tilde X^{(T)}_{i\star})^\top\bigr\|_{r, \infty},
\$
where the inequality follows from Assumption \ref{assumption:trans}. Hence, to cover the empirical image class $\mathfrak{I}_{\cD_n}(\cF^L_j)$ at the resolution $\varepsilon$, it remains to cover the empirical image class $\mathfrak{I}_{\cD_n}(\cF^{L, (T)}_\mha)$ at the resolution $\varepsilon$.
\vskip4pt
\noindent{\bf Propagation of Covering Resolutions.} For the recursive constructions in \eqref{eq:xz-hat}, it holds for any $i \in [n]$ that
\#\label{eq:xhat-diff}
& \bigl\|(X^{(t)}_i)^\top - (\tilde{X}^{(t)}_i)^\top\bigr\|_{r, \infty}\notag\\
 &\quad = \bigl\|\ffn({X}^{(t)}_{i\star}; A^{(t)})^\top - (\tilde{X}^{(t)}_i)^\top\bigr\|_{r, \infty}\notag\\
&\quad \leq \bigl\|\ffn({X}^{(t)}_{i\star}; A^{(t)})^\top - \ffn(\tilde{X}^{(t)}_{i\star}; A^{(t)})^\top\bigr\|_{r, \infty} + \bigl\| \ffn(\tilde{X}^{(t)}_{i\star}; A^{(t)})^\top - (\tilde{X}^{(t)}_i)^\top\bigr\|_{r, \infty}\notag\\
&\quad \leq \bigl\|\ffn({X}^{(t)}_{i\star}; A^{(t)})^\top - \ffn(\tilde{X}^{(t)}_{i\star}; A^{(t)})^\top\bigr\|_{r, \infty} + \varepsilon_\ffn^{(t)},
\#
where the first line follows from \eqref{eq:trans-arch} and the last line follows from the definition of $\tilde{X}^{(t)}_i$ in \eqref{eq:xz-hat}. For the first term on the right-hand side of \eqref{eq:xhat-diff}, it holds for all $t = 0, \dots, T-1$ and any $i \in [n]$ that
\#\label{eq:nn-diff}
& \bigl\|\ffn({X}^{(t)}_{i\star}; A^{(t)})^\top - \ffn(\tilde{X}^{(t)}_{i\star};{A}^{(t)})^\top\bigr\|_{r, \infty}\notag\\
&\quad = \bigl\|(A^{\sigma, (t)})^\top\relu({X}^{(t)}_{i\star}{A}^{\tx, (t)})^\top + ({X}^{(t)}_{i\star})^\top - (A^{\sigma, (t)})^\top\relu(\tilde{X}^{(t)}_{i\star}{A}^{\tx, (t)})^\top - (\tilde{X}^{(t)}_{i\star})^\top\bigr\|_{r, \infty}\notag\\
&\quad \leq \bigl\|(A^{\sigma, (t)})^\top\bigr\|_r \cdot \bigl\|\relu({X}^{(t)}_{i\star}{A}^{\tx, (t)})^\top - \relu(\tilde{X}^{(t)}_{i\star}{A}^{\tx, (t)})^\top\bigr\|_{r, \infty} + \bigl\| ({X}^{(t)}_{i\star})^\top - (\tilde{X}^{(t)}_{i\star})^\top\bigr\|_{r, \infty}\notag\\
&\quad \leq \Bigl(1 + \bigl\|({A}^{\tx, (t)})^\top\bigr\|_{r}\cdot \bigl\|({A}^{\sigma, (t)})^\top\bigr\|_{r}\Bigr) \cdot \bigl\| ({X}^{(t)}_{i\star})^\top - (\tilde{X}^{(t)}_{i\star})^\top\bigr\|_{r, \infty}\notag\\
&\quad \leq \tilde{\alpha}^{(t)} \cdot \bigl\| ({X}^{(t)}_{i\star})^\top - (\tilde{X}^{(t)}_{i\star})^\top\bigr\|_{r, \infty},
\#
where the third and fourth lines follow from Lemma \ref{lemma:norm} and the last line follows from the requirement in Assumption \ref{assumption:constraint} and the definition of $\tilde{\alpha}^{(t)}$ in \eqref{eq:mag-xt}. Hence, it holds for all $t = 0, \dots, T-2$ and any $i \in [n]$ that
\#\label{eq:z-diff}
\bigl\|({X}^{(t+1)}_{i\star})^\top - (\tilde{X}^{(t+1)}_{i\star})^\top\bigr\|_{r, \infty} & \leq  \bigl\|\mha({X}^{(t)}_i; A^{(t)})^\top + ({X}^{(t)}_i)^\top - \mha(\tilde{X}^{(t)}_i; A^{(t)})^\top - (\tilde{X}_i^{(t)})^\top\bigr\|_{r, \infty}\notag\\
&\qquad + \bigl\|\mha(\tilde{X}^{(t)}_i; A^{(t)})^\top + (\tilde{X}_i^{(t)})^\top - (\tilde{X}^{(t+1)}_{i\star})^\top\bigr\|_{r, \infty}\notag\\
& \leq \bigl(\rho(\mathfrak{W}^{(t)})+1\bigr) \cdot \bigl\|(X^{(t)}_i)^\top - (\tilde{X}_i^{(t)})^\top\bigr\|_{r, \infty} + \varepsilon^{(t)}_\mha\notag\\
& \leq \tilde{\rho}^{(t)} \cdot \bigl\|(X^{(t)}_i)^\top - (\tilde{X}^{(t)}_i)^\top \bigr\|_{r, \infty} + \varepsilon^{(t)}_\mha,
\#
where the second inequality follows from Lemma \ref{lemma:lip-mha-input} and the definition of $\tilde X^{(t+1)}_{i\star}$ in \eqref{eq:xz-hat} and the last inequality follows from Lemma \ref{lemma:const-bound}. Taking \eqref{eq:nn-diff} into \eqref{eq:xhat-diff} and \eqref{eq:xhat-diff} into \eqref{eq:z-diff}, we obtain for any $i \in [n]$ and $t = 0, \dots, T-2$ that
\#\label{eq:step-prop}
\bigl\|(X^{(t+1)}_{i\star})^\top - (\tilde{X}^{(t+1)}_{i\star})^\top \bigr\|_{r, \infty} \leq \tilde{\rho}^{(t)}\tilde{\alpha}^{(t)} \cdot \bigl\|(X_{i\star}^{(t)})^\top - (\tilde{X}_{i\star}^{(t)})^\top\|_{r, \infty} + \tilde{\rho}^{(t)} \varepsilon^{(t)}_\ffn +  \varepsilon^{(t)}_\mha.
\#
Recursively applying \eqref{eq:step-prop}, we have
\$
\bigl\|(X^{(T-1)}_{i\star})^\top - (\tilde{X}^{(T-1)}_{i\star})^\top \bigr\|_{r, \infty} \leq \sum_{t = 0}^{T-2} \biggl[(\tilde{\rho}^{(t)}\varepsilon^{(t)}_\ffn + \varepsilon^{(t)}_\mha)\cdot \prod_{\tau = t+1}^{T-2}\tilde{\rho}^{(\tau)}\tilde{\alpha}^{(\tau)} \biggr],
\$
which implies
\$
& \bigl\|(X^{(T)}_{i\star})^\top - (\tilde{X}^{(T)}_{i\star})^\top \bigr\|_{r, \infty}\\
& \quad  = \bigl\|\mha({X}^{(T-1)}_i; W^{(T-1)}) + {X}_i^{(T-1)} - \mha(\tilde{X}^{(T-1)}_i; W^{(T-1)}) - \tilde{X}_i^{(T-1)}\bigr\|_{r, \infty}\\
& \quad \leq \tilde{\rho}^{(T-1)} \cdot \bigl\|(X^{(T-1)}_i)^\top - (\tilde{X}^{(T-1)}_i)^\top \bigr\|_{r, \infty}\\
&\quad \leq \tilde{\rho}^{(T-1)}\tilde{\alpha}^{(T-1)} \cdot \bigl\| ({X}^{(t)}_{i\star})^\top - (\tilde{X}^{(t)}_{i\star})^\top\bigr\|_{r, \infty} + \tilde{\rho}^{(T-1)} \varepsilon_\ffn^{(T-1)}\\
&\quad \leq \sum_{t = 0}^{T-2} \biggl[(\tilde{\rho}^{(t)}\varepsilon^{(t)}_\ffn + \varepsilon^{(t)}_\mha)\cdot \prod_{\tau = t+1}^{T-1}\tilde{\rho}^{(\tau)}\tilde{\alpha}^{(\tau)} \biggr] + \tilde{\rho}^{(T-1)} \varepsilon_\ffn^{(T-1)} = \varepsilon,
\$
where the second line follows from \eqref{eq:xz-hat-final}, the third line follows from \eqref{eq:z-diff}, the fourth line follows from \eqref{eq:xhat-diff} and \eqref{eq:nn-diff}, and the last line follows from the definition of $\varepsilon$ in \eqref{eq:epsilon}.
To cover the empirical image class $\mathfrak{I}_{\cD_n}(\cF_\mha^{L, (T)})$ at the resolution $\varepsilon$, it suffices to cover (i) the empirical image class $\mathfrak{I}_{\cD_n}(\cF_\mha^{L, (t+1)})$ at the resolution $\varepsilon^{(t)}_\mha$ for all $t = 0, \dots, T-2$, and (ii) the empirical image class $\mathfrak{I}_{\cD_n}(\cF_\ffn^{L, (t)})$ at the resolution $\varepsilon^{(t)}_\ffn$ for all $t = 0, \dots, T-1$. Therefore, we conclude the proof of Lemma \ref{lemma:prop}.
\end{proof}

\subsubsection{Proof of Lemma \ref{lemma:cover_trans}}\label{appendix:cover_trans}
\begin{proof}
By Lemma \ref{lemma:prop}, we have
\$
\log N\bigl(\mathfrak{I}_{\cD_n}(\cF_j^{L}), \varepsilon, \|\cdot\|_{r, \infty}\bigr) & \leq \sum_{t = 0}^{T-2} \sup_{\{\theta^{(\tau)} \in \Theta^{(\tau)}\}_{0 \leq \tau \leq t}} \log N\bigl(\mathfrak{I}_{\cD_n}(\cF^{L, (t+1)}_\mha), \varepsilon^{(t)}_\mha, \|\cdot\|_{r, \infty}\bigr)\\
&\quad\qquad +  \sum_{t = 0}^{T-1} \sup_{\{\theta^{(\tau)} \in \Theta^{(\tau)}\}_{0 \leq \tau \leq t}} \log N\bigl(\mathfrak{I}_{\cD_n}(\cF^{L, (t)}_\ffn), \varepsilon^{(t)}_\ffn, \|\cdot\|_{r, \infty}\bigr),
\$
where $\varepsilon$ is defined in \eqref{eq:epsilon}. In what follows, we set 
\#\label{eq:cover-eps-t}
\varepsilon^{(t)}_\ffn = \varepsilon^{(t)}\cdot \frac{\alpha^{\tx, (t)} R^{\sigma, (t)} + \alpha^{\sigma, (t)} R^{\tx, (t)}}{\tilde{\omega}^{\tv, (t)}\tilde{\alpha}^{(t)}}, \qquad \varepsilon^{(t)}_\mha = \varepsilon^{(t)} \cdot \frac{R^{(t)}_\mha}{\tilde{\omega}^{\tv, (t)}}.
\#
By Lemma \ref{lemma:mag}, the intermediate inputs $\{X^{(t)}_i\}_{i \in [n]}$ and $\{X^{(t)}_{i\star}\}_{i \in [n]}$ satisfy
\$
\max_{i \in [n]}\bigl\|(X^{(t)}_i)^\top \bigr\|_{r, \infty}  \leq \tilde{\alpha}^{(t)} R^{(t)}, \qquad \max_{i \in [n]}\bigl\|(X^{(t)}_{i\star})^\top \bigr\|_{r, \infty}  \leq R^{(t)},
\$
where $R^{(t)}$ is defined in \eqref{eq:rt}. By Lemma \ref{lemma:cover_ffn}, it holds for all $t = 0, \dots, T-1$ that
\#\label{eq:bound-fa}
& \log N\bigl(\mathfrak{I}_{\cD_n}(\cF^{L, (t)}_\ffn), \varepsilon^{(t)}_\ffn, \|\cdot\|_{r, \infty}\bigr)\notag\\
&\quad  \leq 2dd_\sigma \cdot \log\biggl(1 + \frac{2(\alpha^{\tx, (t)} R^{\sigma, (t)} + \alpha^{\sigma, (t)} R^{\tx, (t)}) \cdot R^{(t)}}{\varepsilon^{(t)}_\ffn}\biggr) \leq 2D^2 \cdot \log \biggl(1 + \frac{2 R^{(t+1)}}{\varepsilon^{(t)}}\biggr),
\#
where the second inequality follows from \eqref{eq:cover-eps-t}, the fact that $D = \max\{d, d_\tp, d_\sigma, d_\ty\}$, and the definition of $R^{(t)}$ in \eqref{eq:rt}. By Lemmas \ref{lemma:cover_attn} and \ref{lemma:const-bound}, it holds for all $t = 0, \dots, T-2$ that
\#\label{eq:bound-fw}
& \log N\bigl(\mathfrak{I}_{\cD_n}(\cF^{L, (t+1)}_\mha), \varepsilon^{(t)}_\mha, \|\cdot\|_{r, \infty}\bigr)\notag\\
&\quad \leq (2+h) d^2\cdot \log \biggl(1 + \frac{2\tilde{\alpha}^{(t)} R^{(t)}\cdot R^{(t)}_\mha}{\varepsilon^{(t)}_\mha}\biggr) \leq (2+h) D^2\cdot \log \biggl(1 + \frac{2R^{(t+1)}}{\varepsilon^{(t)}}\biggr),
\#
where the second inequality follows from \eqref{eq:cover-eps-t}, the fact that $D = \max\{d, d_\tp, d_\sigma, d_\ty\}$, and the definition of $R^{(t)}$ in \eqref{eq:rt}. It remains to choose the resolutions $\{\varepsilon^{(t)}\}_{0 \leq t \leq T-1}$ that satisfy \eqref{eq:epsilon}, which is
\#\label{eq:epsilon-t}
\varepsilon & = \sum_{t = 0}^{T-2} \biggl(\varepsilon^{(t)}\cdot \frac{R^{(t)}_\mha}{\tilde{\omega}^{\tv, (t)}}\cdot \prod_{\tau = t+1}^{T-1}\tilde{\rho}^{(\tau)}\tilde{\alpha}^{(\tau)}\biggr)\notag\\
& \quad\qquad + \sum_{t = 0}^{T-1} \biggl(\varepsilon^{(t)}\cdot  \frac{\alpha^{\tx, (t)} R^{\sigma, (t)} + \alpha^{\sigma, (t)} R^{\tx, (t)}}{\tilde{\alpha}^{(t)}}\cdot \prod_{\tau = t+1}^{T-1}\tilde{\rho}^{(\tau)}\tilde{\alpha}^{(\tau)}\biggr).
\#
Recall that $R_\trans$ is defined in \eqref{eq:const-trans}. For all $t = 0, \dots, T-1$, we set
\#\label{eq:choice-eps-t}
\varepsilon^{(t)} & = \frac{\varepsilon}{{R}_\trans \cdot \prod_{\tau = t+1}^{T-1}  \tilde{\omega}^{\tv, (\tau)}\tilde{\alpha}^{(\tau)}},
\#
which satisfies \eqref{eq:epsilon-t}. Note that, by the definition of $R^{(t)}$ in \eqref{eq:rt}, it holds that $R^{(t+1)}\cdot \prod_{\tau = t+1}^{T-1}  \tilde{\omega}^{\tv, (\tau)}\tilde{\alpha}^{(\tau)} = R^{(T)}$ for all $t = 0, \dots, T-1$. Combining Lemma \ref{lemma:prop}, \eqref{eq:bound-fa}, \eqref{eq:bound-fw}, and the choices of $\{\varepsilon^{(t)}\}_{0 \leq t \leq T-1}$ in \eqref{eq:choice-eps-t}, we obtain
\$
\log N\bigl(\mathfrak{I}_{\cD_n}(\cF^{L}_j), \varepsilon, \|\cdot\|_{r, \infty}\bigr)&\leq \bigl[(4+h)T - h-2\bigr] D^2\cdot\log \biggl(1 + \frac{2{R}^{(T)} R_\trans}{\varepsilon}\biggr)\\
& \leq (4+h) D^2T\cdot\log \biggl(1 + \frac{2{R}^{(T)} R_\trans}{\varepsilon}\biggr).
\$
Therefore, we conclude the proof of Lemma \ref{lemma:cover_trans}.
\end{proof}


\section{Optimization Error Analysis} \label{appendix:opt-ap}
\begin{proof}[Proof of Proposition \ref{prop:opt-stat}]
Let $\hat \cL(f_\theta) = \hat{\EE}[\cL((X, y), {f}_{\theta})]$. By \eqref{eq:statw}, it holds for the stationary point $\hat \theta$ that,
	\$
		0 \leq \bigl\la \nabla_\theta \hat \cL(f_{\hat \theta}), \theta - \hat\theta\bigr\ra = \hat \EE\Bigl[\nabla_f \cL\bigl((X, y), f_{\hat \theta}\bigr) \nabla_\theta f_{\hat \theta}(X)^\top (\theta - \hat \theta) \Bigr].
	\$
	Since the objective function $\cL((X, y), f) = \|y - f(X)\|_2^2$ is convex with respect to $f(X)$, we have
	\#\label{eq:optf}
		0 \leq \hat \EE\Bigl[\nabla_f\cL\bigl((X, y), f_{\theta^*}\bigr)^\top (f - f_{\theta^*})(X) \Bigr],
	\#
where $\theta^* = \argmin_{\theta \in \Theta}\hat \cL(f_\theta)$. By definition of the objective function $\cL((X, y), f)$, we have
\#\label{eq:obj-norm}
\Bigl\|\nabla_f\cL\bigl((X, y), f_{\hat \theta}\bigr)\Bigr\|_2 = 2\bigl\|y - f_{\hat \theta}(X)\bigr\|_2 \leq 2\|y\|_2 + 2\|f_{\hat \theta}(X)\bigr\|_2 \leq 2,
\#
where the last inequality follows from Assumption \ref{asu::data} and that the aggregation layer $\agg_{\theta_0}: \RR^{d_\tp} \to \RR^{d_\ty}$ outputs within $\mathfrak{Y}$. For any $\theta \in \Theta$, it holds that,
	\#\label{eq:obj-diff}
		\hat\cL(f_{\hat \theta}) - \hat\cL(f_{\theta^*}) &\leq \hat \EE\Bigl[\nabla_f\cL\bigl((X, y), f_{\hat \theta}\bigr)^\top (f_{\hat \theta} - f_{\theta^*})(X) \Bigr]\notag \\
		& \le \hat\EE\Bigl[\nabla_f\cL\bigl((X, y), f_{\hat \theta}\bigr)^\top (f_{\hat \theta} - f_{\theta^*})(X) \Bigr] + \hat \EE\Bigl[\nabla_f\cL\bigl((X, y), f_{\hat \theta}\bigr)^\top \nabla_\theta f_{\hat \theta}(x)^\top (\theta - \hat \theta) \Bigr]\notag \\
		& = \hat \EE\Bigl[\nabla_f\cL\bigl((X, y), f_{\hat \theta}\bigr)^\top \bigl(f_{\hat \theta}(X) + \nabla_\theta f_{\hat \theta}(X)^\top (\theta - \hat \theta) - f_{\theta^*}(X)\bigr) \Bigr] \notag\\
		& \leq \hat \EE\biggl[\Bigl\|\nabla_f\cL\bigl((X, y), f_{\hat \theta}\bigr)\Bigr\|_2\cdot \bigl\|f_{\hat \theta}(X) + \nabla_\theta f_{\hat \theta}(X)^\top (\theta - \hat \theta) - f_{\theta^*}(X)\bigr\|_2 \biggr] \notag\\
		& \le 2 \cdot \bigl\|f_{\hat \theta}(X) + \nabla_\theta f_{\hat \theta}(X)^\top (\theta - \hat \theta) - f_{\theta^*}(X)\bigr\|_2,
	\#
where the second line follows from \eqref{eq:optf}, the fourth line follows from the Cauchy-Schwartz inequality, and the last line follows from \eqref{eq:obj-norm}. Since \eqref{eq:obj-diff} holds for any $\theta \in \Theta$, we have
	\$
		\hat\cL(f_{\hat \theta}) - \hat\cL(f_{\theta^*}) \leq 2 \cdot \min_{\theta \in \Theta}\hat \EE\Bigl[\bigl\|f_{\hat \theta}(X) + \nabla_\theta f_{\hat \theta}(X)^\top (\theta - \hat \theta) - f_{\theta^*}(X)\bigr\|_2\Bigr].
	\$
	Therefore, we conclude the proof of Proposition \ref{prop:opt-stat}.
\end{proof}

\section{Approximation Error Analysis}
\subsection{Latent-to-Value RKHS}\label{sec:ltv-rkhs}
In what follows, we cast the function class $\cG^\dagger_i$ defined in \eqref{eq:g-class-i} as the RKHS $\cH_\ltv$, which plays a key role in our subsequent analysis of the approximation error. Recall that the latent-to-value mapping $\psi(z; \mask)$ is defined in \eqref{eq:exact-sm}, which induces the kernel function $\fk_{\ltv}(z, z'; \mask) = \psi(z; \mask)^\top \psi(z'; \mask)$ and the following RKHS,
\#\label{eq:def-rkhs-ltv}
\cH_{\ltv} = \biggl\{g_\alpha(z; \mask) = \int \alpha(z')\fk_\ltv(z', z; \mask) \ud z': \bigl\|g_\alpha(\cdot; \mask)\bigr\|_{\cH_\ltv} < \infty \biggr\},
\#
which is equipped with the inner product $\la \cdot, \cdot \ra_{\cH_{\ltv}}$. By the definition of the kernel function $\fk_\ltv(\cdot, \cdot; \mask)$, we have for any $g_\alpha(\cdot; \mask) \in \cH_\ltv$ that
\#\label{eq:g-rkhs}
g_\alpha(z; \mask) & = \int \alpha(z')\fk_\ltv(z', z; \mask) \ud z'\notag\\
& = \biggl(\underbrace{\int \alpha(z')\psi(z'; \mask) \ud z'}_{\displaystyle w_\alpha \in \RR^{d}}\biggr)^\top \psi(z; \mask) = w_\alpha^\top \psi(z; \mask).
\#
Here $w_\alpha$ corresponds to the parameter vector $w_i \in \RR^{d}$ in the function class $\cG^\dagger_i$. On the other hand, we have
\#\label{eq:g-rkhs-norm}
\bigl\|g_\alpha(\cdot; \mask)\bigr\|^2_{\cH_\ltv} & = \bigl\la g_\alpha(\cdot; \mask), g_\alpha(\cdot; \mask) \bigr\ra_{\cH_\ltv}\notag\\
& = \int \alpha(z')\fk_\ltv(z', z; \mask)\alpha(z)\ud z \ud z'\notag\\
& = \biggl(\int \alpha(z')\psi(z'; \mask)\ud z'\biggr)^\top\biggl(\int \alpha(z)\psi(z; \mask)\ud z \biggr) = \|w_\alpha\|_2^2,
\#
where the third equality follows from the definition of $\fk_\ltv(\cdot, \cdot; \mask)$ and the last equality follows from the definition of $w_\alpha$ in \eqref{eq:g-rkhs}. Combining \eqref{eq:g-rkhs}, \eqref{eq:g-rkhs-norm}, and the definition of $\cH_\ltv$ in \eqref{eq:def-rkhs-ltv}, we have
\$
\cH_\ltv = \bigl\{w_\alpha^\top \psi(z; \mask): w_\alpha \in \RR^d, \|w_\alpha\|_2<\infty\bigr\} = \cG_i^\dagger.
\$
Thus, the function class $\cG_i^\dagger$, which correspondes to the $i$-th entry of the function class $\cG$, is the RKHS $\cH_\ltv$. Here the function class $\cG^\dagger$ is defined in \eqref{eq:g-class}, which contains the latent-to-target function $g_W^\dagger(z; \mask) = W^\top \psi(z;\mask)$ within the reweighted CME attention $f^\dagger_W(X; \mask)$ defined in \eqref{eq:f-att-cme}.

\subsection{Supervised Learning}\label{sec::pf_lem_approx_error}
\begin{proof}[Proof of Theorem \ref{lem::approx_error}]
Suppose $f_\theta \in \cF_\att$ and $\epsilon_\att \in [0,+\infty)$ satisfy \eqref{eq:appprox-sm-cme}. By the definition of the approximation error $\cE_{\rm approx}$ in \eqref{eq::risk_decomp}, we have
\#\label{eq::pf_lem_approx_eq1}
\cE_{\rm approx}& = \min_{f\in\cF_\att}\EE\Bigl[\cL\bigl((X, y), f\bigr)\Bigr] - \EE\Bigl[\cL\bigl((X, y), f^*\bigr)\Bigr]\notag\\
& \leq \EE\Bigl[\cL\bigl((X, y), f_\theta\bigr)\Bigr] - \EE\Bigl[\cL\bigl((X, y), f^*\bigr)\Bigr]\notag\\
& = \EE\Bigl[\bigl\|f_\theta(X ;\mask) - f^*(X)\bigr\|_2^2\Bigr]\notag\\
& \leq  2\EE\Bigl[\bigl\|f_\theta(X; \mask) - f_W^\dagger(X; \mask)\bigr\|_2^2\Bigr] + 2\EE\Bigl[\bigl\|f_W^\dagger(X; \mask) - f^*(X)\bigr\|_2^2\Bigr]\notag\\
& \leq 2\epsilon^2_\att + 2\EE\Bigl[\bigl\|f_W^\dagger(X; \mask) - f^*(X)\bigr\|_2^2\Bigr],
\#
where the second line follows from the fact that $f_\theta \in \cF_\att$, the third line follows from the fact that $\cL((X, y), f) = \|y - f(X)\|_2^2$ and the definition of the regression function $f^*(X) = \EE[y\given X]$, and the last line follows from \eqref{eq:appprox-sm-cme} and the definition of $f_W^\dagger(X; \mask)$ in \eqref{eq:f-att-cme}.

In what follows, we characterize the gap between the regression function $f^*(X)$  and the reweighted CME attention in $f_W^\dagger(X; \mask)$, which is used as a surrogate function for approximating $f^*(X)$. By \eqref{eq:latent-target}, we have
\$
f^*(X) & = \EE_{z \given X}\bigl[g^*(z; \mask)\bigr]\notag\\
& = \EE_{z \given X}\bigl[g_{W}^\dagger(z; \mask)\bigr] + \EE_{z \given X}\bigl[g^*(z; \mask) - g_{W}^\dagger(z; \mask)\bigr]\notag\\
& = f_{W}^\dagger(X; \mask) + \EE_{z \given X}\bigl[g^*(z; \mask) - g_{W}^\dagger(z; \mask)\bigr],
\$
where the last line follows from \eqref{eq:f-att-cme-g}. Hence, it holds for $f_{W}^\dagger(X; \mask)$ that
\#\label{eq::pf_lem_approx_eq2}
\EE\Bigl[\bigl\|f^*(X) - f_{W}^\dagger(X; \mask)\bigr\|_2^2\Bigr] = \EE\biggl[ \Bigl\|\EE_{z \given X}\bigl[g^*(z; \mask) - g^\dagger_{W}(z; \mask)\bigr]\Bigr\|_2^2\biggr].
\#
By Assumption \ref{asu::approx_err_sl}, we have
\#\label{eq::pf_lem_approx_eq3}
\Bigl\|\EE_{z \given X}\bigl[g^*(z; \mask) - g^\dagger_{W}(z; \mask)\bigr]\Bigr\|_2^2 & = \sum_{i = 1}^{d_\ty}\EE_{z \given X}\bigl[g_i^*(z; \mask) - g^\dagger_{W, i}(z; \mask)\bigr]^2\notag\\
& \leq \sum_{i = 1}^{d_\ty}\bigl\|g_i^*(\cdot; \mask) - g^\dagger_{W, i}(\cdot; \mask)\bigr\|_\infty^2 \leq \epsilon^2_g(\mask),
\#
where the $\ell_\infty$-norm is taken over the latent variable $z$. Taking \eqref{eq::pf_lem_approx_eq3} into \eqref{eq::pf_lem_approx_eq2}, we obtain
\#\label{eq::pf_lem_approx_eq4}
\EE\Bigl[\bigl\|f^*(X) - f_{W}^\dagger(X; \mask)\bigr\|_2^2\Bigr] \leq \epsilon^2_g(\mask).
\#
Taking \eqref{eq::pf_lem_approx_eq4} into \eqref{eq::pf_lem_approx_eq1}, we obtain
\$
\cE_{\rm approx} \leq 2\epsilon_\att^2 + 2\epsilon_g^2(\mask),
\$
which concludes the proof of Theorem \ref{lem::approx_error}.
\end{proof}

\subsection{Self-Supervised Learning}\label{sec::pf_lem_approx_ssl}
\begin{proof}[Proof of Theorem \ref{lem::approx_ssl}]
Suppose that $f_\dt(\overbar X; \mask_\dt)$ attains the infimum on the right-hand side of \eqref{eq:approx-sm-cme-ssl}. Recall that $B = W_\dt^\top(W_\ssl W^\top_\ssl)^{-1}W_\ssl$ is defined in \eqref{eq:b}. We define a surrogate function as follows,
\$
\tilde f_\pre(\overbar X; \mask_\pre) = B \hat{f}_{\pre}(\overbar X; \mask_\pre).
\$
Here $\hat f_\pre(\overbar X;\mask_\pre)$ is the attention neural network obtained from the pretraining process. Recall that the regression function $f^*_\dt(\overbar X)$ for the downstream task is defined in \eqref{eq:fg-dt} and $f^\dagger_{W_\dt}(\overbar X; \mask_\dt)$ is defined in \eqref{eq:surrogate-ssl}. For the approximation error $\cE_{\rm approx}$ defined in \eqref{eq:approx-ssl}, we have
\#\label{eq::pf_lem_approx_ssl_eq1}
\cE_{\text{approx}} &\leq \EE\Bigl[\cL\bigl((\overbar X, y_\dt), f_{\dt}\bigr)\Bigr] - \EE\Bigl[\cL\bigl((\overbar X, y_\dt), f_\dt^*\bigr)\Bigr]\\
& = \EE\Bigl[\bigl\|f_{\dt}(\overbar X; \mask_\dt) - f^*_\dt(\overbar X)\bigr\|^2_2\Bigr]\notag\\
& = \EE\Bigl[\bigl\|f_{\dt}(\overbar X; \mask_\dt) -\tilde f_{\pre}(\overbar X; \mask_\pre) + \tilde f_{\pre}(\overbar X; \mask_\pre) - f^\dagger_{W_\dt}(\overbar X; \mask_\dt)\notag\\
& \quad\qquad + f^\dagger_{W_\dt}(\overbar X; \mask_\dt) - f^*_\dt(\overbar X)\bigr\|^2_2\Bigr]\notag\\
& \leq 3\EE\Bigl[\bigl\|f_{\dt}(\overbar X; \mask_\dt) - \tilde f_{\pre}(\overbar X; \mask_\pre)\bigr\|^2_2\Bigr] + 3\EE\Bigl[\bigl\|\tilde f_{\pre}(\overbar X; \mask_\pre) - f^\dagger_{W_\dt}(\overbar X; \mask_\dt)\bigr\|^2_2\Bigr]\notag\\
& \quad \qquad  + 3 \EE\Bigl[\bigl\|f^\dagger_{W_\dt}(\overbar X; \mask_\dt) - f^*_\dt(\overbar X)\bigr\|^2_2\Bigr]\notag\\
& \leq 3\epsilon^2_\agg(B) + 3\underbrace{\EE\Bigl[\bigl\|\tilde f_{W_\dt}(\overbar X; \mask_\dt) - f^\dagger_{W_\dt}(\overbar X; \mask_\dt)\bigr\|^2_2\Bigr]}_{\displaystyle \text{(i)}}  + 3 \underbrace{\EE\Bigl[\bigl\|f^\dagger_{W_\dt}(\overbar X; \mask_\dt) - f^*_\dt(\overbar X)\bigr\|^2_2\Bigr]}_{\displaystyle \text{(ii)}},\notag
\#
where the second line follows from the definition of the regression function $f_\dt^*(\overbar X) = \EE[y_\dt\given \overbar X]$ and the last line follows from \eqref{eq:approx-sm-cme-ssl}. In what follows, we characterize terms (i) and (ii).
\vskip4pt
\noindent{\bf Term (i).} Recall that the regression function $f^*_\pre(\overbar X)$ for the pretraining process is defined in \eqref{eq::def_hu_ssl}. For any truncated input sequence $\overbar X$, it holds that
\#\label{eq::pf_lem_approx_ssl_eq5}
&  \bigl\|\tilde f_\pre(\overbar X; \mask_\pre) - f_{W_\dt}^\dagger(\overbar X; \mask_\dt) \bigr\|_2^2\notag\\
  &\quad = \Bigl\|B \hat f_\pre(z; \mask_\pre) - W^\top_\dt\EE_{z \given \overbar X}\bigl[\psi_\dt(z; \mask_\dt)\bigr]\Bigr\|_2^2\notag\\
  &\quad = \biggl\| B \bigl(\hat f_\pre(\overbar X; \mask_\pre) - f^*_\pre(\overbar X)\bigr)  + \Bigl(B f^*_\pre(\overbar X)- \EE_{z \given \overbar X}\bigl[W^\top_\dt\psi_\dt(z; \mask_\dt)\bigr]\Bigr)\biggr\|_2^2\notag\\
    &\quad \leq 2\underbrace{\Bigl\| B \bigl(\hat f_\pre(\overbar X; \mask_\pre) - f^*_\pre(\overbar X)\bigr)\Bigr\|_2^2}_{\displaystyle \text{(i.a)}}  + 2\underbrace{\biggl\|\Bigl(B f^*_\pre(\overbar X)- \EE_{z \given \overbar X}\bigl[W^\top_\dt\psi_\dt(z; \mask_\dt)\bigr]\Bigr)\biggr\|_2^2}_{\displaystyle \text{(i.b)}},
    \#
where the second line follows from the definition of $f^\dagger_{W_\dt}(\overbar X; \mask_\dt)$ in \eqref{eq:f-att-cme-g-ssl}. In the sequel, we characterize terms (i.a) and (i.b). By Assumption \ref{asu::effec_ssl_param}, we have
\#\label{eq::pf_lem_approx_ssl_eq6}
\text{(i.a)}  \leq \|B\|^2_{2} \cdot \bigl\|\hat f_\pre(\overbar X; \mask_\pre) - f^*_\pre(\overbar X)\bigr\|_2^2 \leq \mu \cdot \bigl\|\hat f_\pre(\overbar X; \mask_\pre) - f^*_\pre(\overbar X)\bigr\|_2^2.
\#
Recall that $g_{W_\ssl}^\dagger(z; \mask_\dt)$ is defined in Assumption \ref{asu::approx_err_ss}. Since $BW^\top_\ssl = W^\top_\dt$, we have
\#\label{eq::pf_lem_approx_ssl_eq7}
\text{(i.b)} & = \Bigl\|B\EE_{z \given \overbar X}\bigl[g^*_\pre(z; \mask_\pre) - W_\ssl^\top\psi_\dt(z; \mask_\dt)\bigr]\Bigr\|_2^2\notag\\
& \leq \|B\|_2^2 \cdot \Bigl\|\EE_{z \given \overbar X}\bigl[g^*_\pre(z; \mask_\pre) - g_{W_\ssl}^\dagger(z; \mask_\dt)\bigr]\Bigr\|_2^2\notag\\
& \leq \mu \cdot \sum_{i = 1}^{d}\EE_{z \given \overbar X}\bigl[g^*_{\pre, i}(z; \mask_\pre) - g_{W_\ssl, i}^\dagger(z; \mask_\dt)\bigr]^2\notag\\
& \leq \mu \cdot \sum_{i = 1}^{d}\bigl\|g^*_{\pre, i}(\cdot; \mask_\pre) - g_{W_\ssl, i}^\dagger(\cdot; \mask_\dt)\bigr\|_\infty^2\notag\\
& \leq \mu \cdot \epsilon^2_\ssl(\mask_\pre, \mask_\dt),
\#
where the third line follows from Assumption \ref{asu::effec_ssl_param} and the last line follows from Assumption \ref{asu::approx_err_ss}. Taking \eqref{eq::pf_lem_approx_ssl_eq6} and \eqref{eq::pf_lem_approx_ssl_eq7} into \eqref{eq::pf_lem_approx_ssl_eq5}, we obtain
\#\label{eq::pf_lem_approx_ssl_eq8}
\text{(i)} & = \EE\Bigl[\bigl\|\tilde f_\pre(\overbar X; \mask_\pre) - f_{W_\dt}^\dagger(\overbar X; \mask_\dt) \bigr\|_2^2\Bigr]\notag\\
& \leq \EE\Bigl[ 2\mu \cdot \bigl\|\hat f_\pre(\overbar X; \mask_\pre) - f^*_\pre(\overbar X)\bigr\|_2^2 + 2\mu \cdot \epsilon_\ssl(\mask_\pre, \mask_\dt)\Bigr]\notag\\
& \leq 2\mu \cdot \cE_{\rm approx}^\pre + \mu \cdot \EE\Bigl[ \epsilon^2_\ssl(\mask_\pre, \mask_\dt)\Bigr]\notag\\
& = 2\mu \cdot \bigl(\cE_{\rm approx}^\pre + \epsilon^2_\ssl(\mask_\pre, \mask_\dt)\bigr),
\#
where the third line follows from the definition of the regression function $f^*_\pre(\overbar X) = \EE[y_\pre\given \overbar X]$ for the pretraining process and the definition of $\cE_{\rm approx}^\pre$ in \eqref{eq:approx-pre}.
\vskip4pt
\noindent{\bf Term (ii).} By the same argument for \eqref{eq::pf_lem_approx_eq2}, we have
\#\label{eq::pf_lem_approx_ssl_eq2}
\text{(ii)}= \EE\biggl[ \Bigl\|\EE_{z \given \overbar X}\bigl[g_\dt^*(z; \mask_\dt) - g^\dagger_{W_\dt}(z; \mask_\dt)\bigr]\Bigr\|_2^2\biggr].
\#
By Assumption \ref{asu::approx_err_ss}, we have
\#\label{eq::pf_lem_approx_ssl_eq3}
& \Bigl\|\EE_{z \given \overbar X}\bigl[g_\dt^*(z; \mask_\dt) - g^\dagger_{W_\dt}(z; \mask_\dt)\bigr]\Bigr\|_2^2\notag\\
 &\quad = \sum_{i = 1}^{d_\ty}\EE_{z \given \overbar X}\bigl[g_{\dt, i}^*(z; \mask_\dt) - g^\dagger_{W_\dt, i}(z; \mask_\dt)\bigr]^2\notag\\
& \quad \leq \sum_{i = 1}^{d_\ty}\bigl\|g_{\dt, i}^*(\cdot; \mask_\dt) - g^\dagger_{W_\dt, i}(\cdot; \mask_\dt)\bigr\|_\infty^2 \leq \epsilon^2_g(\mask_\dt).
\#
Taking \eqref{eq::pf_lem_approx_ssl_eq3} into \eqref{eq::pf_lem_approx_ssl_eq2}, we obtain
\#\label{eq::pf_lem_approx_ssl_eq4}
\text{(ii)} \leq \epsilon^2_g(\mask_\dt).
\#
Taking \eqref{eq::pf_lem_approx_ssl_eq8} and \eqref{eq::pf_lem_approx_ssl_eq4} into \eqref{eq::pf_lem_approx_ssl_eq1}, we conclude the proof of Theorem \ref{lem::approx_ssl}.
\end{proof}


\section{Auxiliary Proofs for Generalization}\label{appendix:gen-aux}
\subsection{Proof of Lemma \ref{lemma:dudley}}\label{appendix:dudley}
\begin{proof}
Throughout this proof, we consider a fixed dataset $\cD_n = \{(X_i, y_i)\}_{i \in [n]}$. Let $\varepsilon_m = 2^{-m}$ with $m \in [M+2]$, where $M$ is a positive integer. We denote by $\mathfrak{N}_m$ the covering of the empirical image class $\mathfrak{I}_{\cD_n}(\cF_j)$ that achieves the covering number $N(\mathfrak{I}_{\cD_n}(\cF_j), \varepsilon_m, \|\cdot\|_{r, \infty})$. In other words, for any $f_j \in \cF_j$, let $\hat{f}^m[f_j] = (\hat{f}^m[f_{j, i}])_{i \in [n]}^\top \in \mathfrak{N}_m$ be the nearest element of $f_j(X_i)$ in $\mathfrak{N}_m$, which implies that
\$
\max_{i \in [n]}\bigl|f_{j}(X_i)  - \hat{f}^m[f_{j, i}]\bigr| \leq \varepsilon_m.
\$
We have
\#\label{eq:dudley-decomp}
\cR_{\cD_n}(\cF_j)& = \EE\biggl[\sup_{f_j \in \cF_j}\frac{1}{n}\sum_{i = 1}^n\epsilon_i\cdot f_j(X_i)\biggr]\\
& = \EE\biggl[\sup_{f_j \in \cF_j}\biggl\{\frac{1}{n}\sum_{i = 1}^n\epsilon_i\cdot \bigl(f_j(X_i) - \hat{f}^M[f_{j, i}]\bigr) + \frac{1}{n}\sum_{m = 1}^{M-1}\sum_{i = 1}^n\epsilon_i \cdot (\hat{f}^m[f_{j, i}] - \hat{f}^{m+1}[f_{j, i}])\notag\\
& \qquad - \frac{1}{n}\sum_{i = 1}^n\epsilon_i\cdot \hat{f}^{1}[f_{j, i}]\biggr\}\biggr]\notag\\
& \leq \EE\biggl[\sup_{f_j \in \cF_j}\frac{1}{n}\sum_{i = 1}^n\epsilon_i\cdot \bigl(f_j(X_i) - \hat{f}^M[f_{j, i}]\bigr)\biggr] + \sum_{m = 1}^{M-1}\EE\biggl[\sup_{f_j \in \cF_j}\frac{1}{n}\sum_{i = 1}^n\epsilon_i \cdot (\hat{f}^m[f_{j, i}] - \hat{f}^{m+1}[f_{j, i}])\biggr]\notag\\
& \qquad + \EE\biggl[\sup_{f_j \in \cF_j}\frac{1}{n}\sum_{i = 1}^n\epsilon_i\cdot \hat{f}^{1}[f_{j, i}]\biggr]\notag\\
& \leq \underbrace{\EE\biggl[\sup_{f_j \in \cF_j}\frac{1}{n}\sum_{i = 1}^n\epsilon_i\cdot \bigl(f_j(X_i) - \hat{f}^M[f_{j, i}]\bigr)\biggr]}_{\displaystyle \text{(i)}} + \sum_{m = 1}^{M-1}\EE\biggl[\underbrace{\sup_{f_j \in \cF_j}\frac{1}{n}\sum_{i = 1}^n\epsilon_i \cdot (\hat{f}^m[f_{j, i}] - \hat{f}^{m+1}[f_{j, i}])}_{\displaystyle \text{(ii)}}\biggr],\notag
\#
where the last line follows from the choice $\mathfrak{N}_1 = \{0\}$ and the fact that $f_{j}(X_i) \in [0, 1/2]$ for any $i \in [n]$. In what follows, we analyze terms (i) and (ii).
\vskip4pt
\noindent{\bf Term (i).} We have
\#\label{eq:dudley-i}
\text{(i)} \leq  \EE\biggl[\frac{1}{n}\sum_{i = 1}^n|\epsilon_i|\biggr]\cdot \sup_{f_j \in \cF_j}\max_{i \in [n]}\bigl|f_{j}(X_i) - \hat{f}^M[f_{j, i}]\bigr| \leq n \cdot \varepsilon_M.
\#
\vskip4pt
\noindent{\bf Term (ii).} Let $f_{j, \cD_n} = (f_j(X_i))_{i \in [n]} \in \RR^{1 \times n}$. We have
\#\label{eq:dudley-ii-temp}
 \sup_{f_j \in \cF_j}\bigl\|\hat{f}^m[f_j] - \hat{f}^{m+1}[f_j]\bigr\|_2& \leq \sup_{f_j \in \cF_j}\bigl\|\hat{f}^m[f_j] - f_{j, \cD_n}\bigr\|_2 + \sup_{f_j \in \cF_j}\bigl\|f_{j, \cD_n} - \hat{f}^{m+1}[f_j]\bigr\|_2\notag\\
& \leq \sqrt{n} \cdot \sup_{f_j \in \cF_j}\bigl\|\hat{f}^m[f_j] - f_{j, \cD_n}\bigr\|_\infty + \sqrt{n} \cdot \sup_{f_j \in \cF_j}\bigl\|f_{j, \cD_n} - \hat{f}^{m+1}[f_j]\bigr\|_\infty\notag\\
& \leq \sqrt{n} \cdot \varepsilon_m + \sqrt{n} \cdot \varepsilon_{m+1} = 3\sqrt{n} \cdot \varepsilon_{m+1}.
\#
Combining \eqref{eq:dudley-ii-temp} with the Massart's finite class lemma \citep{mohri2018foundations}, we obtain
\#\label{eq:dudley-ii}
\EE\biggl[\sup_{f_j \in \cF_j}\sum_{i = 1}^n\epsilon_i \cdot \bigl(\hat{f}^m[f_{j, i}] - \hat{f}^{m+1}[f_{j, i}]\bigr)\biggr] & \leq 3\sqrt{n} \cdot \varepsilon_{m+1} \cdot \sqrt{2 \log \bigl(|\mathfrak{N}_m|\cdot |\mathfrak{N}_{m+1}|\bigr)}\notag\\
& \leq 6\sqrt{n} \cdot \varepsilon_{m+1} \cdot \sqrt{\log |\mathfrak{N}_{m+1}|},
\#
where the second line follows from the fact that $|\mathfrak{N}_{m+1}| \geq |\mathfrak{N}_{m}|$. Taking \eqref{eq:dudley-i} and \eqref{eq:dudley-ii} into \eqref{eq:dudley-decomp}, we obtain
\$
\cR_{\cD_n}(\cF_j) & \leq \varepsilon_M + \frac{6}{\sqrt{n}} \cdot\sum_{m = 1}^{M-1}\varepsilon_{m+1} \cdot \sqrt{\log |\mathfrak{N}_{m+1}|}\\
& \leq \varepsilon_M + \frac{12}{\sqrt{n}} \cdot\sum_{m = 1}^{M}(\varepsilon_m - \varepsilon_{m+1}) \cdot \sqrt{\log |\mathfrak{N}_{m}|}\\
& \leq \varepsilon_M + \frac{12}{\sqrt{n}} \int_{\varepsilon_{M+1}}^{1/2}\sqrt{\log N\bigl(\mathfrak{I}_{\cD_n}(\cF_j), \varepsilon, \|\cdot\|_{r, \infty}\bigr)} \ud \varepsilon\\
& \leq 4\xi + \frac{12}{\sqrt{n}}\int_{\xi}^{1/2}\sqrt{\log N\bigl(\mathfrak{I}_{\cD_n}(\cF_j), \varepsilon, \|\cdot\|_{r, \infty}\bigr)} \ud \varepsilon,
\$
where the last inequality holds for any $0 < \xi < 1$ and the smallest $M$ such that $\xi \leq \varepsilon_{M+1}$, which implies that $\varepsilon_{M} = 2\varepsilon_{M+1} < 4\xi$. Therefore, we conclude the proof of Lemma \ref{lemma:dudley}.
\end{proof}

\subsubsection{Proof of Matrix Ball Covering Lemma \ref{lemma:cover_matrix_ball}}\label{appendix:cover_ball}
\begin{proof}[Proof of Lemma \ref{lemma:cover_matrix_ball}]
Let $M^\top  = (m_1, \dots, m_{d_1}) \in \RR^{d_2 \times d_1}$, where $m_j \in \RR^{d_1}$ with $j \in [d_2]$. We define the vectorization of the matrix $M \in \RR^{d_1 \times d_2}$ as $\vec(M) = (m_j^\top)^\top_{j \in [d_2]} \in \RR^{d_1d_2}$. We define the sectional norm for the vector $\vec(M) \in \RR^{d_1d_2}$ as follows,
\$
\bigl\|\vec(M)\bigr\|_{r(d_2), s(d_1)} = \|M^\top\|_{r, s},
\$
which can be verified to be a proper norm. In Lemma \ref{lemma:cover_ball}, setting 
\$
\mathbb{B} = \mathbb{B}_* = \bigl\{m \in \RR^{d_1d_2}: \|m\|_{r(d_2), s(d_1)} \leq 1\bigr\} = \bigl\{M^\top \in \RR^{d_2 \times d_1}: \|M^\top\|_{r, s} \leq 1\bigr\},
\$
and $\|\cdot\| = \|\cdot\|_* = \|\cdot\|_{r(d_2), s(d_1)}$, we obtain
\$
& \log N\Bigl(\bigl\{M^\top \in \RR^{d_2 \times d_1}: \|M^\top\|_{r, s} \leq R^{\mathrm{m}}\bigr\}, \varepsilon, \|\cdot\|_{r, s}\Bigr)\\
& \quad = \log N\Bigl(\bigl\{m \in \RR^{d_1d_2}: \|m\|_{r(d_2), s(d_1)} \leq 1 \bigr\}, \varepsilon/R^M, \|\cdot\|_{r, s}\Bigr)\\
& \quad \leq \frac{\text{vol}(2R_M/\varepsilon \cdot \mathbb{B}_* + \mathbb{B})}{\text{vol}(\mathbb{B})} = d_1d_2\cdot \log \biggl(1+\frac{2R_M}{\varepsilon}\biggr).\notag
\$
Therefore, we conclude the proof of Lemma \ref{lemma:cover_matrix_ball}.
\end{proof}

\subsection{Lipschitz Continuity of Multihead Attention}\label{appendix:lip-sm}
\begin{lemma}[Lipschitz Continuous Softmax]\label{lemma:lip-sm}
Let $(r, s)$ be a conjugate pair. Under Assumption \ref{assumption:ker}, it holds for any $q, \hat{q} \in \RR^{d_\tp}$ and $K = (k^\ell)_{\ell \in [L]}^\top, \hat{K} = (\hat k^\ell)_{\ell \in [L]}^\top \in \RR^{L \times d_\tp}$ that
\#
\Bigl\|\Norm_\sm\bigl(\fk_\rbf(K, q)\bigr) - \Norm_\sm\bigl(\fk_\rbf(\hat K, q)\bigr)\Bigr\|_{1} & \leq \bigl(\|q\|_r + \|K^\top\|_{r, \infty}\bigr) \cdot\|K^\top  - \hat{K}^\top \|_{r, \infty}, \label{eq:sm1}\\
\Bigl\|\Norm_\sm\bigl(\fk_\rbf(K, q)\bigr) - \Norm_\sm\bigl(\fk_\rbf(K, \hat q)\bigr)\Bigr\|_{1} & \leq \bigl(\|q\|_r + \|K^\top\|_{r, \infty}\bigr) \cdot \|q - \hat{q}\|_r.\label{eq:sm2}
\#
\end{lemma}
\begin{proof}
Let $P = \text{diag}(p) - p p^\top \in \RR^{L \times L}$ with $p = (p_\ell)_{\ell \in [L]}=\Norm_\sm\bigl(\fk_\rbf(K, q)\bigr) \in \RR^{L}$. We have
\$
p_\ell \propto \exp\bigl\{- \|q - k^\ell\|_2^2/2\sigma^2\bigr\}.
\$
We define $g_\ell = - \|q - k^\ell\|_2^2/2\sigma^2$ and $g = (g_\ell)_{\ell \in [L]}^\top \in \RR^L$. Let the Jacobian of $p \in \RR^L$ with respect to $k^\ell \in \RR^{d_\tp}$ be $J_\ell \in \RR^{L\times d_\tp}$. We have
\$
J_\ell = \frac{\partial p}{\partial k^\ell} =  \frac{\partial p}{\partial g}\cdot\frac{\partial g}{\partial k^\ell} =P\frac{\partial g}{\partial k^\ell},
\$
where $\partial g/\partial k^\ell = (e_\ell q^\top - E_{\ell, \ell}K)/\sigma^2$. Here $E_{\ell, \ell'} \in \RR^{L \times L}$ is the unit matrix whose $(\ell, \ell')$-th entry is one and all other entries are zero. Note that
\$
\biggl\|\sum_{\ell = 1}^L J_\ell \Delta_\ell \biggr\|_1 \leq \sum_{\ell = 1}^L \|J_\ell \Delta_\ell \|_1 \leq \sum_{\ell = 1}^L \|J_\ell\|_{r \to 1} \cdot \|\Delta_\ell \|_r \leq \|\Delta\|_{r, \infty} \cdot \sum_{\ell = 1}^L\|J_\ell\|_{r \to 1},
\$
where $\Delta = (\Delta_\ell^\top)_{\ell \in [L]}$. Thus, the Lipschitz continuity constant of $\softmax(q, K)$ is bounded by $\sum_{\ell = 1}^L\|J_\ell\|_{r \to 1}$. Let $e_\ell \in \RR^L$ be the $\ell$-th one-hot vector with $\ell \in [L]$. For any $\ell \in [L]$, we have
\#\label{eq:jl-bound}
\|J_\ell\|_{r \to 1} & \leq \frac{d_\tp^{1 - 1/r}}{\sigma^2}\cdot \bigl\|P (e_\ell q^\top - E_{\ell, \ell}K)\bigr\|_{1}\notag\\
 & = \frac{d_\tp^{1/s}}{\sigma^2}\cdot p_\ell \cdot \bigl\|(e_\ell - p)(q - k^\ell)^\top\bigr\|_{1}\notag\\
 & = \frac{d_\tp^{1/s}p_\ell}{\sigma^2}\cdot \|e_\ell - p\|_s\cdot\|q - k^\ell\|_r,
\#
where the equalities follow from $1/r + 1/s = 1$.
Summing up \eqref{eq:jl-bound} for all $\ell \in [L]$, we obtain
\#\label{eq:sum-jinf}
\sum_{\ell \in [L]}\|J_\ell\|_{r \to 1} & \leq  \frac{d_\tp^{1/s}}{\sigma^2}\cdot \sum_{\ell \in [L]}p_\ell\cdot \|e_\ell - p\|_s\cdot\|q - k^\ell\|_r\notag\\
& \leq  \frac{d_\tp^{1/s}}{\sigma^2}\cdot \bigl(\|q\|_r + \|K^\top\|_{r, \infty}\bigr)\cdot\sum_{\ell \in [L]}p_\ell \cdot \|e_\ell - p\|_s.
\#
On the other hand, we have
\#\label{eq:sum-ps}
\sum_{\ell \in [L]}p_\ell\cdot \|e_\ell - p\|_s & = \sum_{\ell \in [L]}\biggl\{p_\ell\cdot \biggl[\sum_{\ell' \neq l}p_{\ell'}^s + (1 - p_\ell)^s\biggr]^{1/s}\biggr\}\notag\\
& \leq \sum_{\ell \in [L]}p_\ell\cdot \bigl[2(1 - p_\ell)^s\bigr]^{1/s} \leq 2^{1/s}.
\#
Combining \eqref{eq:sum-jinf} and \eqref{eq:sum-ps}, we have for $\sigma = (2d_\tp)^{1/2s}$ that
\$
\sum_{\ell \in [L]}\|J_\ell\|_{r \to 1}  \leq \frac{(2d_\tp)^{1/s}}{\sigma^2}\cdot\bigl(\|q\|_r + \|K^\top\|_{r, \infty}\bigr).
\$
Thus, $\Norm_\sm\bigl(\fk_\rbf(K, q)\bigr)$ is $(\|q\|_r + \|K^\top\|_{r, \infty})$-Lipschitz in $K^\top$ with respect to $\|\cdot\|_{r, \infty}$, which concludes the proof of \eqref{eq:sm1}. Since $\fk_\rbf(q, k) = \fk_\rbf(k, q)$, we also have \eqref{eq:sm2} by the same arguments for \eqref{eq:sm1}. Therefore, we conclude the proof of Lemma \ref{lemma:lip-sm}.
\end{proof}

\subsubsection{Proof of Lemma \ref{lemma:lip-mha-param}}\label{appendix:lip-mha-param}
\begin{proof}
For notational simplicity, we write
\$
\head_i = \att_\sm(XW^{\tq}_i, XW^{\tk}_i, XW^{\tv}_i), \quad \hat{\head}_i = \att_\sm(X\hat{W}^{\tq}_i, X\hat{W}^{\tk}_i, X\hat{W}^{\tv}_i).
\$
By the definition of sequence-to-sequence multihehead attention in \eqref{eq:mha-seq2seq}, we have
\#\label{eq:mha-decomp}
\bigl\|\mha(X; W)^\top - \mha(X; \hat{W})^\top\bigr\|_{r, \infty} & \leq \biggl\|\sum_{i = 1}^h\head^\top_i - \sum_{i = 1}^h\hat{\head}^\top_i\biggr\|_{r, \infty}\notag\\
&\leq \sum_{i = 1}^h\bigl\|(\head_i - \hat{\head}_i)^\top\bigr\|_{r, \infty}.
\#
Also, we have
\#\label{eq:head-diff-decomp}
&\bigl\|(\head_i  - \hat{\head}_i)^\top\bigr\|_{r, \infty}\\
&\quad = \bigl\|\att_\sm(XW^{\tq}_i, XW^{\tk}_i, XW^{\tv}_i)^\top - \att_\sm(X\hat{W}^{\tq}_i, X\hat{W}^{\tk}_i, X\hat{W}^{\tv}_i)^\top\bigr\|_{r, \infty}\notag\\
& \quad = \biggl\|\Bigl((XW^{\tv}_i)^\top\Norm_\sm\bigl(\fk_\rbf(XW^{\tk}_i, x^\ell W^{\tq}_i) \bigr)\Bigr)_{\ell \in [L]} - \Bigl((X\hat W^{\tv}_i)^\top\Norm_\sm\bigl(\fk_\rbf(X\hat W^{\tk}_i, x^\ell \hat W^{\tq}_i) \bigr)\Bigr)_{\ell \in [L]}\biggr\|_{r, \infty}\notag\\
&\quad = \max_{\ell \in [L]}\Bigl\|(XW^{\tv}_i)^\top\Norm_\sm\bigl(\fk_\rbf(XW^{\tk}_i, x^\ell W^{\tq}_i) \bigr) - (X\hat{W}^{\tv}_i)^\top\Norm_\sm\bigl(\fk_\rbf(X\hat W^{\tk}_i, x^\ell \hat W^{\tq}_i)\bigr) \Bigr\|_{r}.\notag
\#
Note that
\#\label{eq:lip-bound}
\|x^\ell W^{\tq}_i\|_{r} + \bigl\|(XW^{\tk}_i)^\top\bigr\|_{r, \infty} & \leq \bigl\|(XW^{\tq}_i)^\top\bigr\|_{r, \infty} + \bigl\|(XW^{\tk}_i)^\top\bigr\|_{r, \infty}\\
& \leq \Bigl(\bigl\|(W^{\tq}_i)^\top\bigr\|_{r} + \bigl\|(W^{\tk}_i)^\top\bigr\|_{r}\Bigr)\cdot \|X^\top\|_{r, \infty} \leq (\omega^{\tq}_i + \omega^{\tk}_i) \cdot R.\notag
\#
Then for any $\ell \in [L]$, we have
\#\label{eq:attn-diff}
&\Bigl\|(XW^{\tv}_i)^\top\Norm_\sm\bigl(\fk_\rbf(XW^{\tk}_i, x^\ell W^{\tq}_i) \bigr) - (X\hat{W}^{\tv}_i)^\top\Norm_\sm\bigl(\fk_\rbf(X\hat W^{\tk}_i, x^\ell \hat W^{\tq}_i)\bigr) \Bigr\|_{r}\notag\\
& \quad \leq \Bigl\|\Norm_\sm\bigl(\fk_\rbf(X\hat W^{\tk}_i, x^\ell \hat W^{\tq}_i)\bigr)^\top X(W^{\tv}_i - \hat{W}^{\tv}_i) \Bigr\|_{r}\notag\\
& \quad \qquad + \biggl\|\Bigl(\Norm_\sm\bigl(\fk_\rbf(XW^{\tk}_i, x^\ell W^{\tq}_i) \bigr) - \Norm_\sm\bigl(\fk_\rbf(X\hat W^{\tk}_i, x^\ell \hat W^{\tq}_i)\bigr)\Bigr)^\top X{W}^{\tv}_i \biggr\|_{r}\notag\\
& \quad \leq \bigl\|(W^{\tv}_i - \hat{W}^{\tv}_i)^\top\bigr\|_{r, s} \cdot \|X^\top\|_{r, \infty}\notag\\
& \quad \qquad + \Bigl\|\Norm_\sm\bigl(\fk_\rbf(XW^{\tk}_i, x^\ell W^{\tq}_i) \bigr) - \Norm_\sm\bigl(\fk_\rbf(X\hat W^{\tk}_i, x^\ell \hat W^{\tq}_i)\bigr)\Bigr\|_1 \cdot \bigl\|({W}^{\tv}_i)^\top\bigr\|_{r} \cdot \|X^\top\|_{r, \infty}\notag\\
& \quad \leq R \cdot \varepsilon_i^\tv +  \Bigl(\|x^\ell W^{\tq}_i\|_{r} + \bigl\|(XW^{\tk}_i)^\top\bigr\|_{r, \infty}\Bigr) \cdot \omega^{\tv}_i R^2 \cdot (\varepsilon_i^\tq + \varepsilon_i^\tk)\notag\\
& \quad \leq R \cdot \varepsilon_i^\tv +  (\omega^{\tq}_i + \omega^{\tk}_i) \omega^{\tv}_i \cdot R^3 \cdot (\varepsilon_i^\tq + \varepsilon_i^\tk),
\#
where the third inequality follows from Lemma \ref{lemma:norm}, the fourth inequality follows from Lemma \ref{lemma:lip-sm}, and the last inequality follows from \eqref{eq:lip-bound}. Taking \eqref{eq:attn-diff} into \eqref{eq:head-diff-decomp}, we have
\#\label{eq:head-diff}
\bigl\|(\head_i  - \hat{\head}_i)^\top\bigr\|_{r, \infty} \leq  R \cdot \varepsilon_i^\tv +  (\omega^{\tq}_i + \omega^{\tk}_i) \omega^{\tv}_i \cdot  R^3 \cdot (\varepsilon_i^\tq + \varepsilon_i^\tk).
\#
Taking \eqref{eq:head-diff} into \eqref{eq:mha-decomp}, we obtain
\$
\bigl\|\mha(X; W)^\top - \mha(X; \hat{W})^\top\bigr\|_{r, \infty} \leq  R\cdot \sum_{i = 1}^h \varepsilon_i^\tv + R^3\cdot \sum_{i = 1}^h (\omega^{\tq}_i + \omega^{\tk}_i) \cdot (\varepsilon_i^\tq + \varepsilon_i^\tk).
\$
Therefore, we conclude the proof of Lemma \ref{lemma:lip-mha-param}.
\end{proof}

\subsubsection{Proof of Lemma \ref{lemma:lip-mha-input}}\label{appendix:lip-mha-input}
\begin{proof}
In this proof, with a slight abuse of notations, we write
\$
\head_i = \att_\sm(XW^{\tq}_i, XW^{\tk}_i, XW^{\tv}_i), \quad \hat{\head}_i = \att_\sm(\hat{X}{W}^{\tq}_i, \hat{X}{W}^{\tk}_i, \hat{X}{W}^{\tv}_i).
\$
Similar to \eqref{eq:mha-decomp}, we have
\$
& \bigl\|\mha(X; W)^\top - \mha(\hat{X}; W)^\top\bigr\|_{r, \infty} \leq  \sum_{i = 1}^h\bigl\|(\head_i - \hat{\head}_i)^\top\bigr\|_{r, \infty}.
\$
For any fixed $\ell \in [L]$, we have
\#\label{eq:head-diff-input} 
& \bigl\|(\head_i - \hat{\head}_i)^\top\bigr\|_{r, \infty}\notag\\
&\quad =\biggl\|\Bigl( (XW^{\tv}_i)^\top \Norm_\sm\bigl(\fk_\rbf(XW^{\tk}_i, x^\ell W^{\tq}_i) \bigr) - (\hat{X}W^{\tv}_i)^\top \Norm_\sm\bigl(\fk_\rbf(\hat X W^{\tk}_i, \hat x^\ell W^{\tq}_i)\bigr)\Bigr)_{\ell \in [L]}\biggr\|_{r, \infty}\notag\\
&\quad =\max_{\ell \in [L]}\Bigl\| (XW^{\tv}_i)^\top \Norm_\sm\bigl(\fk_\rbf(XW^{\tk}_i, x^\ell W^{\tq}_i) \bigr) - (\hat{X}W^{\tv}_i)^\top \Norm_\sm\bigl(\fk_\rbf(\hat X W^{\tk}_i, \hat x^\ell W^{\tq}_i)\bigr)\Bigr\|_{r}\notag\\
&\quad \leq \max_{\ell \in [L]}\Bigl\|\Norm_\sm\bigl(\fk_\rbf(XW^{\tk}_i, x^\ell W^{\tq}_i) \bigr)\Bigr\|_1 \cdot\bigl\|(W^{\tv}_i)^\top\bigr\|_{r} \cdot \|X^\top - \hat{X}^\top\|_{r, \infty}\notag\\
&\quad \qquad + \max_{\ell \in [L]}\Bigl\|\Norm_\sm\bigl(\fk_\rbf(XW^{\tk}_i, x^\ell W^{\tq}_i) \bigr) - \Norm_\sm\bigl(\fk_\rbf(\hat X W^{\tk}_i, \hat x^\ell W^{\tq}_i)\bigr)\Bigr\|_1 \cdot \bigl\|({W}^{\tv}_i)^\top\bigr\|_{r} \cdot \|X^\top\|_{r, \infty}\notag\\
&\quad \leq \bigl\|(W^{\tv}_i)^\top\bigr\|_{r} \cdot \|X^\top - \hat{X}^\top\|_{r, \infty}\notag\\
&\quad \qquad + \max_{\ell \in [L]}\bigl\|(W^{\tv}_i)^\top\bigr\|_{r}\cdot  R^2 \cdot  (\omega^{\tq}_i + \omega^{\tk}_i)\cdot \Bigl(\|x^\ell W^{\tq}_i\|_{r} + \bigl\|(XW^{\tk}_i)^\top\bigr\|_{r, \infty}\Bigr)\notag\\
&\quad \leq \omega^{\tv}_i \cdot \bigl[1 + R^2 \cdot (\omega^{\tq}_i+ \omega^{\tk}_i)^2\bigr] \cdot  \|X^\top - \hat{X}^\top\|_{r, \infty},
\#
where the second inequality follows from Lemma \ref{lemma:lip-sm} and \eqref{eq:lip-bound}, the third inequality follows from Lemma \ref{lemma:lip-sm}, and the last inequality follows from \eqref{eq:lip-bound}. Summing up \eqref{eq:head-diff-input}, we obtain
\$
\bigl\|\mha(X; W)^\top - \mha(\hat{X}; W)^\top\bigr\|_{r, \infty} & \leq \biggl\{\sum_{i = 1}^h\omega^{\tv}_i \cdot \bigl[1 + R^2 \cdot (\omega^{\tq}_i+ \omega^{\tk}_i)^2\bigr]\biggr\} \cdot \|X^\top - \hat{X}^\top\|_{r, \infty}\\
& =  \biggl[\sum_{i = 1}^h\omega^{\tv}_i  + R^2\cdot\sum_{i = 1}^h(\omega^{\tq}_i+ \omega^{\tk}_i)^2\omega^{\tv}_i\biggr] \cdot \|X^\top - \hat{X}^\top\|_{r, \infty}.
\$
Therefore, we conclude the proof of Lemma \ref{lemma:lip-mha-input}.
\end{proof}

\section{Auxiliary Lemmas}
\begin{lemma}[Volume Ratios and Metric Entropy, \cite{wainwright2019high}]\label{lemma:cover_ball}
Consider a pair of norms $\|\cdot\|$ and $\|\cdot\|_*$ on $\RR^d$, and let $\mathbb{B}$ and $\mathbb{B}_*$ be their corresponding unit balls. Then the $\varepsilon$-covering number of $\mathbb{B}_*$ in the $\|\cdot\|$-norm obeys the bounds
\$
\varepsilon^{-d}\cdot \frac{\text{vol}(\mathbb{B}_*)}{\text{vol}(\mathbb{B})} \leq N\bigl(\mathbb{B}_*, \varepsilon, \|\cdot\|\bigr) \leq \frac{\text{vol}(2/\varepsilon \cdot \mathbb{B}_* + \mathbb{B})}{\text{vol}(\mathbb{B})}.
\$
\end{lemma}

\begin{lemma}[\citet{caponnetto2007optimal}]
	\label{lem:cme-concen}
	Let $(\Omega, \nu)$ be a probability space and $\xi$ be a random variable on $\Omega$ taking value in a real separable Hilbert space $\cH$. We assume that there exists constants $B, \sigma > 0$ such that
	\$
		\bigl\|\xi(w)\bigr\|_\cH \le B/2,\ \mathrm{a.s.}, \quad \EE\bigl[\norm{\xi}_\cH^2\bigr] \le \sigma^2.
	\$
	Then, it holds with probability at least $ 1- \delta$ that
	\begin{align*}
		\biggl\| L^{-1} \sum_{i = 1}^L \xi(\omega_i) - \EE[\xi] \biggr\| \le 2\biggl( \frac{B}{L} + \frac{\sigma}{\sqrt{L}} \biggr) \log \frac{2}{\delta}.
	\end{align*}
\end{lemma}


\begin{lemma}\label{lemma:norm}
Let $(r, s)$ be a conjugate pair. For any $M \in \RR^{d_1 \times d_2}$, $u \in \RR^{d_2 \times 1}$, and $U \in \RR^{d_2 \times d_3}$, we have
\$
\|Mu\|_{r} & \leq \|M\|_{r, \infty} \cdot \|u\|_1,\\
\|MU\|_{r, \infty} & \leq \|M\|_{r, s} \cdot \|U\|_{r, \infty},\\
\|MU\|_{r, \infty} & \leq \|M\|_{r} \cdot \|U\|_{r, \infty}.
\$
\end{lemma}
\begin{proof}
	Let $M = (m_i)_{i \in [d_2]}$ and $b = (b_i)_{i \in [d_2]}^\top$. We have
	\$
	\|Mu\|_r = \biggl\|\sum_{j =1}^{d_2} u_j \cdot m_j\biggr\|_r \leq \biggl(\sum_{j =1}^{d_2} |u_j|\biggr) \cdot \max_{j \in [d_2]}\|m_j\|_r = \|M\|_{r, \infty} \cdot \|u\|_1.
	\$
	Also, we have
	\$
	\|Mu\|_r = \biggl\|\sum_{j =1}^{d_2} u_j \cdot m_j\biggr\|_r \leq \biggl(\sum_{j =1}^{d_2} |u_j| \cdot \|m_j\|_r \biggr) \leq \|M\|_{r, s} \cdot \|u\|_{r}.
	\$
	As a consequence, with $U = (u_i)_{i \in [d_3]}$, we have
	\$
	\|MU\|_{r, \infty} = \max_{j \in [d_3]} \|M u_j\|_r \leq \max_{j \in [d_3]} \|M\|_{r, s}\cdot \|u_j\|_{r} = \|M\|_{r, s} \cdot \|U\|_{r, \infty}.
	\$
	On the other hand, by the definition of the matrix operator norm, we obtain
	\$
	\|MU\|_{r, \infty} = \max_{j \in [d_3]} \|M u_j\|_r \leq \max_{j \in [d_3]} \|M\|_{r}\cdot \|u_j\|_{r} = \|M\|_{r} \cdot \|U\|_{r, \infty}.
	\$
	Therefore, we conclude the proof of Lemma \ref{lemma:norm}.
\end{proof}

\begin{lemma}[Covering Coefficient Bounds]\label{lemma:const-bound}
We have for all $t = 0, \dots, T-1$ that
\$
1 + \rho(\mathfrak{W}^{(t)}) & \leq \tilde{\omega}^{\tv, (t)}  + (\overline{\omega}^{\tq\tk, (t)})^2\omega^{\tv, (t)} \cdot (R^{(t)})^2 = \tilde{\rho}^{(t)},\\
R_\mha(\mathfrak{W}^{(t)}) & \leq  R^{\tv, (t)} + \overline{\omega}^{\tq\tk, (t)}{R}^{\tq\tk, (t)}\cdot (R^{(t)})^2 = R^{(t)}_\mha,
\$
where $\tilde{\rho}^{(t)}$ and $R^{(t)}_\mha$ are defined in \eqref{eq:rho-rmha}.
\end{lemma}
\begin{proof}
By the definition of $R_\mha(\mathfrak{W})$ in \eqref{eq:r-mha}, we have
\$
R_\mha(\mathfrak{W}^{(t)}) & = \sum_{i = 1}^hR^{\tv, (t)}_i + (R^{\tx, (t)})^2\cdot \sum_{i = 1}^h (\omega_i^{\tq, (t)} + \omega_i^{\tk, (t)})(R_i^{\tq, (t)} + R_i^{\tk, (t)})\\
& \leq R^{\tv, (t)} + (R^{(t)})^2\cdot \max_{i \in [h]}\{\omega_i^{\tq, (t)} + \omega_i^{\tk, (t)}\}\cdot \sum_{i = 1}^h (R_i^{\tq, (t)} + R_i^{\tk, (t)})\\
& = R^{\tv, (t)} + \overline{\omega}^{\tq\tk, (t)}{R}^{\tq\tk, (t)}\cdot (R^{(t)})^2 = R^{(t)}_\mha.
\$
Also, by the definition of $\rho(\mathfrak{W})$ in \eqref{eq:rho}, we have
\$
1 + {\rho}(\mathfrak{W}^{(t)}) & = 1 + \sum_{i = 1}^h\omega^{\tv, (t)}_i  + (R^{\tx, (t)})^2\cdot\sum_{i = 1}^h(\omega^{\tq, (t)}_i+ \omega^{\tk, (t)}_i)^2\omega^{\tv, (t)}_i\\
& \leq \tilde{\omega}^{\tv, (t)}  + (R^{(t)})^2\cdot\max_{i \in [h]}\{\omega_i^{\tq, (t)} + \omega_i^{\tk, (t)}\}^2 \cdot\sum_{i = 1}^h\omega^{\tv, (t)}_i\\
& = \tilde{\omega}^{\tv, (t)}  + (\overline{\omega}^{\tq\tk, (t)})^2\omega^{\tv, (t)} \cdot (R^{(t)})^2 = \tilde{\rho}^{(t)}.
\$
Therefore, we conclude the proof of Lemma \ref{lemma:const-bound}.
\end{proof}

\begin{lemma}[Simplified Covering Coefficient]\label{lemma:simplify}
It holds that
\$
\sqrt{\log(1+R^{(T)} R_\trans)} = O\Bigl(T\sqrt{\log(1+ \gamma)} + \sqrt{T} \sqrt{\log(1 + \zeta R^{(0)})} + \sqrt{\log( 1+ \kappa/\zeta)}\Bigr),
\$
where $\gamma$, $\zeta$, and $\kappa$ are defined in \eqref{eq:const-simplify}. Here $R^{(t)}$ and $R_\trans$ are defined in \eqref{eq:rt} and \eqref{eq:const-trans}, respectively.
\end{lemma}
\begin{proof}
By the definitions of $\kappa^{(t)}$ and $\zeta^{(t)}$ in \eqref{eq:ratio}, we have
\$
\frac{R^{(t)}_\mha}{\tilde{\rho}^{(t)}} & = \frac{R^{\tv, (t)} + \overline{\omega}^{\tq\tk, (t)}{R}^{\tq\tk, (t)}\cdot (R^{(t)})^2}{\tilde{\omega}^{\tv, (t)}  + (\overline{\omega}^{\tq\tk, (t)})^2\omega^{\tv, (t)} \cdot (R^{(t)})^2} \leq \max\biggl\{\frac{R^{\tv, (t)}}{\tilde{\omega}^{\tv, (t)}}, \frac{{R}^{\tq\tk, (t)}}{ \overline{\omega}^{\tq\tk, (t)}\omega^{\tv, (t)}}\biggr\} \leq \kappa^{(t)},\\
\frac{\tilde{\rho}^{(t)}}{\tilde{\omega}^{\tv, (t)}} & = 1 + \frac{(\overline{\omega}^{\tq\tk, (t)})^2\omega^{\tv, (t)}}{\tilde{\omega}^{\tv, (t)}}\cdot (R^{(t)})^2 = 1+ \zeta^{(t)}\cdot (R^{(t)})^2,
\$
which implies
\$
{R}_\trans & = \sum_{t = 0}^{T-2}\biggl(\frac{R^{(t)}_\mha}{\tilde{\rho}^{(t)}} \cdot \prod_{\tau = t+1}^{T-1}\frac{\tilde{\rho}^{(\tau)}}{\tilde{\omega}^{\tv, (\tau)}}\biggr) + \sum_{t = 0}^{T-1}\biggl(\frac{\alpha^{\tx, (t)} R^{\sigma, (t)} + \alpha^{\sigma, (t)} R^{\tx, (t)}}{\tilde{\alpha}^{(t)}}\cdot \prod_{\tau = t+1}^{T-1}\frac{\tilde{\rho}^{(\tau)}}{\tilde{\omega}^{\tv, (\tau)}}\biggr)\\
& \leq \sum_{t = 0}^{T-1}\biggl\{(\kappa^{(t)} + \kappa^{(t)}) \cdot \prod_{\tau = t+1}^{T-1}\bigl[1+ \zeta^{(t)}\cdot (R^{(t)})^2\bigr]\biggr\}\\
& \leq 2\kappa \cdot \sum_{t = 0}^{T-1}\biggl\{\prod_{\tau = t+1}^{T-1}\bigl[1+ \zeta\cdot (R^{(t)})^2\bigr]\biggr\},
\$
where the last line follows from the definition of $\kappa$ and $\zeta$ in \eqref{eq:const-simplify}. By the definition of $R^{(t)}$ in \eqref{eq:rt} and the definition of $\gamma$ in \eqref{eq:const-simplify}, we have
\$
R^{(t)} = R^{(0)}\cdot \prod_{\tau = 0}^{t-1}\tilde{\omega}^{\tv, (\tau)}\tilde{\alpha}^{(\tau)} \leq R^{(0)}\cdot (1 + \gamma)^{2t}.
\$
As a consequence, we obtain
\$
R^{(T)} R_\trans & \leq 2\kappa R^{(0)} \cdot (1+\gamma)^{2T} \cdot \sum_{t = 0}^{T-1} \prod_{\tau = t+1}^{T-1}\bigl[1 + \zeta R^{(0)}\cdot (1+\gamma)^{4\tau} \bigr]\\
& \leq 2\kappa R^{(0)}\cdot (1+\gamma)^{2T} \cdot \sum_{t = 0}^{T-1}\bigl[1 + \zeta R^{(0)}\cdot (1+\gamma)^{4T} \bigr]^{T-t+2}\\
& = 2\kappa \cdot \bigl[1 + \zeta R^{(0)} \cdot (1+\gamma)^{4T} \bigr]^3 \cdot \frac{\bigl[1 + \zeta R^{(0)} \cdot (1+\gamma)^{4T} \bigr]^T - 1}{\zeta \cdot (1+\gamma)^{2T}}\\
& \leq \frac{2\kappa}{\zeta}  \cdot \bigl[1 + \zeta R^{(0)} \cdot (1+\gamma)^{4T} \bigr]^{T+3}.
\$
Thus, using $\sqrt{a+b} \leq \sqrt{a} + \sqrt{b}$ and $\log(1+ab)\leq \log(1+a) + \log(1+b)$, we further obtain
\$
\sqrt{\log(1 + R^{(T)} R_\trans)} & = O\Bigl(\sqrt{T^2  \log(1 + \gamma) + T  \log(1 + \zeta R^{(0)}) +  \log(1+ \kappa/\zeta)}\Bigr)\\
& = O\Bigl(T \sqrt{\log(1 + \gamma)} + \sqrt{T} \sqrt{\log(1 + \zeta R^{(0)})} +  \sqrt{\log(1+ \kappa/\zeta)}\Bigr).
\$
Therefore, we conclude the proof of Lemma \ref{lemma:simplify}.
\end{proof}
\begin{lemma}[Inter-Layer Magnitude]\label{lemma:mag}
Under Assumptions \ref{asu::data} and \ref{assumption:constraint}, it holds that
\$
\bigl\|(X^{(t)})^\top\bigr\|_{r, \infty} \leq \tilde \alpha^{(t)}R^{(t)}, \qquad \bigl\|(X_\star^{(t)})^\top\bigr\|_{r, \infty} \leq  R^{(t)},
\$
where $R^{(t)}$ is defined in \eqref{eq:rt}.
\end{lemma}
\begin{proof}
Setting $\hat{X} = 0^{L \times d}$ in \eqref{eq:mha-decomp}, we have
\$
& \bigl\|\mha(X; W)^\top\bigr\|_{r, \infty} \leq \sum_{i = 1}^h\|\head_i^\top\|_{r, \infty}.
\$
We have for all $i \in [h]$ that
\$
\|{\head}_i^\top\|_{r, \infty} & = \biggl\|\Bigl((X{W}^{\tv}_i)^\top\Norm_\sm\bigl(\fk_\rbf(XW^{\tk}_i, x^\ell W^{\tq}_i) \bigr)\Bigr)_{\ell \in [L]} \biggr\|_{r, \infty}\\
& = \max_{l \in [L]} \Bigl\|(X{W}^{\tv}_i)^\top\Norm_\sm\bigl(\fk_\rbf(XW^{\tk}_i, x^\ell W^{\tq}_i) \bigr) \Bigr\|_{r}\notag\\
& \leq \max_{l \in [L]}\Bigl\|\Norm_\sm\bigl(\fk_\rbf(XW^{\tk}_i, x^\ell W^{\tq}_i) \bigr)\Bigr\|_1 \cdot \bigl\|(W^{\tv}_i)^\top\bigr\|_{r}\cdot \| X^\top \|_{r, \infty}  \leq \omega^{\tv}_i\cdot \| X^\top \|_{r, \infty},
\$
which implies that
\$
\bigl\|\mha(X; W)^\top + X^\top\bigr\|_{r, \infty} \leq \|X^\top\|_{r, \infty} +  \biggl(\sum_{i = 1}^h\omega^{\tv}_i\biggr)\cdot \| X^\top \|_{r, \infty} = (1 + \omega^{\tv})\cdot \| X^\top \|_{r, \infty}.
\$
As a consequence, we have
\#\label{eq:z-mag}
 \bigl\|(X_\star^{(t)})^\top\bigr\|_{r, \infty} \leq (1 + \omega^{\tv, (t)})\cdot \bigl\|(X^{(t)})^\top\bigr\|_{r, \infty} = \tilde{\omega}^{\tv, (t)}\cdot \bigl\|(X^{(t)})^\top\bigr\|_{r, \infty}.
\#
On the other hand, setting $\hat{X}_\star^{(t)} = 0^{L \times d}$ in \eqref{eq:nn-diff}, we have
\#\label{eq:x-mag}
\bigl\|(X^{(t+1)})^\top\bigr\|_{r, \infty} = \bigl\|\ffn({X}_\star^{(t)}; A^{(t)})^\top\bigr\|_{r, \infty} \leq \tilde{\alpha}^{(t)} \cdot \bigl\|({X}_\star^{(t)})^\top\bigr\|_{r, \infty}.
\#
Recursively applying \eqref{eq:z-mag} and \eqref{eq:x-mag}, we conclude the proof of Lemma \ref{lemma:mag}.
\end{proof}

\end{document}